\documentclass[twoside,11pt]{article}
\usepackage{jmlr2e}

\usepackage{amsmath, amsfonts,amssymb,euscript, graphicx,epsfig,enumerate,float,afterpage, subfigure, ifthen, moreverb, algpseudocode,algorithm, booktabs, enumitem}
\usepackage[normalem]{ulem}
\usepackage{url}
\usepackage{hyperref}
\hypersetup{ hidelinks }
\usepackage{tikz}
\usepackage{placeins}
\usetikzlibrary{arrows.meta, positioning, calc}
\usepackage{xcolor}

\newtheorem{thm}{Theorem}
\newtheorem{prop}{Proposition}

\newtheorem{lem}{Lemma}
\newtheorem{defn}{Definition}

\newtheorem{rmk}{Remark}

\newenvironment{namedproof}[1]{%
  \par\medskip\noindent\textbf{#1.}\ \ignorespaces
}{%
  \hfill\ensuremath{\square}\par\medskip
}

\newcommand{\curl}[1]{\left\{#1\right\}}
\newcommand{\vect}[1]{\mathbf{#1}}
\newcommand{\parn}[1]{\left(#1\right)}

\newcommand{\sqr}[1]{\left[#1\right]}
\newcommand{\bars}[1]{\left|#1\right|}

\tikzset{
  pnode/.style  = {circle, draw=black, line width=0.6pt, fill=white,
                   minimum size=6mm, inner sep=0pt},
  wbox/.style   = {rounded corners=1.5pt, draw=black, line width=0.5pt, fill=white,
                   minimum width=12mm, minimum height=6mm, inner sep=1pt, font=\footnotesize},
  summ/.style   = {circle, draw=black, line width=0.6pt, fill=white,
                   minimum size=8mm, inner sep=0pt},
  mapbox/.style = {rounded corners=2.5pt, draw=black, line width=0.6pt, fill=blue!8,
                   minimum size=8mm, inner sep=2pt, font=\small},
  genbox/.style = {rounded corners=2.5pt, draw=black, line width=0.6pt, fill=green!12,
                   minimum height=10mm, inner sep=3pt, font=\small},
  flow/.style   = {-{Stealth[length=2.2mm,width=1.8mm]}, semithick, black},
  gedge/.style  = {-{Stealth[length=2.6mm,width=2mm]}, line width=1pt, black},
  ann/.style    = {font=\scriptsize, text=black!62, align=center, inner sep=1pt},
}

\newcommand{\cardframe}[2]{%
  \coordinate (m@nw) at (current bounding box.north west);
  \coordinate (m@ne) at (current bounding box.north east);
  \coordinate (m@se) at (current bounding box.south east);
  \coordinate (m@tm) at ($(m@nw)!0.5!(m@ne)$);
  \node[font=\small] at ($(m@tm)+(0,7mm)$) {#2};                       
  \draw[thin] ($(m@nw)+(-3mm,3mm)$) -- ($(m@ne)+(3mm,3mm)$);           % separator
  \draw[semithick, rounded corners=5pt]                               % the card
       ($(m@nw)+(-3mm,11mm)$) rectangle ($(m@se)+(3mm,-3mm)$);
  \node[font=\bfseries, anchor=south] at ($(m@tm)+(0,13mm)$) {#1};    
}

\begin{document}

\title{On the Identifiability of Mixed Ordinal and Exponential Family Causal DAGs under Linear Parametric Models}

\author{\name Sambit~Mishra \email sambitmi@usc.edu \\
       \addr Ming Hsieh Department of Electrical and Computer Engineering\\
       University of Southern California\\
       Los Angeles, CA 90089, USA
       \AND
       \name Urbashi~Mitra \email ubli@usc.edu \\
       \addr Ming Hsieh Department of Electrical and Computer Engineering\\
       University of Southern California\\
       Los Angeles, CA 90089, USA
}

\editor{---}

\maketitle
\begin{abstract}
The problem of identifiability in linear parametric models (LPMs) whose nodes follow either an ordered logit model or a regular one-parameter exponential family is evaluated. The results go beyond classical structural equation models as well as results for nodes with observations from a homogeneous family of distributions.
%Determining causal direction from observational data requires structural assumptions, and existing identifiability theory, developed either for exogenous-noise structural equation models (SEM) or for single-family parametric conditionals, leaves edges between variables of different families outside its scope. 
The main result establishes that the orientation of every edge joining an ordinal node to an exponential-family node is identifiable from the joint distribution alone at every parameter value, provided the ordinal node has at least three categories and the exponential-family node at least three points of support, with no restriction on the sufficient statistic. Converses show that both requirements are necessary: the three-category requirement is binding only for affine sufficient statistics, and the three-point requirement is binding under the canonical link. The guarantee extends to orienting every such mixed ordinal-exponential family edge of a given $d$-node undirected skeleton. Numerical experiments illustrate the theoretical results by  successfully separating orientations within a Markov equivalence class, which are indistinguishable by conditional independence alone.
\end{abstract}


\section{Introduction}
\label{sec:intro}
{Causal reasoning plays a major role across modern scientific inquiry, in fields as varied as wireless networks \citep{thomas2024causal}, oncology \citep{xue2019tumour}, molecular biology \citep{triantafillou2017predicting}, financial fraud detection \citep{ren2025causal}, and digital advertising \citep{hill2015measuring}. Common to all of them is the need to uncover the causal mechanisms driving a stochastic system rather than its correlations.} A commonly used structure for modeling causal relationships among variables is the directed acyclic graph (DAG), where nodes represent variables, and directed edges encode the flow of causal influence \citep{pearl2009causality}. The central task of causal discovery is recovering the graph from data. Graph knowledge enables the design of interventions and the separation of cause from correlation, both of which most applications of causal inference rely on.

Current literature divides causal DAG recovery into two broad classes, interventional and observational. Interventional causal discovery refines the set of graphs compatible with the data, thereby improving identifiability \citep{hauser2012characterization}. Recent works in interventional causal discovery include the optimized selection and design of soft interventions \citep{yang2018characterizing, jaber2020causal, peng2026intervene, peng2025soft} and causal bandit settings \citep{lattimore2016causal, varici2023causal, peng2025bandits}. In many scientific settings, however, experimentation is expensive, unethical, or infeasible, and one must recover the graph from observational data alone. Existing works in observational causal discovery explore DAG recovery by constraint-based procedures testing conditional independence \citep{spirtes2000causation}, by score-based procedures searching over graphs \citep{chickering2002optimal}, or by recent characterizations of the finite-sample and detection-theoretic limits of recovery \citep{ghoshal2017learning, gao2022optimal, shaska2025causal, lungu2025bayesian, mishra2026causal}. Underlying all such procedures is the more basic question of identifiability, namely, whether the observational distribution suffices to determine causal direction at all. {Identifiability is the question we address herein for causal graphs with observational data and nodes that admit mixed-ordinal and one-parameter exponential-family distributions.}

{A key challenge in identifiability theory is the difficulty in identifying a causal DAG beyond its Markov equivalence class (MEC) from observational data alone without additional structural assumptions \citep{spirtes2000causation}. The existing literature addresses the challenge using structural equation modeling (SEM), thereby obtaining identifiability across various continuous and discrete families, as reviewed in Section~\ref{ss:rel_W}. Across those results, the causal mechanism often takes one of two restrictive forms: either the child node's value is a \textit{deterministic} function of its parents and exogenous noise, or the parents act on the child only through a \textit{structural intermediate}, rather than directly shaping the parameters of the child's conditional distribution.}

Real observational datasets rarely fit assumptions of either kind. The present work is motivated by epidemiological studies routinely combining ordinal, count, bounded, and continuous variables in a single record. For instance, a study of COVID-19 case dynamics in Germany \citep{steiger2021causal} models a daily count of new cases against ordinal intervention indicators, binary policy variables, continuous weather measurements, bounded mobility percentages, and unbounded socio-demographic shares. The authors had to center and scale continuous predictors and decorrelate collinear mobility variables via principal component analysis, precisely because no single SEM family natively accommodates the heterogeneity. Similar works on global mammalian virus spillover \citep{kreuder2015spillover, johnson2020global} likewise model a count outcome against a mix of ordinal, binary, and continuous predictors. Forcing such datasets into a continuous additive-noise SEM requires latent thresholding, link transformations, and truncations, all of which risk mis-specifications that can invalidate both the identifiability guarantees and the downstream causal estimates.

A complementary line of work on parametric causal models \citep{bodik2025identifiability, bodik2026structural} lets the parent values directly shape the parameters of the child's conditional distribution. The support, dispersion, and tail behavior of the child are then inherited from a chosen parametric family rather than from an exogenous noise term, allowing count, bounded, and continuous variables to coexist in a single graph without ad hoc transformations. {We shift the formulation to construct the \textbf{Linear Parametric Model (LPM)}, in which the parameters of each child's conditional distribution are functions of a weighted linear combination of the parent values, the weights being the model parameters. This allows for the conditional family to differ across nodes, with parents acting on each child only through the linear predictor defined by the weights.}
{Additionally, the structure of LPMs also allows for the use of more computationally efficient back propagation-based algorithms, which remain out of scope of this paper.}
%\um{say anything about LPM enabling other things, e.g. more computaionally efficient algos?}

We specialize our focus on identifiability in the mixed setting where an edge joins an ordinal node to a node whose conditional distribution belongs to a regular one-parameter exponential family. Existing parametric models that achieve within-MEC identifiability each assume a single distributional family for all nodes, leaving edges across families untreated. The conditionally parametric causal model (CPCM) framework of \citet{bodik2025identifiability} treats the exponential family in generality, but requires both conditional directions to admit the exponential family form. {We fix the ordinal nodes to follow the ordered-logit model of \citet{mccullagh1980regression}, and remark on extensions beyond it. This proportional-odds cumulative link model has served as the standard model for ordinal responses across the applied literature \citep{lall2002review, williams2006generalized, desantis2014regression, french2022regression}, and does not admit a one-parameter exponential family form, as Proposition~\ref{prop:ord_ne} below establishes, so the CPCM identifiability conditions do not carry over to an edge carrying an ordinal endpoint. The identifiability of edges between ordinal and exponential-family nodes in mixed-distribution DAGs is therefore an open problem that we address in this work. We enumerate the contributions of this work as follows:}
\begin{enumerate}
    \item We prove any edge between an ordinal node and a regular one-parameter exponential family node to be distributionally identifiable at every parameter value of the model, requiring only that the ordinal node has at least three categories and that the exponential-family node has at least three points of support, with no restriction on the sufficient statistic. 
    \item We also prove converses, showing that when the exponential-family node has two points of support, we can construct an explicit family of forward and reverse parameters that induce identical joint distributions for any number of ordinal categories. Similarly, when the ordinal node has two categories, we establish a dichotomy that the orientation is identifiable if and only if the exponential-family sufficient statistic is non-affine on its support.
    \item {We establish a multivariate extension orienting every such edge in a general $d$-node DAG at every parameter value, under the Markov property, causal sufficiency, and the absence of selection. We strengthen this result to the identifiability of the entire orientation of a given skeleton, and prove a converse showing that both requirements remain necessary for an edge whose endpoints carry no neighbors other than each other.}
    \item {We empirically validate the framework on the canonical $3$-node MEC, demonstrating a decay of the orientation error rate with increasing sample size, and identifiability within the MEC.}
\end{enumerate}

\subsection{Notation}
\label{ss:setup}
A causal DAG is a pair $\mathcal{G} = \parn{\mathcal{V}, \mathcal{E}}$ on $d$ nodes, where $\mathcal{V} = \curl{1, \dots, d}$ is the vertex set and $\mathcal{E} = \curl{\parn{i, j} \mid i \to j \text{ edge exists}}$ is the edge set, carrying a deterministic weighted adjacency matrix $\vect{W}_{\mathcal{G}} \in \mathbb{R}^{d \times d}$ given by
\begin{equation}
[\vect{W}_{\mathcal{G}}]_{i,j} =
\begin{cases}
0, & \parn{i,j} \notin \mathcal{E}, \\
w_{i,j} \neq 0, & \parn{i,j} \in \mathcal{E}.
\end{cases}
\label{eq1}
\end{equation}
We assume $\mathcal{G}$ is acyclic. The parent set of a node $j$ is $\mathcal{P}\parn{j}$ defined as
\begin{equation}
\mathcal{P}\parn{j} \triangleq
\curl{ i \in \mathcal{V} : \parn{i,j} \in \mathcal{E} },
\label{eq2}
\end{equation}
and each node $i$ carries a scalar random variable $X_i$, with $\vect{x} = \parn{X_1, \dots, X_d}$ denoting the vector of all $d$ variables and $\vect{X} \in \mathbb{R}^{n \times d}$ is the data matrix collecting $n$ independent realizations of $\vect{x}$. With a slight abuse of notation, we write $\vect{x}_{\mathcal{P}\parn{i}}$ for the parent variables and for their observed values interchangeably. When more than one graph is in play, we write $\mathcal{P}_{\mathcal{G}}\parn{j}$ for the parent set of $j$ in $\mathcal{G}$.

\subsection{Related Works}
\label{ss:rel_W}
The existing literature has extensively considered additive noise-based SEMs for proving identifiability in causal DAGs, restricting the causal mechanism to $X_i = f_i\parn{\vect{x}_{\mathcal{P}\parn{i}}} + \varepsilon_i$ for a deterministic function $f_i$, with the noise variables $\varepsilon_i$ jointly independent. For continuous variables, \citet{peters2014identifiability} established identifiability for linear Gaussian models with equal error variance, and \citet{shimizu2006linear} established the same for linear non-Gaussian models, with both cases considering the linear model $X_i = \vect{\beta}_i^{\mathsf{T}} \vect{x}_{\mathcal{P}\parn{i}} + \varepsilon_i$ with the noise $\varepsilon_i$ being Gaussian and non-Gaussian, respectively. \citet{hoyer2008nonlinear} considered non-linear additive noise models, while \citet{zhang2009identifiability} introduced the post-nonlinear model $X_i = g_1\parn{g_2\parn{\vect{x}_{\mathcal{P}\parn{i}}} + \varepsilon_i}$ with an invertible link function $g_1$. In the discrete setting, similar results hold for additive noise models \citep{peters2011causal} and integer modular additive noise models \citep{suzuki2014identifiability}. At the same time, \citet{cai2018causal} exploited a structural asymmetry instead, passing the cause through a low-cardinality deterministic intermediate before it generates the effect. In each case, identifiability was established by fixing the functional form through which the noise enters, and by drawing every node of the graph from a single distributional family.

Subsequent studies also allowed the parents to influence the dispersion of the child as well as its mean, giving the location-scale or heteroscedastic noise models $X_i = f_i\parn{\vect{x}_{\mathcal{P}\parn{i}}} + g_i\parn{\vect{x}_{\mathcal{P}\parn{i}}} \varepsilon_i$ for a positive scale function $g_i$. \citet{immer2023identifiability} established that the causal direction in such models is identifiable outside a set of pathological cases, with related literature considering binned approximations \citep{xu2022inferring}, patient-specific root-cause estimation \citep{strobl2022identifying}, autoregressive-flow parametrizations \citep{khemakhem2021causal}, and skewed variants \citep{klippert2025skewness}. 

A third line of work dropped the noise term entirely and lets the parents parametrize the child's conditional distribution within a specified family, deriving identifiability from that family's structural form. \citet{park2015learning} treated Poisson DAGs, in which $X_i \mid \vect{x}_{\mathcal{P}\parn{i}} \sim \operatorname{Poisson}\parn{\theta_i\parn{\vect{x}_{\mathcal{P}\parn{i}}}}$, and \citet{park2019identifiability} considered the broader generalized hypergeometric family, in which the mean and variance stand in a polynomial relationship, while \citet{ni2022ordinal} analyzed the identifiability in ordinal DAGs. \citet{rajendran2021structure} tied exponential-family conditionals to greedy score-based recovery through a Bregman-information characterization, and \citet{wang2024causal} employed node-wise generalized linear models under a formulation admitting latent confounders. Each such treatment, however, fixed a single family for the whole graph, leaving an edge between nodes of different families out of reach.

Closer to the present setting is a body of work that targeted the mixed continuous--discrete edge directly. \citet{wei2018mixed} proposed a locally consistent information criterion over a mixed factor space, and \citet{marx2019causal} oriented pairs of arbitrary and possibly mixed type by minimum description length. \citet{huang2018generalized} constructed kernel-based score functions admitting mixed data within greedy equivalence search, while \citet{zeng2022causal} gave sufficient conditions for a linear mixed model combining LiNGAM with a logistic-type conditional. More recently, \citet{yao2025causal} extended additive noise models to mixed types through a general function class, and \citet{maeda2025density} oriented a continuous--discrete pair by testing monotonicity of the conditional density ratio. An older tradition treated such variables through the conditional Gaussian distributions of \citet{lauritzen1989graphical}, with structure learning studied by \citet{andrews2019learning}. When the discrete variable was the cause, such treatments took the conditionals across its levels either to form a location-shift family or to be independently parametrized, and drew identifiability from non-Gaussianity, additive noise, description length, or a genericity argument placing coincidences in a set of measure zero.

Nearest in formulation is a line of work that replaced the deterministic parent-to-child link with a node-specific conditional distribution whose parameters are themselves functions of the parent values, so that $X_i \mid \vect{x}_{\mathcal{P}\parn{i}} \sim F\parn{\theta_i\parn{\vect{x}_{\mathcal{P}\parn{i}}}}$ for a known parametric family $F$ and a parameter map $\theta_i$. \citet{janzing2009distinguishing} fixed a second-order exponential maximum-entropy form and observed that variables supported on proper subsets of $\mathbb{R}$ can induce distinct joint laws even for Markov-equivalent graphs, illustrating the point for a binary--continuous pair. The aforementioned CPCM model of \citet{bodik2025identifiability} developed the idea across the exponential family, where the density takes the form $H\parn{x} \exp\parn{\theta^{\mathsf{T}} T\parn{x} - a\parn{\theta}}$ with sufficient statistic $T$, base measure $H$, log-partition function $a$, and natural parameter $\theta = \theta_i\parn{\vect{x}_{\mathcal{P}\parn{i}}}$, reaching multi-parameter cases such as the two-parameter Gaussian, Gamma, and Beta as well as non-linear parameter maps, with identifiability characterized by conditions relating that map to the sufficient statistics. The characterization, however, required the exponential family representation to hold in both causal directions.

%\um{do you mention the Asilomar paepr somewhere?}
Our own earlier treatment of mixed Ordinal--Poisson DAGs \citep{shaska2025ordinal}, which motivates this work, showed such an edge between an ordinal and a Poisson distributed node to be orientable outside a Lebesgue-null set of parameters by an analyticity argument. The treatment admits any analytic cumulative link with a nowhere-vanishing derivative and thus covers the probit, but remains confined to the Poisson conditional and the bivariate case, leaves the multivariate extension as a simulation-supported conjecture, and does not address where identifiability fails. {Our most recent work on the treatment of Ordinal--Exponential family DAGs \citep{mishra2026ordinal} explores this further by considering a subset of the regular one-parameter exponential family of distributions, restricted by functional assumptions on the sufficient statistics of the exponential family. The work proves distributional identifiability for continuous exponential family distributions, subject to the specified assumptions on the sufficient statistics and to identifiability outside a Lebesgue-null set of parameters, for their discrete counterparts. However, it omits the necessary conditions, the analysis of non-identifiability, the treatment of the unrestricted sufficient-statistic case, the distributional identifiability of discrete cases, and an extensive analysis of the multi-node DAG case, all of which this work addresses in depth.}

{To the best of our knowledge, no existing treatment addresses the case where an edge has endpoints in different families, one of which admits no exponential-family representation. The mixed continuous--discrete results do not apply since an exponential-family conditional is, in general, neither a location-shift family across the levels of the cause nor a collection of independently parametrized laws. The CPCM characterization fails for a different reason: it is stated in terms of a representation that an ordinal endpoint does not admit. Where such an edge has been treated at all, the conclusion has held outside a null set of parameters rather than at every point of the parameter space, {and where the conclusion has been shown to hold for every point in the parameter set, it has been restricted by strong assumptions on the exponential family,} and nothing has been established about where the conclusion fails.}

\textbf{The remainder of the paper is organized as follows.} Section~\ref{sec:bg} introduces the LPM framework and specializes the framework to the mixed ordinal--exponential family setting. Section~\ref{sec:id} establishes the bivariate and multivariate identifiability results together with their converses. Section~\ref{sec:num} presents the numerical experiments, and Section~\ref{sec:conc} concludes the paper. 

\section{Preliminaries}
\label{sec:bg}
In this section, we formalize the modeling framework used throughout the paper. We begin with the standard assumptions for causal DAGs, then introduce the LPM framework as a specialization of the CPCM framework, and finally specialize the LPM to the ordinal-exponential setting underlying our identifiability results.

\subsection{Causal DAGs}
We assume throughout that the joint distribution of $\vect{x}$ is Markov with respect to $\mathcal{G}$, factorizing as
\begin{equation}
p\parn{\vect{x}} =
\prod_{i=1}^{d}
p\!\parn{ X_i \mid \vect{x}_{\mathcal{P}\parn{i}} }.
\label{eq3}
\end{equation}
Two DAGs are Markov equivalent when they entail the same conditional independence relations, equivalently when they share a skeleton and a set of v-structures. The collection of all such DAGs equivalent to a given DAG is its Markov equivalence class (MEC).

The classical SEM specifies each node as a \textit{deterministic} function of its parents and exogenous noise,
\begin{eqnarray}
    X_i = f_i\parn{\vect{x}_{\mathcal{P}\parn{i}}, \varepsilon_i}, \quad i \in \mathcal{V},
    \label{eq4}
\end{eqnarray}
where $f_i: \mathbb{R}^{\left|\mathcal{P}\parn{i}\right| + 1} \to \mathbb{R}$ is the deterministic function defining the SEM and the noise variables $\curl{\varepsilon_i}_{i = 1}^{d}$ are mutually independent with pre-defined distributions. In such instantiations, however,  the stochasticity of the child enters only through the exogenous noise term $\varepsilon_i$ whose family is fixed in advance, leaving tail and support inherited from $\varepsilon_i$, restricting the parents from acting on the full conditional distribution of the child.

\subsection{Linear Parametric Models}
The CPCM framework of \citet{bodik2025identifiability} replaces the deterministic construction with a node-specific conditional distribution whose parameters are functions of the parent values,
\begin{eqnarray}
    X_i \sim p_i\parn{\cdot \mid \boldsymbol{\theta}_i\parn{\vect{x}_{\mathcal{P}\parn{i}}}}, \quad i \in \mathcal{V},
    \label{eq5}
\end{eqnarray}
where $p_i$ is the conditional distribution of $X_i$. This allows for the support, dispersion, and tail behavior to be inherited from the parametric family rather than an exogenous noise term. Instead of focusing on each individual parent value in $\vect{x}_{\mathcal{P}\parn{i}}$ directly as done in CPCM, we turn our focus on a weighted linear combination of the parent values $\vect{w}_i^{\mathsf{T}}\vect{x}_{\mathcal{P}\parn{i}}$, where the weights $\vect{w}_i$ determine the influence of each parent value on the conditional of the child node. This formulation is motivated by artificial neural networks, where the output of a single neuron is estimated as a function of the weighted linear combination of its inputs. Such constructions also facilitate the development of efficient back-propagation-based algorithms currently studied in the literature, which are outside the scope of this paper. We refer to such models as \textbf{Linear Parametric Models (LPM)}, formalized in Definition~\ref{def:ssm} below.
\begin{defn}[Linear Parametric Model]
\label{def:ssm}
    A linear parametric model on a DAG $\mathcal{G} = \parn{\mathcal{V}, \mathcal{E}}$ is a tuple
    \begin{eqnarray}
        \mathcal{S} = \parn{\mathcal{G}, \curl{p_i}_{i \in \mathcal{V}}, \curl{\boldsymbol{\theta}_i}_{i \in \mathcal{V}}, \vect{W}_{\mathcal{G}}},
        \label{eq6}
    \end{eqnarray}
    satisfying the following:
    \begin{enumerate}
        \item For each $i \in \mathcal{V}$, $p_i\parn{\cdot \mid \boldsymbol{\theta}_i}$ is a regular parametric family with real parameters $\boldsymbol{\theta}_i$ and real support $\mathcal{X}_i \subseteq \mathbb{R}$.
        \item For each $i \in \mathcal{V}$, the parameters are determined by the parent values through a known parameter link function given as,
        \begin{eqnarray}
            \boldsymbol{\theta}_i \triangleq \boldsymbol{\theta}_i\parn{\vect{w}_i^{\mathsf{T}}\vect{x}_{\mathcal{P}\parn{i}}},
            \label{eq7}
        \end{eqnarray}
        where $\vect{w}_i \triangleq \sqr{\vect{W}_{\mathcal{G}}}_{\mathcal{P}\parn{i}, i}$ is the weight vector corresponding to node $i$.
        \item The conditional distributions $\curl{p_i\parn{X_i \mid \vect{x}_{\mathcal{P}\parn{i}}}}_{i \in \mathcal{V}}$ are compatible with the Markov factorization in \eqref{eq3}.
    \end{enumerate}
\end{defn}
It is interesting to note that Definition~\ref{def:ssm} recovers the canonical linear Gaussian SEM considered by \citet{peters2014identifiability} with additive noise $\varepsilon_i \sim \mathcal{N}\parn{0, \sigma_i^2}$ on setting $p_i$ Gaussian with mean $\vect{w}_i^{\mathsf{T}}\vect{x}_{\mathcal{P}\parn{i}}$ and fixed variance $\sigma_i^2$. The conditional family $p_i$ may also differ across nodes, allowing a single LPM to combine ordinal, count, bounded, and continuous variables within one graph without ad-hoc transformations. Lastly, note that the conditional distribution $p_i$ rather than an exogenous noise term determines the support of $X_i$, respecting discrete and bounded variables by construction. We summarize the contrast between the three constructions --- SEM, CPCM, and LPM --- at a single node $X_i$ in Figure~\ref{fig:ssm_toy}, marking the quantities known in advance and those estimated from data. 

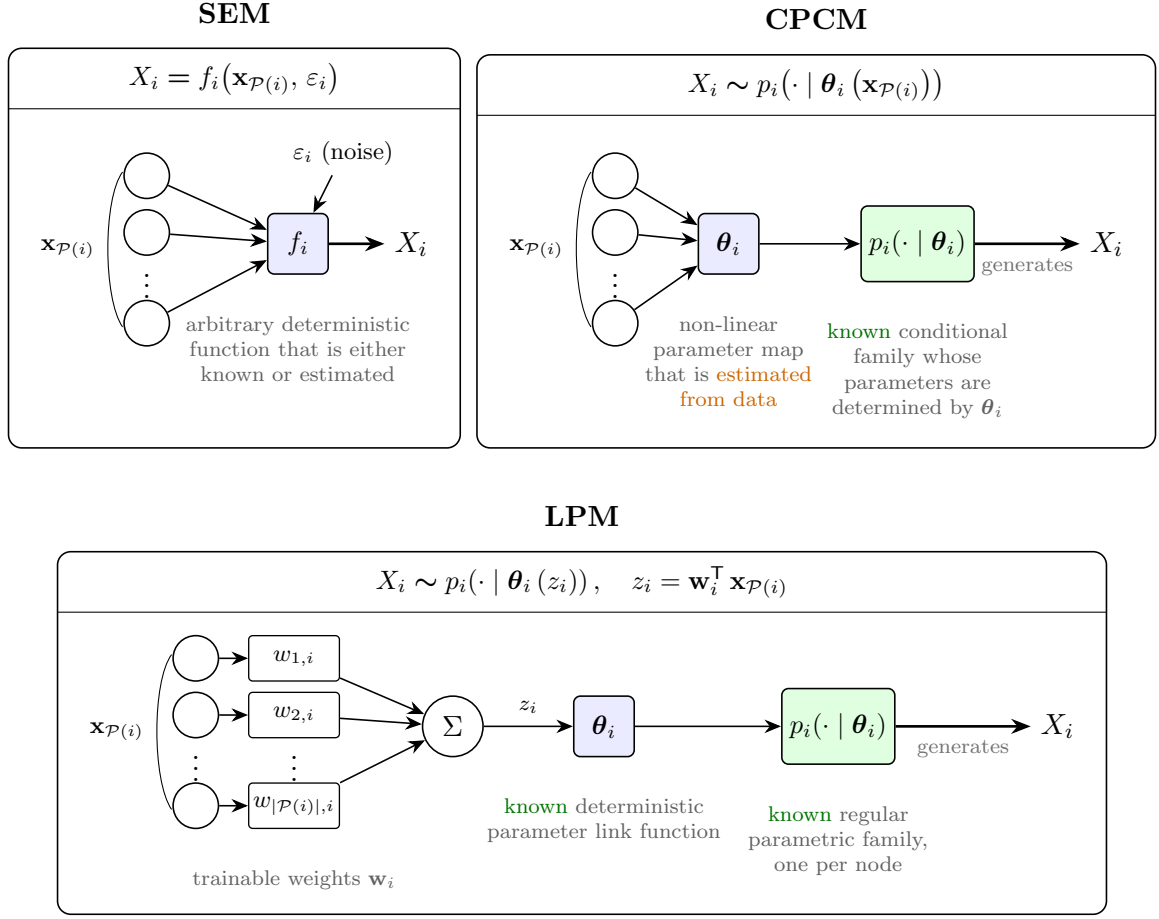
\begin{figure}[!htbp]
\centering
\begin{tabular}{@{}c@{\hspace{2mm}}c@{}}
%==================== SEM ====================
\begin{tikzpicture}[baseline=0pt]
  \node[pnode] (p1) at (0,0.9)  {};   \node[pnode] (p2) at (0,0.15) {};
  \node        (pd) at (0,-0.45){$\vdots$};   \node[pnode] (p3) at (0,-1.05){};
  \draw[thin] (-0.32,0.95) .. controls (-0.58,0.6) and (-0.58,-0.72) .. (-0.32,-1.1)
       node[midway, left=1pt, font=\footnotesize] {$\vect{x}_{\mathcal{P}\parn{i}}$};
  \node[mapbox] (f) at (2.0,0) {$f_i$};
  \node[font=\footnotesize] (eps) at (2.6,1.2) {$\varepsilon_i$ (noise)};
  \draw[flow] (eps) -- (f);
  \draw[flow] (p1) -- (f);  \draw[flow] (p2) -- (f);  \draw[flow] (p3) -- (f);
  \node (xi) at (3.5,0) {$X_i$};
  \draw[gedge] (f) -- (xi);
  \node[ann, below=5mm of f, text width=30mm]
       {arbitrary deterministic function that is either known or estimated};
  \path (0,1.4) (0,-2.4);  % equal-height strut for the top row
  \cardframe{SEM}{$X_i \boldsymbol{=} f_i\!\parn{\vect{x}_{\mathcal{P}\parn{i}},\,\varepsilon_i}$}
\end{tikzpicture}
&
%==================== CPCM ====================
\begin{tikzpicture}[baseline=0pt]
  \node[pnode] (p1) at (0,0.9)  {};   \node[pnode] (p2) at (0,0.15) {};
  \node        (pd) at (0,-0.45){$\vdots$};   \node[pnode] (p3) at (0,-1.05){};
  \draw[thin] (-0.32,0.95) .. controls (-0.58,0.6) and (-0.58,-0.72) .. (-0.32,-1.1)
       node[midway, left=1pt, font=\footnotesize] {$\vect{x}_{\mathcal{P}\parn{i}}$};
  \node[mapbox] (th) at (1.5,0) {$\boldsymbol{\theta}_i$};
  \draw[flow] (p1) -- (th);  \draw[flow] (p2) -- (th);  \draw[flow] (p3) -- (th);
  \node[genbox] (P) at (4.0,0) {$p_i\!\parn{\cdot \mid \boldsymbol{\theta}_i}$};
  \draw[flow] (th) -- (P);
  \node (xi) at (6.5,0) {$X_i$};
  \draw[gedge] (P) -- node[below=1pt, font=\scriptsize, text=black!55] {generates} (xi);
  \node[ann, below=5mm of th, text width=24mm]
       {non-linear parameter map that is \textcolor{orange!80!black}{estimated from data}};
  \node[ann, below=5mm of P, text width=24mm]
       {\textcolor{green!45!black}{known} conditional family whose parameters are determined by $\boldsymbol{\theta}_i$};
  \path (0,1.4) (0,-2.4);
  \cardframe{CPCM}{$X_i \boldsymbol{\sim} p_i\!\parn{\cdot \mid \boldsymbol{\theta}_i\parn{\vect{x}_{\mathcal{P}\parn{i}}}}$}
\end{tikzpicture}
\\
\end{tabular}

\vspace{6mm}
%==================== LPM ====================
\begin{tikzpicture}
  \node[pnode] (p1) at (0,0.9)  {};   \node[pnode] (p2) at (0,0.15) {};
  \node        (pd) at (0,-0.45){$\vdots$};   \node[pnode] (p3) at (0,-1.05){};
  \draw[thin] (-0.32,0.95) .. controls (-0.58,0.6) and (-0.58,-0.72) .. (-0.32,-1.1)
       node[midway, left=1pt, font=\footnotesize] {$\vect{x}_{\mathcal{P}\parn{i}}$};
  \node[wbox] (w1) at (1.3,0.9)  {$w_{1,i}$};   \node[wbox] (w2) at (1.3,0.15) {$w_{2,i}$};
  \node       (wd) at (1.3,-0.45){$\vdots$};    \node[wbox] (wn) at (1.3,-1.05){$w_{|\mathcal{P}\parn{i}|,i}$};
  \draw[flow] (p1) -- (w1);  \draw[flow] (p2) -- (w2);  \draw[flow] (p3) -- (wn);
  \node[summ] (S) at (3.4,0) {$\Sigma$};
  \draw[flow] (w1) -- (S);  \draw[flow] (w2) -- (S);  \draw[flow] (wn) -- (S);
  \node[mapbox] (g) at (5.4,0) {$\boldsymbol{\theta}_i$};
  \draw[flow] (S) -- node[above=1pt, font=\scriptsize] {$z_i$} (g);
  \node[genbox] (P) at (8.5,0) {$p_i\!\parn{\cdot \mid \boldsymbol{\theta}_i}$};
  \draw[flow] (g) -- (P);
  \node (xi) at (11.4,0) {$X_i$};
  \draw[gedge] (P) -- node[below=1pt, font=\scriptsize, text=black!55] {generates} (xi);
  \node[ann, below=5mm of wn, text width=30mm]
       {trainable weights $\vect{w}_i$};
  \node[ann, below=5mm of g, text width=35mm]
       {\textcolor{green!45!black}{known} deterministic parameter link function};
  \node[ann, below=5mm of P, text width=26mm]
       {\textcolor{green!45!black}{known} regular parametric family, one per node};
  \cardframe{LPM}{$X_i \boldsymbol{\sim} p_i\!\parn{\cdot \mid \boldsymbol{\theta}_i\parn{z_i}}, \quad z_i = \vect{w}_i^{\mathsf{T}}\,\vect{x}_{\mathcal{P}\parn{i}}$}
\end{tikzpicture}

\caption{Generative view of the SEM, CPCM, and LPM frameworks at a single node $i$
with parent set $\mathcal{P}\parn{i}$.}
\label{fig:ssm_toy}
\end{figure}

\subsection{Ordinal-Exponential LPM}
We now focus on the ordinal and regular one-parameter exponential family mixed distribution LPM studied in the remainder of this paper. The node set is partitioned as $\mathcal{V} = \mathcal{V}_{\operatorname{ord}} \cup \mathcal{V}_{\operatorname{exp}}$, where $\mathcal{V}_{\operatorname{ord}}$ represents the set of the ordinal nodes and $\mathcal{V}_{\operatorname{exp}}$ represents that of the exponential-family nodes. We describe both node families below.

\subsubsection{Ordinal Nodes} Let $i \in \mathcal{V}_{\operatorname{ord}}$ correspond to an ordinal random variable $X_i$ with finite ordered support $\mathcal{X}_i = \curl{1, \dots, s}$ for $s \geq 2$. If the node is a root node, i.e. $\mathcal{P}\parn{i} = \emptyset$, the distribution of $X_i$ reduces to a categorical distribution specified by category probabilities $\pi_{i, x}$ as
\begin{eqnarray}
    \label{eq:ord_root}
    p_i\parn{X_i = x} = \pi_{i, x}, \quad \pi_{i, x} > 0, \quad \forall x \in \mathcal{X}_i, \quad \text{and} \quad \sum_{x \in \mathcal{X}_i}\pi_{i,x} = 1.
\end{eqnarray}
For the non-root node case, we adopt the ordered logit cumulative link parametrization of \citet{mccullagh1980regression} given by
\begin{eqnarray}
    p\parn{X_i = x \mid \vect{x}_{\mathcal{P}\parn{i}}} = \sigma\!\parn{\gamma_{i,x} - \vect{w}_i^{\mathsf{T}} \vect{x}_{\mathcal{P}\parn{i}}}  - \sigma\!\parn{\gamma_{i,x-1} - \vect{w}_i^{\mathsf{T}} \vect{x}_{\mathcal{P}\parn{i}}}, \quad \forall x \in \mathcal{X}_i
    \label{eq:ord_log_mod}
\end{eqnarray}
where $\sigma\parn{z} = \parn{1 + e^{-z}}^{-1}$ is the sigmoid function, and $\gamma_{i,j}$ are strictly ordered cutpoints such that $-\infty = \gamma_{i, 0} < \gamma_{i, 1} < \cdots < \gamma_{i, s - 1} < \gamma_{i, s} = +\infty$. The cutpoints are further assumed to be independent of the parent variables $\vect{x}_{\mathcal{P}\parn{i}}$ and the weights $\vect{w}_i$.

\subsubsection{Exponential-Family Nodes} 
Let $i \in \mathcal{V}_{\operatorname{exp}}$ correspond to a random variable $X_i$ whose conditional distribution belongs to a regular one-parameter exponential family. The support $\mathcal{X}_i \subseteq \mathbb{R}$ may be discrete or continuous. For the root node case with $\mathcal{P}\parn{i} = \emptyset$, the marginal distribution of $X_i$ is fixed to be non-degenerate and positive on the support $\mathcal{X}_i$. Therefore, we have,
\begin{eqnarray}
    \label{eq:ef_root}
    p_{X_i}\parn{x} > 0, \quad \forall x \in \mathcal{X}_i, \qquad \text{and} \qquad \sum_{x \in \mathcal{X}_i}p_{X_i}\parn{x} = 1
\end{eqnarray}
If the node has parent nodes, i.e. $\mathcal{P}\parn{i} \neq \emptyset$, then the node follows a regular one-parameter exponential family distribution with known sufficient statistic $T_i: \mathcal{X}_i \to \mathbb{R}$, taken to be minimal in the sense that $T_i$ is not constant on $\mathcal{X}_i$, so that distinct natural parameters index distinct distributions, and where the natural parameter is fixed to follow a known strictly monotone function of the linear predictor of the parents $\eta_i\parn{\vect{w}_{i}^{\mathsf{T}}\vect{x}_{\mathcal{P}\parn{i}}}$, and the conditional distribution takes the form
\begin{eqnarray}
    p_{X_i \mid \vect{x}_{\mathcal{P}\parn{i}}}\parn{x \mid \vect{x}_{\mathcal{P}\parn{i}}} = H_i\parn{x} \exp\!\parn{\eta_i\!\parn{\vect{w}_{i}^{\mathsf{T}}\vect{x}_{\mathcal{P}\parn{i}}} T_i\parn{x} - a_i\!\parn{\eta_i\!\parn{\vect{w}_{i}^{\mathsf{T}}\vect{x}_{\mathcal{P}\parn{i}}}}}, \quad \forall x \in \mathcal{X}_i,
\end{eqnarray}
where $H_i\parn{x} > 0$ is the base measure, and $a_i\parn{\cdot}$ is the log-partition function.

We now proceed to record a fundamental mismatch between the ordinal and the exponential family nodes described in this section.
\begin{prop}
\label{prop:ord_ne}
Let $X$ be ordinal with support $\mathcal{X} = \curl{1, \dots, s}$ for $s \geq 3$, let $Y$ be supported on a set $\mathcal{Y}$ containing at least three points, and let
\begin{equation}
    p\parn{X = x \mid Y = y} = \sigma\!\parn{\gamma_x - w y} - \sigma\!\parn{\gamma_{x-1} - w y}, \quad \forall x \in \mathcal{X}, y \in \mathcal{Y},
    \label{eq:ord_ne_cond}
\end{equation}
for some $w \neq 0$ and strictly ordered cutpoints $-\infty = \gamma_0 < \gamma_1 < \cdots < \gamma_{s-1} < \gamma_s = +\infty$. Then, regarded as a family of distributions on $\mathcal{X}$ indexed by $y$, this conditional is not a one-parameter exponential family in $X$.
\end{prop}
\begin{proof}
See Appendix~\ref{app:ord_ne}.
\end{proof}
{Proposition~\ref{prop:ord_ne} shows that on the ordinal support, the ordered logit generates a family that does not admit a one-parameter exponential family representation. At the same time, this makes the edge orientable within the framework developed here, as we show in the following section.}

\section{Identifiability}
\label{sec:id}
This section provides the theoretical analysis of the population-level identifiability of mixed ordinal-exponential Linear Parametric Models (LPMs) as discussed in Section~\ref{sec:bg}. We begin by analyzing a simple bivariate Directed Acyclic Graph (DAG) setting and then extend our findings to multivariate scenarios.

\subsection{Identifiability in Bivariate Case}
\label{ss:id_bivar}
Consider a two-node causal DAG over random variables $X$ and $Y$, whose single edge admits the two competing orientations $\mathcal{M}_{X \to Y}$ and $\mathcal{M}_{Y \to X}$, referred to respectively as the forward and the reverse model throughout the paper. To establish identifiability, we need to show that the two model classes yield distinct joint distributions, which we formalize as follows.
\begin{defn}[Distributional Identifiability of Edge Direction]
\label{def:d_id}
    The edge direction, in a directed bivariate causal model over random
    variables $X$ and $Y$, is \textbf{distributionally identifiable} if the two orientations induce disjoint sets of joint distributions, i.e. $P_{X, Y} \neq Q_{X, Y}$ for all $P_{X, Y} \in \mathcal{M}_{X \to Y}$ and $Q_{X, Y} \in \mathcal{M}_{Y \to X}$.
\end{defn}

We start by fixing the conditional distribution of $X$ to follow the ordered-logit ordinal model with support $\mathcal{X} = \curl{1, \dots, s}$, and the conditional distribution of $Y$ to belong to a regular one-parameter exponential family with known sufficient statistic $T\parn{y}$ and support $\mathcal{Y}$. With a slight abuse of notation, we write $\mathcal{M}_{X \to Y}$ and $\mathcal{M}_{Y \to X}$ also for the sets of joint laws the two models respectively induce on $\mathcal{X} \times \mathcal{Y}$. In the forward model, the edge weight is $w \neq 0$, and the node corresponding to $X$ is the root node with a categorical marginal:
\begin{equation}
    p_X\parn{x; \mathcal{M}_{X \to Y}} = \pi_{x}, \quad \text{with } \pi_{x} > 0
    \ \ \forall x \in \mathcal{X} \ \text{ and } \ \sum_{x \in \mathcal{X}}\pi_{x} = 1,
    \label{eq:fwd_mar_X}
\end{equation}
and $Y \mid X$ follows a regular one-parameter exponential family distribution, with its conditional distribution defined as
\begin{equation}
    p_{Y \mid X}\parn{y \mid x; \mathcal{M}_{X \to Y}}
    = H\parn{y} \exp\!\parn{\eta\!\parn{wx} T\parn{y} - a\!\parn{\eta\!\parn{wx}}},
    \quad \forall y \in \mathcal{Y},\ x \in \mathcal{X},
    \label{eq:fwd_cond_Y}
\end{equation}
where $\eta\parn{\cdot}$ is a known monotone injective link, $H\parn{\cdot}$ the base measure, and $a\parn{\cdot}$ the log-partition function, so the forward class is
\begin{eqnarray}
    p_{X, Y}\parn{x, y; \mathcal{M}_{X \to Y}} = \pi_xH\parn{y} \exp\!\parn{\eta\!\parn{wx} T\parn{y} - a\!\parn{\eta\!\parn{wx}}},
    \quad \forall y \in \mathcal{Y},\ x \in \mathcal{X}.
    \label{eq:fwd_jnt}
\end{eqnarray}

In the reverse model with edge weight $v \neq 0$, the node corresponding to $Y$ is the root node with its marginal fixed to be non-degenerate and positive on $\mathcal{Y}$,
\begin{equation}
    p_{Y}\parn{y; \mathcal{M}_{Y \to X}} > 0, \quad \forall y \in \mathcal{Y}, \qquad \text{and}
    \qquad \sum_{y \in \mathcal{Y}} p_{Y}\parn{y} = 1,
    \label{eq:rev_mar_Y}
\end{equation}
and $X \mid Y$ follows the ordered-logit ordinal model with cutpoints $\gamma_1 < \dots < \gamma_{s-1}$,
\begin{equation}
    p_{X \mid Y}\parn{x \mid y; \mathcal{M}_{Y \to X}}
    = \sigma\!\parn{\gamma_{x} - vy} - \sigma\!\parn{\gamma_{x-1} - vy},
    \quad \forall x \in \mathcal{X},\ y \in \mathcal{Y},
    \label{eq:rev_cond_X}
\end{equation}
where $\sigma\parn{z} = \parn{1 + e^{-z}}^{-1}$ is the sigmoid function. We adopt the boundary conventions $\gamma_0 = -\infty$ and $\gamma_s = +\infty$, giving $\sigma\parn{\gamma_0 - vy} = 0$ and $\sigma\parn{\gamma_s - vy} = 1$. The reverse class is thus
\begin{eqnarray}
    p_{X, Y}\parn{x, y; \mathcal{M}_{Y \to X}} = p_{Y}\parn{y; \mathcal{M}_{Y \to X}}\parn{\sigma\!\parn{\gamma_{x} - vy} - \sigma\!\parn{\gamma_{x-1} - vy}},
    \quad \forall x \in \mathcal{X},\ y \in \mathcal{Y}.
    \label{eq:rev_jnt}
\end{eqnarray} 

To compare the two classes, we focus on the conditional distribution of $X \mid Y$. In the reverse model \eqref{eq:rev_cond_X} supplies the conditional directly, while in the forward model, Bayes' theorem obtains this conditional from the marginal \eqref{eq:fwd_mar_X} and the $Y \mid X$ conditional \eqref{eq:fwd_cond_Y}. Since a categorical law on $\mathcal{X}$ is determined uniquely by its log-odds relative to a fixed reference category, comparing conditional log-ratios suffices. For any $u, x \in \mathcal{X}$ and any $y \in \mathcal{Y}$, define the conditional \textbf{log-odds-ratio}
\begin{eqnarray}
    \label{eq:log_odds_ratio}
    R\parn{y, u, x; \mathcal{M}} \triangleq \log\parn{\frac{p_{X \mid Y}\parn{u \mid y; \mathcal{M}}}{p_{X \mid Y}\parn{x \mid y; \mathcal{M}}}},
\end{eqnarray}
where $\mathcal{M} \in \curl{\mathcal{M}_{X \to Y}, \mathcal{M}_{Y \to X}}$. The following lemma establishes an important restriction of the forward model.
\begin{lem}
\label{lem:fwd_affine}
In the forward model $\mathcal{M}_{X \to Y}$ with edge weight $w \neq 0$,
where the marginal distribution of $X$ follows a categorical distribution on
support $\mathcal{X} = \curl{1, \dots, s}$ given by \eqref{eq:fwd_mar_X}, and
$Y \mid X$ follows a regular one-parameter exponential family distribution given
by \eqref{eq:fwd_cond_Y}, \textbf{the log-odds-ratio \eqref{eq:log_odds_ratio} is affine
in the sufficient statistic $T\parn{y}$}, i.e.,
\begin{equation}
    R\parn{y, u, x; \mathcal{M}_{X \to Y}} = \alpha\, T\parn{y} + \beta, \quad \forall\, u, x \in \mathcal{X}, u \neq x
    \text{ and } y \in \mathcal{Y},
    \label{eq:forward_lpr}
\end{equation}
where $\alpha \neq 0$ and $\beta$ are independent of $y$ and are given by
\begin{equation}
    \alpha = \eta\parn{wu} - \eta\parn{wx}, \qquad
    \beta = a\parn{\eta\parn{wx}} - a\parn{\eta\parn{wu}}
    + \log\!\parn{\frac{\pi_u}{\pi_x}}.
\end{equation}
\end{lem}
\begin{proof}
See Appendix \ref{app:proof_fwd_affine}.
\end{proof}

We now turn to the reverse model, whose conditional \eqref{eq:rev_cond_X} is strictly positive on $\mathcal{X}$ for every $y \in \mathcal{Y}$ by the strict ordering of the cutpoints and the strict monotonicity of the sigmoid. It is important to note that $T\parn{y}$ carries no distinguished role in the reverse model since $Y$ is not constrained to an exponential family. Therefore, $T\parn{y}$ enters only as a fixed function on $\mathcal{Y}$ for the reverse model. However, if the two models were to induce a common joint law, then Lemma~\ref{lem:fwd_affine} would force every reverse log-odds-ratio to be affine in that same $T\parn{y}$, although nothing but the sigmoid geometry of the cutpoints governs those ratios. We analyze the reverse model by focusing on the two adjacent category pairs whose log-odds ratios admit closed forms with opposite curvature, an important observation that helps settle the affineness argument for the reverse model.
\begin{lem}
\label{lem:rev_bounds}
For the reverse model $\mathcal{M}_{Y \to X}$ with edge weight $v \neq 0$, where $X \mid Y$ follows the ordered-logit ordinal conditional given in \eqref{eq:rev_cond_X} with cutpoints $-\infty = \gamma_0 < \gamma_1 < \cdots < \gamma_{s-1} < \gamma_s = +\infty$ and support $\mathcal{X} = \curl{1, \dots, s}$ where $s \geq 3$, the extreme adjacent log-odds-ratios $B\parn{y} \triangleq R\parn{y, 2, 1; \mathcal{M}_{Y \to X}}$ and $U\parn{y} \triangleq R\parn{y, s, s-1; \mathcal{M}_{Y \to X}}$ admit the closed forms
\begin{equation}
    B\parn{y} = c_B - \log\parn{1 + e^{\gamma_2 - vy}}, \quad
    U\parn{y} = c_U + \log\parn{1 + e^{vy - \gamma_{s-2}}}, \quad
    \forall\, y \in \mathcal{Y},
    \label{eq:b_u_def}
\end{equation}
for finite constants $c_B$ and $c_U$ independent of $y$ given by
\begin{equation}
    c_B = \log\parn{\frac{e^{\gamma_2} - e^{\gamma_1}}{e^{\gamma_1}}}, \quad
    c_U = \log\parn{\frac{e^{\gamma_{s-2}}}{e^{\gamma_{s-1}} - e^{\gamma_{s-2}}}}.
\end{equation}
Moreover, the right-hand sides of \eqref{eq:b_u_def} are defined for every $y \in \mathbb{R}$, thus, for the real line extensions $\bar{B}, \bar{U} : \mathbb{R} \to \mathbb{R}$ such that $B = \bar{B}\big\vert_{\mathcal{Y}}$ and $U = \bar{U}\big\vert_{\mathcal{Y}}$, \textbf{$\bar{B}$ is strictly concave and $\bar{U}$ is strictly convex on $\mathbb{R}$}.
\end{lem}
\begin{proof}
See Appendix \ref{app:proof_rev_bounds}.
\end{proof}

The two lemmas are found to impose incompatible requirements on a single quantity: Lemma~\ref{lem:fwd_affine} confines every forward log-odds-ratio to a straight line in $T\parn{y}$, while Lemma~\ref{lem:rev_bounds} bends the two extreme reverse ones in opposite directions. Eliminating $T\parn{y}$ between the two extremes, as defined in Lemma~\ref{lem:rev_bounds}, leaves a function of $y$ whose shape the cutpoints and the reverse edge weight settle by themselves, and which the elimination forces to be constant on $\mathcal{Y}$, whereas strict curvature permits a constant value at no more than two points. This creates a contradiction, which leads us to the identifiability result formally stated in Theorem~\ref{thm:general_id}.
\begin{thm}
\label{thm:general_id}
Let $\mathcal{M}_{X \to Y}$ be the forward model with edge weight $w \neq 0$, in which $X$ follows the categorical marginal \eqref{eq:fwd_mar_X} on $\mathcal{X} = \curl{1, \dots, s}$ with $s \geq 3$, and $Y \mid X$ follows the regular one-parameter exponential family conditional \eqref{eq:fwd_cond_Y} with sufficient statistic $T\parn{y}$ and support $\mathcal{Y}$. Let $\mathcal{M}_{Y \to X}$ be the reverse model with edge weight $v \neq 0$, in which $Y$ follows the positive marginal \eqref{eq:rev_mar_Y} on $\mathcal{Y}$, and $X \mid Y$ follows the ordered-logit conditional \eqref{eq:rev_cond_X} with cutpoints $-\infty = \gamma_0 < \gamma_1 < \cdots < \gamma_{s-1} < \gamma_s = +\infty$. If the support contains at least three points, $\bars{\mathcal{Y}} \geq 3$, then \textbf{$\mathcal{M}_{X \to Y}$ and $\mathcal{M}_{Y \to X}$ are distributionally identifiable}.
\end{thm}
\begin{proof}
Two proofs are given in Appendix~\ref{app:proof_general_id}, the first proceeding from the opposing curvature of Lemma~\ref{lem:rev_bounds} and the second from Proposition~\ref{prop:ord_ne}.
\end{proof}

{Theorem~\ref{thm:general_id} imposes no condition on the forward model sufficient statistic $T\parn{y}$, beyond the minimality already required in Section~\ref{sec:bg}, so a single result governs every regular one-parameter exponential family, over finite, countable, and continuous supports alike. The LPM setting calls for exactly such generality, since a graph whose exponential-family nodes carry distinct families admits no reduction to a canonical case, and a criterion stated in terms of $T$ would have to be verified edge by edge. The guarantee also holds for every parameter value, rather than only outside an exceptional set, in contrast to \cite{shaska2025ordinal}. It is also important to note that guarantees based on generic identifiability, which prove identifiability outside an exceptional set, cannot be reliably used on real datasets whose parameters are unknown. Such issues are mitigated under distributional identifiability, which holds for every valid parameter in the parameter set.}

{Appendix~\ref{app:proof_general_id} gives two proofs for Theorem~\ref{thm:general_id}, and both of them are important for establishing two distinct ideas. The first proof provides direct intuition for the requirement of $\lvert \mathcal{Y} \rvert \geq 3$, and more importantly, it also paves the way for a sufficient condition on any cumulative link beyond the ordered-logit, as formalized in Proposition~\ref{prop:general_link}. The second proof relies directly on Proposition~\ref{prop:ord_ne} and is the shorter one, exhibiting the asymmetry in its plainest form, at the cost of being tied to the logit through the algebra that establishes that proposition.}
\begin{prop}
\label{prop:general_link}
Let the forward model and the remaining hypotheses of Theorem~\ref{thm:general_id} be unchanged, and let the ordered-logit conditional \eqref{eq:rev_cond_X} of the reverse model be replaced by
\begin{equation}
    p_{X \mid Y}\parn{x \mid y; \mathcal{M}_{Y \to X}} = F\parn{\gamma_{x} - vy} - F\parn{\gamma_{x-1} - vy}, \quad \forall\, x \in \mathcal{X},\ y \in \mathcal{Y},
    \label{eq:gen_link_cond}
\end{equation}
for a continuous and strictly increasing cumulative link $F : \mathbb{R} \to \parn{0, 1}$, the boundary conventions $\gamma_0 = -\infty$ and $\gamma_s = +\infty$ being retained, so that $F\parn{\gamma_0 - vy} = 0$ and $F\parn{\gamma_s - vy} = 1$. Write
\begin{equation}
    \bar{B}\parn{y} = \log \parn{\frac{F\parn{\gamma_2 - vy} - F\parn{\gamma_1 - vy}}{F\parn{\gamma_1 - vy}}}, \quad
    \bar{U}\parn{y} = \log \parn{\frac{1 - F\parn{\gamma_{s-1} - vy}}{F\parn{\gamma_{s-1} - vy} - F\parn{\gamma_{s-2} - vy}}},
    \label{eq:gen_link_bu}
\end{equation}
for the two extreme adjacent log-odds-ratios of \eqref{eq:gen_link_cond}, regarded as functions on all of $\mathbb{R}$. If, for every $v \neq 0$ and every strictly ordered set of cutpoints, $\bar{B}$ is strictly concave and $\bar{U}$ strictly convex on $\mathbb{R}$, then \textbf{$\mathcal{M}_{X \to Y}$ and $\mathcal{M}_{Y \to X}$ are distributionally identifiable}.
\end{prop}
\begin{proof}
See Appendix~\ref{app:proof_general_link}.
\end{proof}

We note that Proposition~\ref{prop:general_link} provides only a sufficient condition, and its necessity remains out of the scope of this paper. Proposition~\ref{prop:general_link}'s hypothesis isolates the single property of the logit that the first proof of Theorem~\ref{thm:general_id} uses, and Lemma~\ref{lem:rev_bounds} shows the logit to possess that property. We investigate in the following subsection whether the thresholds $s \geq 3$ and $\bars{\mathcal{Y}} \geq 3$ are necessary, or merely sufficient for the arguments given.

\subsection{Analysis of Non-Identifiability}
\label{subsec:non_id}

We start by analyzing the necessity of the $\bars{\mathcal{Y}} \geq 3$ condition. Intuitively, when the support of $Y$ is only two points, the constraints linking the two models become too few to force a unique orientation. The following result establishes the necessity of the support condition on the exponential-family node for key cases.
\begin{thm}
\label{thm:converse_two_points}
Let $\mathcal{M}_{X \to Y}$ be the forward model with edge weight $w \neq 0$, in which $X$ follows the categorical marginal \eqref{eq:fwd_mar_X} on $\mathcal{X} = \curl{1, \dots, s}$ with $s \geq 3$, and $Y \mid X$ follows the regular one-parameter exponential family conditional \eqref{eq:fwd_cond_Y} with sufficient statistic $T\parn{y}$ and support $\mathcal{Y}$. Let $\mathcal{M}_{Y \to X}$ be the reverse model with edge weight $v \neq 0$, in which $Y$ follows the positive marginal \eqref{eq:rev_mar_Y} on $\mathcal{Y}$, and $X \mid Y$ follows the ordered-logit conditional \eqref{eq:rev_cond_X} with cutpoints $-\infty = \gamma_0 < \gamma_1 < \cdots < \gamma_{s-1} < \gamma_s = +\infty$. If the support of $Y$ contains exactly two points, $\mathcal{Y} = \curl{y_1, y_2}$ with $y_1 < y_2$, the forward model sufficient statistic is affine in $y$, i.e. $T\parn{y} = c_1 y + c_0$ with $c_0, c_1 \in \mathbb{R}$ and $c_1 \neq 0$, and the link is identity, i.e. $\eta\parn{z} = z$, then there exists a two-parameter family of forward parameters $\parn{\pi, w}$ with $w \neq 0$ and reverse parameters $\parn{\gamma, v}$ with $v \neq 0$ that induce identical joint distributions on $\mathcal{X} \times \mathcal{Y}$. \textbf{The direction of the edge is therefore not distributionally identifiable.}
\end{thm}
\begin{proof}
See Appendix~\ref{app:proof_converse_two_points}
\end{proof}

Theorem~\ref{thm:converse_two_points} shows the requirement $\bars{\mathcal{Y}} \geq 3$ cannot be dropped from Theorem~\ref{thm:general_id}. Its assumptions are weaker than they appear. Since two distinct points are always collinear, every non-constant sufficient statistic is affine on a two-point support. Therefore, the minimality already assumed in Section~\ref{sec:bg} supplies the affineness, leaving the canonical link as the only genuine restriction added. An exponential-family node supported on two points is a binary variable, so the non-identifiable edges are those joining an ordinal node to a binary one, which is a common pairing in practical datasets, highlighting the importance of this result. We provide a construction for deriving parameters which produce identical joint distributions in either direction under the condition that $\bars{\mathcal{Y}} = 2$ in Remark~\ref{rem:reconstruction}. It is important to note that this construction may not be the only way to derive parameters that produce the non-identifiable models. 
\begin{rmk}
\label{rem:reconstruction}
{\textbf{Given} the support $\mathcal{Y} = \curl{y_1, y_2}$ with $y_1 < y_2$, the affine sufficient statistic $T\parn{y} = c_1 y + c_0$ with $c_1 \neq 0$, the canonical link $\eta\parn{z} = z$ with log-partition $a\parn{\cdot}$ and base measure $H\parn{\cdot}$, the number of categories $s \geq 3$, and a forward model edge weight $w \neq 0$ with $w c_1 > 0$ (the case for $w c_1 < 0$ follows by reversing every inequality below).}

{\textbf{We define} the following terms for a reverse model with edge weight $v$ (the value to be determined below),
\begin{equation}
    \Delta y = y_2 - y_1, \qquad \Delta T = c_1 \Delta y, \qquad L = v \Delta y, \qquad t_i = e^{v y_i}, \; i \in \curl{1, 2},
\end{equation}
and, for a sequence $L = D_0 > D_1 > \cdots > D_{s-1} > D_s = 0$, define the cutpoints
\begin{equation}
    \gamma_j = \log\parn{\frac{t_2 - t_1 e^{D_j}}{e^{D_j} - 1}}, \qquad j = 1, \dots, s - 1 .
    \label{eq:rem_gamma}
\end{equation}}

{\textbf{Construction of the reverse parameters $\parn{\gamma, v}$.}
\begin{enumerate}
    \item If $s$ is even, i.e. $\exists$  $m > 0$ such that $s = 2m$: set $v = m c_1 w$, choose any $D_1 \in \parn{\tfrac{m - 1}{m} L,\, L}$, and set
    \begin{equation}
        D_j = L - \tfrac{j}{2m} L \;\; (j \text{ even}), \qquad
        D_j = D_1 - \tfrac{j - 1}{2m} L \;\; (j \text{ odd}), \qquad j = 0, \dots, s .
    \end{equation}
    \item If $s$ is odd, i.e. $\exists$  $m > 0$ such that $s = 2m + 1$: choose any $v \in \parn{m w c_1,\, \parn{m + 1} w c_1}$, and set
    \begin{equation}
        D_j = L - \tfrac{j}{2} w \Delta T \;\; (j \text{ even}), \qquad
        D_j = \parn{m - \tfrac{j - 1}{2}} w \Delta T \;\; (j \text{ odd}), \qquad j = 0, \dots, s .
    \end{equation}
    \item In either case, obtain $\gamma_1 < \cdots < \gamma_{s-1}$ from \eqref{eq:rem_gamma}.
\end{enumerate}}

{\textbf{Construction of the forward marginal $\pi$.} With $\parn{\gamma, v}$ in hand, set for $x \in \mathcal{X}$
\begin{equation}
    l_x = \log\parn{\frac{\sigma\parn{\gamma_x - v y_1} - \sigma\parn{\gamma_{x-1} - v y_1}}{\sigma\parn{\gamma_1 - v y_1}}}, \qquad
    \beta_x = l_x - w \parn{x - 1} T\parn{y_1},
\end{equation}
with $\gamma_0 = -\infty$ and $\gamma_s = +\infty$, and take
\begin{equation}
    \pi_x = \frac{\exp\parn{\beta_x + a\parn{w x} - a\parn{w}}}{\sum_{x' \in \mathcal{X}} \exp\parn{\beta_{x'} + a\parn{w x'} - a\parn{w}}}, \qquad x \in \mathcal{X} .
\end{equation}}

{\textbf{Reverse marginal of $Y$.} Set
\begin{equation}
    p_Y\parn{y_i; \mathcal{M}_{Y \to X}} = \sum_{x \in \mathcal{X}} \pi_x H\parn{y_i} \exp\parn{w x\, T\parn{y_i} - a\parn{w x}}, \qquad i \in \curl{1, 2}.
\end{equation}}

{The forward model $\parn{\pi, w}$ and the reverse model $\parn{\gamma, v}$ so obtained induce the same joint distribution on $\mathcal{X} \times \mathcal{Y}$. The one free scalar ($D_1$ when $s$ is even, $v$ when $s$ is odd) ranges over a non-empty open interval, so the construction yields a continuum of coincident pairs. Conversely, a reverse model $\parn{\gamma, v}$ admits a matching forward model if and only if $D_j = \log\parn{\parn{e^{\gamma_j} + t_2} / \parn{e^{\gamma_j} + t_1}}$ satisfies $D_j - D_{j+2} = D_0 - D_2$ for all $j = 0, \dots, s-2$, in which case the forward weight is uniquely $w = \parn{L - D_2} / \Delta T$ and $\pi$ follows as above.}
\end{rmk}

The strict ordering of the cutpoints is equivalent to the strict decrease of the $D_j$ from $D_0 = L$ to $D_s = 0$, and the constant $D_j - D_{j + 2}$ for any $j \in \curl{0, \dots, s - 2}$ splits that sequence into two interleaved arithmetic chains, each pinned at an endpoint. One scalar survives the count and ranges over a non-empty open interval, so the construction returns a continuum of coincident pairs rather than a single one. Only two scalars are free in total, so the coincident forward models form a two-dimensional subset of the $s$-dimensional forward parameter space, which is Lebesgue-null whenever $s \geq 3$. Theorem~\ref{thm:converse_two_points} therefore establishes that the two model classes fail to be disjoint at two support points, not that the orientation is unrecoverable at a typical parameter value.

We now move to analyze the necessity of the $s \geq 3$ condition on the ordinal nodes. It is interesting to observe that at $s = 2$, the reverse model retains a single cutpoint $\gamma_1$, and the resulting two-category conditional is itself a one-parameter exponential family in $X$, with natural parameter $vy - \gamma_1$. Proposition~\ref{prop:ord_ne} therefore no longer applies, neither form of the argument establishing Theorem~\ref{thm:general_id} survives, and the outcome turns entirely on the sufficient statistic. This intuitively leads to the non-identifiability claim, provided in the following theorem.
\begin{thm}
\label{thm:converse_two_categories}
Let $\mathcal{M}_{X \to Y}$ be the forward model with edge weight $w \neq 0$, in which $X$ follows the categorical marginal \eqref{eq:fwd_mar_X} on $\mathcal{X} = \curl{1, 2}$, and $Y \mid X$ follows the regular one-parameter exponential family conditional \eqref{eq:fwd_cond_Y} with sufficient statistic $T\parn{y}$ and support $\mathcal{Y}$ with $\bars{\mathcal{Y}} \geq 3$. Let $\mathcal{M}_{Y \to X}$ be the reverse model with edge weight $v \neq 0$, in which $Y$ follows the positive marginal \eqref{eq:rev_mar_Y} on $\mathcal{Y}$, and $X \mid Y$ follows the ordered-logit conditional \eqref{eq:rev_cond_X} with the single cutpoint $\gamma_1$. Then \textbf{$\mathcal{M}_{X \to Y}$ and $\mathcal{M}_{Y \to X}$ are distributionally identifiable if and only if the sufficient statistic is non-affine on the support}, that is, if and only if there exist $y_1, y_2, y_3 \in \mathcal{Y}$ for which $\parn{T\parn{y_2} - T\parn{y_1}}/\parn{y_2 - y_1} \neq \parn{T\parn{y_3} - T\parn{y_1}}/\parn{y_3 - y_1}$. In particular, if $T\parn{y} = c_1 y + c_0$ is affine, then for every forward model, there exist reverse parameters $\parn{\gamma_1, v}$ with $v \neq 0$ that induce the same joint distribution, and the models are not identifiable.
\end{thm}
\begin{proof}
See Appendix~\ref{app:proof_converse_two_categories}
\end{proof}

Theorem~\ref{thm:converse_two_categories} shows that instead of a full non-identifiability result, we arrive at a dichotomy. The condition of $s \geq 3$ remains necessary when the sufficient statistic is affine and dispensable otherwise, so the requirement of three categories belongs to the statistic rather than to the ordinal variable alone. A third category then supplies a second cutpoint, and with it both the curvature of Lemma~\ref{lem:rev_bounds} and the applicability of Proposition~\ref{prop:ord_ne}, which is why $T$ is unconstrained once $s \geq 3$. We provide a construction for deriving model parameters such that both the forward and reverse models induce the same joint distribution for affine $T$, as noted in Remark~\ref{rem:reconstruction_two_categories}. We note that Theorem~\ref{thm:converse_two_points} leaves a Lebesgue-null set of coincident forward models, whereas at $s = 2$ with affine $T$ every forward model admits a matching reverse one, so the failure is complete rather than exceptional. For instance, the Poisson statistic is affine, so an Ordinal--Poisson edge is not identifiable at two ordinal categories, and a binary ordinal node is the one case in which the guarantees of Section~\ref{ss:id_bivar} require checking the statistic first.

\begin{rmk}
\label{rem:reconstruction_two_categories}
{\textbf{Given} the support $\mathcal{Y}$, the affine sufficient statistic $T\parn{y} = c_1 y + c_0$ with $c_1 \neq 0$, the monotone injective link $\eta\parn{\cdot}$ with log-partition $a\parn{\cdot}$ and base measure $H\parn{\cdot}$, the number of categories $s = 2$, and forward parameters $\parn{\pi, w}$ with $\pi_1, \pi_2 > 0$, $\pi_1 + \pi_2 = 1$, and $w \neq 0$.}

{\textbf{We define} the slope and intercept of the forward log-odds-ratio,
\begin{equation}
    \alpha \triangleq \eta\parn{2w} - \eta\parn{w} \neq 0, \qquad
    \beta \triangleq a\parn{\eta\parn{w}} - a\parn{\eta\parn{2w}} + \log\parn{\frac{\pi_2}{\pi_1}} .
    \label{eq:rem2_alpha_beta}
\end{equation}}

{\textbf{Construction of the reverse parameters $\parn{\gamma_1, v}$.}
\begin{enumerate}
    \item Set $v = \alpha c_1$, which is nonzero since $\alpha \neq 0$ and $c_1 \neq 0$.
    \item Set $\gamma_1 = -\parn{\alpha c_0 + \beta}$.
\end{enumerate}}

{\textbf{Reverse marginal of $Y$.} Set
\begin{equation}
    p_Y\parn{y; \mathcal{M}_{Y \to X}} = \sum_{x \in \curl{1, 2}} \pi_x H\parn{y} \exp\parn{\eta\parn{w x}\, T\parn{y} - a\parn{\eta\parn{w x}}}, \qquad y \in \mathcal{Y}.
\end{equation}}

{The forward model $\parn{\pi, w}$ and the reverse model $\parn{\gamma_1, v}$ so obtained induce the same joint distribution on $\mathcal{X} \times \mathcal{Y}$. Unlike Remark~\ref{rem:reconstruction}, no condition is imposed on the forward parameters, so every forward model admits a matching reverse model. Conversely, given reverse parameters $\parn{\gamma_1, v}$ with $v \neq 0$ and the reverse marginal of $Y$ set as above, the forward parameters are recovered by solving $\eta\parn{2w} - \eta\parn{w} = v / c_1$ for $w$, which gives $w = v / c_1$ under the canonical link, and then setting
\begin{equation}
    \frac{\pi_2}{\pi_1} = \exp\parn{-\gamma_1 - \alpha c_0 - a\parn{\eta\parn{w}} + a\parn{\eta\parn{2w}}}, \qquad \pi_1 + \pi_2 = 1 .
\end{equation}}
\end{rmk}

Theorem~\ref{thm:general_id} through Theorem~\ref{thm:converse_two_categories} settle the bivariate case, and the classification they give is fixed by the model specification rather than by the parameter values. The number of categories, the size of the support, and the shape of the sufficient statistic are all determined once the node families are chosen, so an edge can be placed before any data are seen. 

\subsection{Multivariate Identifiability}
\label{subsec:multivariate}

The results established so far concern a single edge in isolation, forming the bivariate DAG. We now analyze multivariate mixed ordinal--exponential-family DAGs in which each node follows the LPM model. We begin by stating what identifiability of an edge direction means in the multivariate setting.

%%%%%%%%%%%%%%%%%%%%%%%%%%%% MULTIVARIATE REDUCTION %%%%%%%%%%%%%%%%%%%%%%%%
\begin{defn}[Edge Direction Identifiability in Multivariate DAGs]
\label{def:g_id}
Let $\mathcal{G} = \parn{\mathcal{V}, \mathcal{E}}$ be a causal DAG with undirected skeleton $\mathcal{G}_u = \parn{\mathcal{V}, \mathcal{E}_u}$, and let $\parn{i, j} \in \mathcal{E}_u$. Denote by $\mathcal{G}^{i \to j}$ and $\mathcal{G}^{j \to i}$ the two directed graphs that agree with $\mathcal{G}$ on every edge of $\mathcal{E}_u \setminus \curl{\parn{i, j}}$ and orient the edge $\parn{i, j}$ as $i \to j$ and as $j \to i$ respectively, both assumed acyclic, and let $\mathcal{M}\parn{\mathcal{G}^{i \to j}}$ and $\mathcal{M}\parn{\mathcal{G}^{j \to i}}$ denote the sets of joint distributions over $\mathcal{V}$ induced by them. The direction of the edge $\parn{i, j}$ is \textbf{distributionally identifiable} if
\begin{equation}
    p_{\mathcal{V}} \neq q_{\mathcal{V}}, \quad
    \forall\, p_{\mathcal{V}} \in \mathcal{M}\parn{\mathcal{G}^{i \to j}},\;
    q_{\mathcal{V}} \in \mathcal{M}\parn{\mathcal{G}^{j \to i}}.
    \label{eq:g_id_disjoint}
\end{equation}
\end{defn}

We observe that Definition~\ref{def:g_id} reduces to Definition~\ref{def:d_id} when $\mathcal{V} = \curl{i, j}$, with $\mathcal{M}\parn{\mathcal{G}^{i \to j}}$ and $\mathcal{M}\parn{\mathcal{G}^{j \to i}}$ becoming $\mathcal{M}_{X \to Y}$ and $\mathcal{M}_{Y \to X}$, so the bivariate results of the preceding subsections are the two-node case of the theory below. It is intuitive to see that within a causal DAG, an ordinal--exponential-family edge is embedded among further nodes, and both endpoints may carry additional parents whose contributions enter the conditional laws. Those contributions shift the cutpoints of the ordinal endpoint and the natural parameter of the other, but leave both conditional forms intact, so the pair is again an instance of the bivariate setting. The reverse class available within a graph is also contained in the bivariate reverse class, since the graph constrains a reverse-oriented exponential-family endpoint that the bivariate model leaves free. We formalize this intuition in the next theorem, which establishes that the conditions that secure identifiability in the bivariate setting also suffice for an arbitrary graph.

\begin{thm}
\label{thm:multivariate_id}
Let $\mathcal{G} = \parn{\mathcal{V}, \mathcal{E}}$ be a causal DAG on $d$ nodes with undirected skeleton $\mathcal{G}_u = \parn{\mathcal{V}, \mathcal{E}_u}$, in which each node is designated either ordinal or exponential family and every edge weight is nonzero. Let each non-root ordinal node $k$ with $s_k$ categories follow the ordered-logit conditional \eqref{eq:rev_cond_X} given its parents with cutpoints $-\infty = \gamma_0 < \gamma_1 < \cdots < \gamma_{s_k - 1} < \gamma_{s_k} = +\infty$, let each non-root exponential family node $k$ follow the regular one-parameter exponential family conditional \eqref{eq:fwd_cond_Y} given its parents with sufficient statistic $T_k$, monotone injective link $\eta$, and support $\mathcal{X}_k$, and let each root node $k$ follow an arbitrary strictly positive marginal on its support, categorical on $\curl{1, \dots, s_k}$ when $k$ is ordinal. Let the joint distribution $p_{\mathcal{V}}$ be Markov with respect to $\mathcal{G}$, factorizing as $p_{\mathcal{V}} = \prod_{k \in \mathcal{V}} p_{X_k \mid \vect{x}_{\mathcal{P}\parn{k}}}$, with $\mathcal{G}$ causally sufficient, containing no latent confounders and no selection variables. Let $\parn{i, j} \in \mathcal{E}_u$ join an ordinal node $i$ with $s$ categories to an exponential family node $j$ with sufficient statistic $T$ and support $\mathcal{X}_j$. If the ordinal node has at least three categories, $s \geq 3$, and the exponential family node has at least three points of support, $\bars{\mathcal{X}_j} \geq 3$, then \textbf{the two orientations of the edge $\curl{i, j}$ induce disjoint sets of joint distributions on $\mathcal{V}$, and the direction of the edge is distributionally identifiable}.
\end{thm}
\begin{proof}
See Appendix~\ref{app:proof_multivariate_id}.
\end{proof}

Theorem~\ref{thm:multivariate_id} implicitly assumes the Markov property, causal sufficiency, and the absence of selection, but not faithfulness, and the relaxation is substantive. Constraint-based procedures read orientation off conditional independencies and so require the graph to entail every independence the distribution exhibits; the argument here compares joint distributions directly and tests no independence, remaining untouched at parameter values where path cancellation produces an independence $\mathcal{G}$ does not imply. The exemption covers orientation alone, since recovering the skeleton by conditional independence testing still requires adjacency faithfulness \citep{spirtes2000causation, ramsey2006adjacency} or a similar condition, and the skeleton is taken as given throughout.

The orientation is also local, resting on the conditional law of the two endpoints given their parents, so neither endpoint needs to be a source node, no topological ordering is required, and nothing is propagated from one edge to another. Locality distinguishes the guarantee from orientation rules that follow a v-structure and depend on what the rest of the skeleton contains. It is important to note that locality is achieved precisely because the root marginals of both the ordinal and exponential-family nodes are maximally free as described in \eqref{eq:ord_root} and \eqref{eq:ef_root}, allowing subsequent constraints to produce subsets that inherit the properties derived for the original root node cases. Definition~\ref{def:g_id}, however, compares two graphs agreeing everywhere but on a single edge, so Theorem~\ref{thm:multivariate_id} rules out one reversal at a time, and the following result extends the conclusion to the whole orientation space.

\begin{thm}
\label{thm:orient_id}
Let $\mathcal{G}_u = \parn{\mathcal{V}, \mathcal{E}_u}$ be an undirected skeleton in which every edge joins an ordinal node to an exponential family node, and let $\mathcal{O}\parn{\mathcal{G}_u}$ denote its acyclic orientations. Let $\mathcal{G} \in \mathcal{O}\parn{\mathcal{G}_u}$, and let $p_{\mathcal{V}}$ be generated by an ordinal-exponential LPM on $\mathcal{G}$ in which every ordinal node carries at least three categories, every exponential family node at least three points of support, every link $\eta$ is monotone injective, and every edge weight is nonzero. Then \textbf{for every $\mathcal{G}' \in \mathcal{O}\parn{\mathcal{G}_u}$ with $\mathcal{G}' \neq \mathcal{G}$, no ordinal-exponential LPM on $\mathcal{G}'$ generates $p_{\mathcal{V}}$, irrespective of how many edges the two orientations differ in.}
\end{thm}
\begin{proof}
See Appendix~\ref{app:proof_orient_id}.
\end{proof}

Theorem~\ref{thm:orient_id} strengthens Theorem~\ref{thm:multivariate_id} from a single reversal to the entire orientation space, and the strengthening does not follow from repeated application of the earlier result: a competing orientation may differ in many edges at once, and the comparisons cannot be chained through intermediate graphs, none of which is the truth. Theorem~\ref{thm:orient_id} instead excludes every member of $\mathcal{O}\parn{\mathcal{G}_u}$ simultaneously.

We note that the result concerns distributions rather than any particular score, although a consequence for score-based estimation is immediate: for a likelihood-based criterion, Theorem~\ref{thm:orient_id} gives strictly positive Kullback--Leibler divergence from $p_{\mathcal{V}}$ at every competing orientation and every parameter value, so the truth is the unique population optimum. Uniqueness is weaker than consistency, which would additionally require that the divergence remain bounded away from zero as the competing parameters range over an unbounded set, a condition that rests on a compactness argument we do not pursue herein. Whether the two conditions of Theorem~\ref{thm:multivariate_id} remain necessary once the edge sits in a graph is settled next.

\begin{thm}
\label{thm:multivariate_converse}
Let $\mathcal{G} = \parn{\mathcal{V}, \mathcal{E}}$ be a causal DAG on $d$ nodes with undirected skeleton $\mathcal{G}_u = \parn{\mathcal{V}, \mathcal{E}_u}$, in which each node is designated either ordinal or exponential family and every edge weight is nonzero. Let each non-root ordinal node $k$ with $s_k$ categories follow the ordered-logit conditional \eqref{eq:rev_cond_X} given its parents with cutpoints $-\infty = \gamma_0 < \gamma_1 < \cdots < \gamma_{s_k - 1} < \gamma_{s_k} = +\infty$, let each non-root exponential family node $k$ follow the regular one-parameter exponential family conditional \eqref{eq:fwd_cond_Y} given its parents with sufficient statistic $T_k$, monotone injective link $\eta$, and support $\mathcal{X}_k$, and let each root node $k$ follow an arbitrary strictly positive marginal on its support, categorical on $\curl{1, \dots, s_k}$ when $k$ is ordinal. Let the joint distribution $p_{\mathcal{V}}$ be Markov with respect to $\mathcal{G}$, factorizing as $p_{\mathcal{V}} = \prod_{k \in \mathcal{V}} p_{X_k \mid \vect{x}_{\mathcal{P}\parn{k}}}$, with $\mathcal{G}$ causally sufficient, containing no latent confounders and no selection variables. Let $\parn{i, j} \in \mathcal{E}_u$ join an ordinal node $i$ with $s$ categories to an exponential family node $j$ with sufficient statistic $T$ and support $\mathcal{X}_j$, where neither $i$ nor $j$ is adjacent in $\mathcal{G}_u$ to any node other than the other endpoint. If either the exponential family node has exactly two points of support, $\bars{\mathcal{X}_j} = 2$, with affine sufficient statistic $T\parn{y} = c_1 y + c_0$, $c_0, c_1 \in \mathbb{R}$, $c_1 \neq 0$, and canonical link $\eta\parn{z} = z$, or the ordinal node has exactly two categories, $s = 2$, with $T$ affine on $\mathcal{X}_j$ in the same sense, then \textbf{the direction of the edge $\parn{i, j}$ is not distributionally identifiable}.
\end{thm}
\begin{proof}
See Appendix~\ref{app:proof_multivariate_converse}.
\end{proof}

Theorem~\ref{thm:multivariate_converse} shows that neither condition of Theorem~\ref{thm:multivariate_id} can be dropped, under the condition that the true DAG has isolated ordinal--exponential-family edges. When the edge is not isolated, the remaining parents of the endpoints impose constraints that the bivariate construction need not satisfy, and whether identifiability is restored at such an edge is beyond the scope of this paper. The converse regime is correspondingly narrow, requiring a two-node connected component of the skeleton together with either a two-point support or a binary ordinal node, and an affine sufficient statistic in either case. The conditions on the number of ordinal categories and the size of the support are therefore sufficient in an arbitrary causal DAG, and necessary as well for an edge whose endpoints carry no neighbors other than each other.

\section{Numerical Results and Discussion}
\label{sec:num}
\label{sec:three_node}

We illustrate the identifiability results of Section~\ref{sec:id} on synthetic data generated from the models of Section~\ref{sec:bg}. Since Definition~\ref{def:g_id} compares graphs sharing a skeleton, the estimation problem the theory poses is the orientation of a known bipartite skeleton rather than its recovery, and no procedure below inserts or deletes an edge. The graph considered here is sufficiently small that an exhaustive search yields an exact solution to the score minimization, so the recovery error reflects only finite-sample noise. The large-graph regime, in which greedy search supplies a scalable but heuristic estimator, is reported in Appendix~\ref{app:d_node}. {Our main aim for this section is to show that DAGs are identifiable beyond the MEC, in contrast to non-identifiable settings such as linear SEMs, where observational discovery is limited to the MEC \citep{spirtes2000causation, chickering2002optimal, peters2014identifiability}.} 
%\um{a feature we gain due to the LPM model and not true for SEMs}. 
% RESPONSE: Cannot write that as SEM is a very wide array of models and there is no paper out there which disproves identifiability for SEMs. The dangerous ones are the non-linear non-parametric NN based SEMs which will lead to SHD almost zero but no one can prove or disprove identifiability in those cases. Adding this line will invite a baseline request that will just be a disaster to handle.
Additionally, we note that estimation is not a central focus of this work. The estimators used in this section and the corresponding appendices are either exhaustive search-based or greedy heuristics, and are used solely to demonstrate the theoretical population-level identifiability results, which remain the paper's core focus.

Let $\vect{X} \in \mathbb{R}^{n \times d}$ be an $n$-sample dataset generated by an unknown $d$-node ground truth DAG $\mathcal{G}$, and let $\mathcal{P}_{\mathcal{G}_{\operatorname{est}}}\parn{i}$ be the parent set of node $i$ under an estimated graph $\mathcal{G}_{\operatorname{est}}$ with weighted adjacency matrix $\vect{W}_{\operatorname{est}}$. Candidates are scored by the Bayesian Information Criterion \citep{chickering2002optimal},
\begin{eqnarray}
\operatorname{BIC}\parn{\mathcal{G}_{\operatorname{est}}; \vect{X}}
 \triangleq -\sum_{j = 1}^{n} \sum_{i = 1}^{d} \ln p\parn{\vect{X}_{j, i} \mid \vect{X}_{j, \mathcal{P}_{\mathcal{G}_{\operatorname{est}}}\parn{i}}; \vect{W}_{\operatorname{est}}, \boldsymbol{\gamma}_i} + \, \frac{\log n}{2} \sum_{i = 1}^{d} k_i,
\label{eq:bic}
\end{eqnarray}
which is to be minimized over the acyclic orientations $\mathcal{O}\parn{\mathcal{G}_u}$ of the supplied skeleton, with $k_i$ the number of parameters fitted at node $i$. Both terms are sums over nodes, so the score decomposes into node contributions, and the parameters of each candidate are estimated per node using conditional maximum likelihood. Additional details on estimation algorithms, parameter values, and the simulation code repository are provided in Appendix~\ref{app:alg} and Appendix~\ref{app:per_node_mle}, with experimentation details in Appendix~\ref{app:exp_details}.

Every candidate carries the skeleton $\mathcal{G}_u$ of the truth, so an estimate can differ from the truth only in the direction assigned to an edge, and we measure recovery by the orientation error rate
\begin{equation}
\rho\parn{\mathcal{G}_{\operatorname{est}}, \mathcal{G}} \triangleq \frac{1}{\bars{\mathcal{E}_u}} \bars{\curl{\parn{i, j} \in \mathcal{E}_u \; : \; \mathcal{G}_{\operatorname{est}} \text{ and } \mathcal{G} \text{ orient } \parn{i, j} \text{ oppositely}}},
\label{eq:nshd}
\end{equation}
which measures the fraction of skeleton edges assigned the wrong direction. With the skeleton fixed, every discrepancy is a reversal, so $\rho$ is the structural Hamming distance normalized by edge count, lies in $\sqr{0, 1}$, and is comparable across skeletons of different size and density. An estimated graph that is Markov-equivalent to the truth but oriented differently has strictly positive $\rho$, so the figures below report identifiability \emph{within} Markov equivalence classes rather than \emph{up to} them.

\begin{figure}[!htbp]
\centering
\subfigure[$\mathcal{G}_1$: $X_1 \to X_2 \to X_3$ (chain)]{
\includegraphics[width=0.475\linewidth]{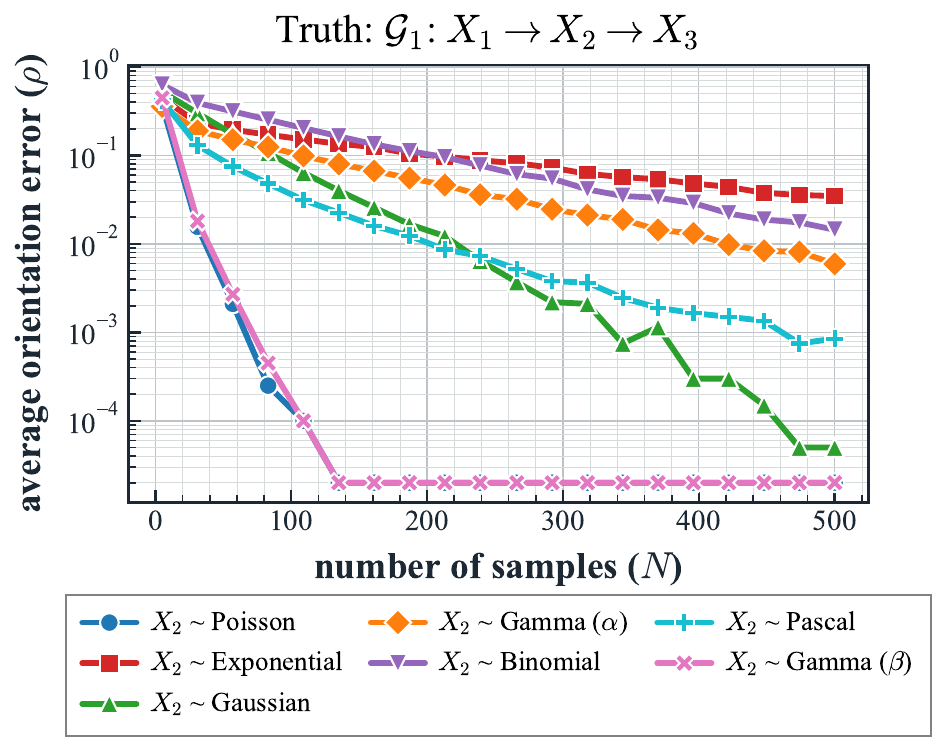}
\label{fig:g1_nshd}
}
\hfill
\subfigure[$\mathcal{G}_2$: $X_1 \to X_2 \leftarrow X_3$ (inverse fork)]{
\includegraphics[width=0.475\linewidth]{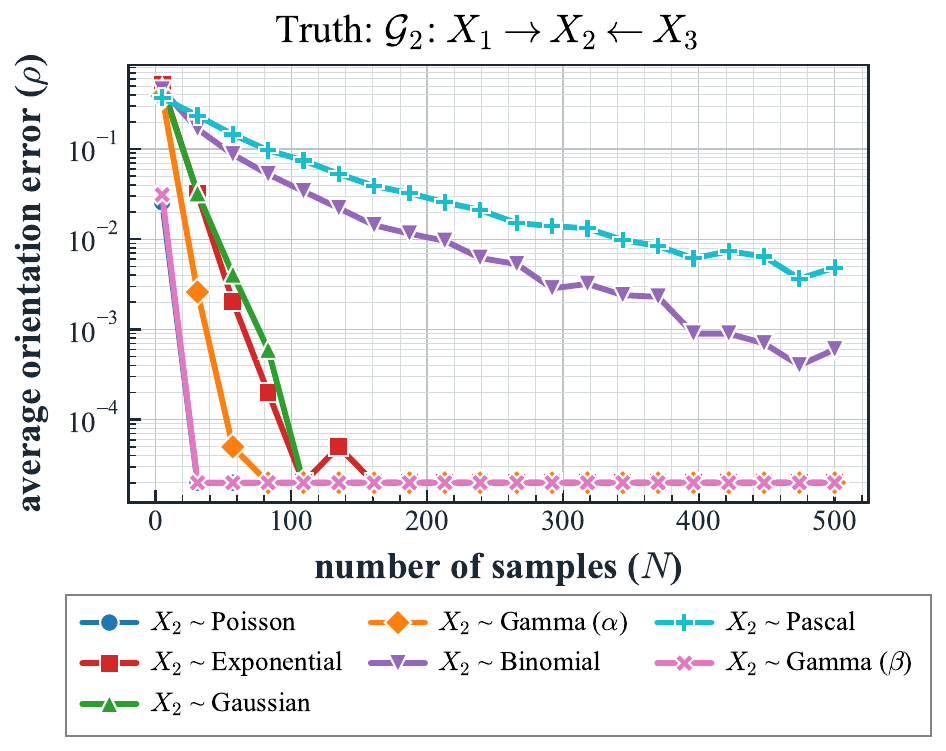}
\label{fig:g2_nshd}
}
\hfill
\subfigure[$\mathcal{G}_3$: $X_1 \leftarrow X_2 \to X_3$ (fork)]{
\includegraphics[width=0.475\linewidth]{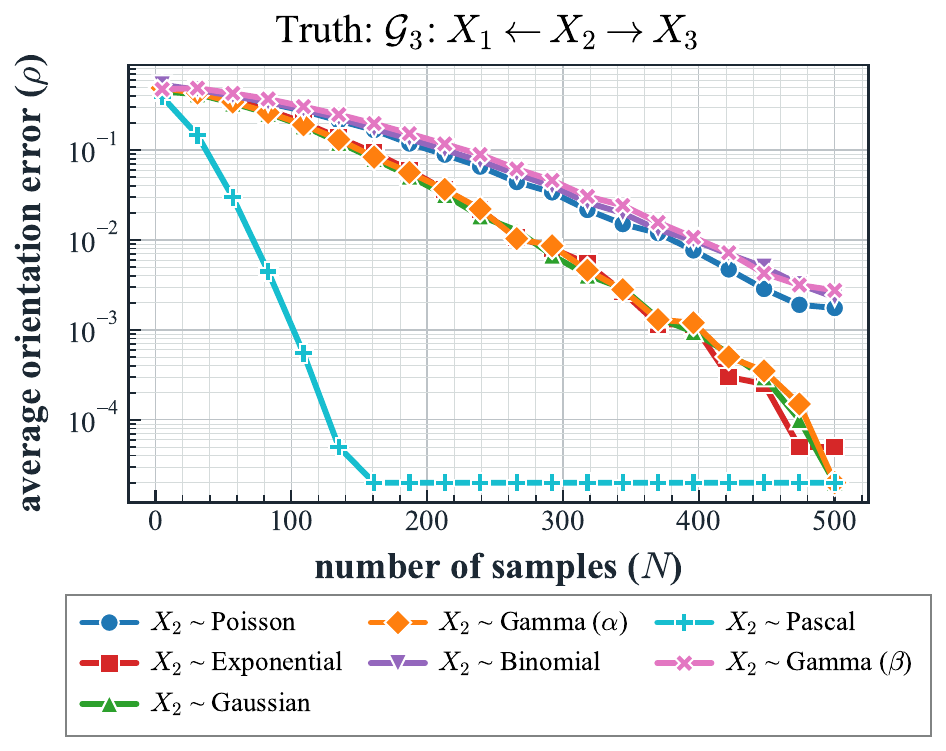}
\label{fig:g3_nshd}
}
\caption{Orientation error rate $\rho$ versus sample size $N$ for the three-node bipartite DAGs,
with $X_2$ drawn from each family of Table~\ref{tab:exp_fam_dists}, over $B = 10000$ trials.}
\label{fig:three_node_nshd}
\end{figure}

We consider three ground-truth bipartite DAGs over $\curl{X_1, X_2, X_3}$ with $X_1, X_3 \in \mathcal{V}_{\operatorname{ord}}$ and $X_2 \in \mathcal{V}_{\operatorname{exp}}$:
\begin{align}
&\mathcal{G}_1 \text{ (chain):}\quad X_1 \to X_2 \to X_3, \nonumber \\
&\mathcal{G}_2 \text{ (inverse fork):}\quad X_1 \to X_2 \leftarrow X_3, \nonumber \\
&\mathcal{G}_3 \text{ (fork):}\quad X_1 \leftarrow X_2 \to X_3.
\end{align}
The chain and the fork share a skeleton and have no v-structures, so both lie in one Markov equivalence class, while the inverse fork has a v-structure at $X_2$ and lies in another. Exhaustive search runs over the four acyclic orientations of $X_1 - X_2 - X_3$, exactly the set Theorem~\ref{thm:orient_id} separates, so any deviation from the truth is finite-sample noise rather than optimization error. {The four do not carry equal parameter counts: summing $k_i$ gives the edge count and the ordinal cutpoint totals, both fixed across orientations, plus the number of exponential-family nodes left without parents (which are treated as free variables as they can have an arbitrary marginal as given in \eqref{eq:ef_root}), so the fork carries one parameter more than the rest and the penalty in \eqref{eq:bic} charges accordingly.}

{Figure~\ref{fig:three_node_nshd} shows $\rho$ against $N$ for each structure, with $X_2$ drawn in turn from each of the seven families of Table~\ref{tab:exp_fam_dists}. Every curve decays with $N$, the flat asymptotes being the numerical floor of the log plot, at which no trial recovered an incorrect orientation. The three structures differ in the orientation error decay rates 
%\um{do these different decay rates make sense, are they predictable?}
%RESPONSE: No, its hard to quantify them here as exhaustive search is a discrete method and we have no mathematical result for BIC consistency (out of scope for this paper; I am deliberately keeping estimation as out of scope here and the focus is on population level identifaibility.
the inverse fork is fastest, reaching the floor for every family but Binomial and Pascal; the chain reaches the floor for Poisson and Gamma~$\parn{\beta}$ alone, with Exponential the slowest curve in the figure; and the fork moves least at the smallest sample sizes and accelerates as sample sizes get larger, reaching the floor for Pascal, Gaussian, and Gamma~$\parn{\alpha}$. The rate of descent depends on the criterion at finite $N$ and on the generating parameters of Appendix~\ref{app:exp_details}, which cannot be justified from population-level identifiability alone. 
%\um{awkward and too longe, break up?} 
We see in Figure~\ref{fig:three_node_nshd} that both the fork and the chain structures can be distinguished despite being in the same MEC.  That the error decays with the number of observations  and is zero for some of the families verifies the result that Theorem~\ref{thm:multivariate_id} establishes in the population law. The rapid convergence of the inverse fork instead reflects the ease of detecting a v-structure, {as it is the only member of its MEC,} 
%\um{why easier?} 
and serves as a check that the search favors no particular structure.}

\section{{Conclusions}}
\label{sec:conc}
{This paper establishes the identifiability of causal DAGs whose nodes follow either the ordered-logit ordinal model or a regular one-parameter exponential family, within the Linear Parametric Model framework defined herein. Two features distinguish the guarantee obtained: it holds at every parameter value rather than outside an exceptional set, and it does not constrain the sufficient statistic -- governing finite, countable, and continuous supports alike. We also provide converses highlighting the limits of the identifiability results and the necessity of the conditions; a multivariate extension that carries the conclusion to every such edge of a $d$-node DAG under the Markov property, causal sufficiency, and the absence of selection. Our numerical experiments confirm that the orientation error rate decays as the sample size increases within a Markov equivalence class. Future work includes consistency guarantees for likelihood-based estimators under the LPM, edges joining two exponential-family nodes, multi-parameter families with unknown nuisance parameters, and cumulative links beyond the logit.}

\acks{We are grateful to Christine K. Johnson and Yingying Wang of the University of California, Davis, for discussions of epidemiological practice, and to Joni Shaska for discussions of ordinal links and for early assistance with the simulation setup.

This work has been funded by one or all of the following grants: ARO W911NF1910269, ARO W911NF2410094, ONR N00014-22-1-2363, NSF CIF-2311653, NSF CIF-2148313, NSF RINGS-2148313, NSF DBI-2412522, and is also supported in part by funds from federal agencies and industry partners as specified in the RINGS program.}
\newpage

\appendix

\section{Proofs}
\subsection{Proof of Proposition~\ref{prop:ord_ne}}
\label{app:ord_ne}
\begin{proof}
We proceed by contradiction. Assume the conditional distribution $p_{X \mid Y}\parn{\cdot \mid \cdot}$ is a one-parameter exponential family in $X$. Then there exist a base measure $\curl{h_x}_{x = 1}^{s}$ with $h_x > 0$, a sufficient statistic $\curl{T_x}_{x = 1}^{s}$, a scalar natural parameter $\theta\parn{y}$, and a log-partition function $A$ such that
\begin{equation}
    p\parn{X = x \mid Y = y}
    = h_x \exp\!\parn{\theta\parn{y} T_x - A\parn{\theta\parn{y}}},
    \quad \forall\, x \in \mathcal{X},\ y \in \mathcal{Y}.
    \label{eq:ef_hyp}
\end{equation}
Fix distinct $x, x' \in \mathcal{X}$. Taking the ratio in \eqref{eq:ef_hyp} cancels the log-partition term, and taking logarithms gives
\begin{equation}
    \log \frac{p\parn{x \mid y}}{p\parn{x' \mid y}}
    = \log \frac{h_x}{h_{x'}} + \theta\parn{y} \parn{T_x - T_{x'}},
    \quad \forall\, y \in \mathcal{Y},
    \label{eq:ef_pair}
\end{equation}
so every pairwise log-ratio is affine in the scalar $\theta\parn{y}$, with slope $T_x - T_{x'}$ independent of $y$. Let $u, v, r$ be pairwise distinct with $T_v \neq T_r$. Writing \eqref{eq:ef_pair} for the pairs $\parn{u, r}$ and $\parn{v, r}$ and eliminating $\theta\parn{y}$ between them,
\begin{equation}
    \log \frac{p\parn{u \mid y}}{p\parn{r \mid y}}
    = \frac{T_u - T_r}{T_v - T_r}
      \log \frac{p\parn{v \mid y}}{p\parn{r \mid y}} + c,
    \quad \forall\, y \in \mathcal{Y},
    \label{eq:ef_affine}
\end{equation}
with $c$ independent of $y$. We exhibit three categories for which \eqref{eq:ef_affine} fails.
 
Take the categories $\curl{1, 2, s}$ and set $z \triangleq wy$, strictly monotone in $y$ since $w \neq 0$, so that distinct points of $\mathcal{Y}$ give distinct values of $z$. From \eqref{eq:ord_ne_cond} with the boundary conventions $e^{\gamma_0} = 0$ and $e^{\gamma_s} = \infty$,
\begin{equation}
    p\parn{1 \mid y} = \frac{e^{\gamma_1}}{e^{\gamma_1} + e^{z}},
    \quad
    p\parn{s \mid y} = \frac{e^{z}}{e^{\gamma_{s-1}} + e^{z}},
    \quad
    p\parn{2 \mid y} = \frac{\parn{e^{\gamma_2} - e^{\gamma_1}} e^{z}}
                            {\parn{e^{\gamma_1} + e^{z}} \parn{e^{\gamma_2} + e^{z}}},
    \label{eq:ef_probs}
\end{equation}
where $e^{\gamma_2} - e^{\gamma_1} > 0$ by the strict ordering of the cutpoints. Taking category $1$ as reference, define
\begin{equation}
    \ell_{21}\parn{z} \triangleq \log \frac{p\parn{2 \mid y}}{p\parn{1 \mid y}},
    \qquad
    \ell_{s1}\parn{z} \triangleq \log \frac{p\parn{s \mid y}}{p\parn{1 \mid y}}.
    \label{eq:ef_ell_def}
\end{equation}
Substituting \eqref{eq:ef_probs} into \eqref{eq:ef_ell_def},
\begin{align}
    \ell_{21}\parn{z} &= \log\!\parn{e^{\gamma_2} - e^{\gamma_1}} - \gamma_1 + z
                         - \log\!\parn{e^{\gamma_2} + e^{z}}, \nonumber \\
    \ell_{s1}\parn{z} &= z + \log\!\parn{e^{\gamma_1} + e^{z}} - \gamma_1
                         - \log\!\parn{e^{\gamma_{s-1}} + e^{z}},
    \label{eq:ef_ell}
\end{align}
both differentiable on $\mathbb{R}$, with
\begin{equation}
    \ell_{21}'\parn{z} = \sigma\!\parn{\gamma_2 - z} > 0,
    \qquad
    \ell_{s1}'\parn{z} = \sigma\!\parn{\gamma_{s-1} - z} + \sigma\!\parn{z - \gamma_1} > 0.
    \label{eq:ef_derivs}
\end{equation}
Since $\ell_{21}'\parn{z} > 0$, the map $z \mapsto \ell_{21}\parn{z}$ is strictly increasing, so the planar curve $z \mapsto \parn{\ell_{21}\parn{z}, \ell_{s1}\parn{z}}$ admits reparametrization by its first coordinate, along which
\begin{equation}
    \frac{d \ell_{s1}}{d \ell_{21}}
    = \frac{\sigma\!\parn{\gamma_{s-1} - z}}{\sigma\!\parn{\gamma_2 - z}}
    + \frac{\sigma\!\parn{z - \gamma_1}}{\sigma\!\parn{\gamma_2 - z}}.
    \label{eq:ef_ratio}
\end{equation}
Set $t \triangleq e^{z}$, strictly increasing in $z$. Substituting $\sigma\!\parn{\gamma - z} = e^{\gamma} / \parn{e^{\gamma} + t}$ and $\sigma\!\parn{z - \gamma_1} = t / \parn{e^{\gamma_1} + t}$ into \eqref{eq:ef_ratio},
\begin{equation}
    \frac{d \ell_{s1}}{d \ell_{21}} = F_1\parn{t} + F_2\parn{t},
    \quad \text{where }
    F_1\parn{t} \triangleq \frac{e^{\gamma_{s-1}} \parn{e^{\gamma_2} + t}}
                                {e^{\gamma_2} \parn{e^{\gamma_{s-1}} + t}},
    \quad
    F_2\parn{t} \triangleq \frac{t \parn{e^{\gamma_2} + t}}
                                {e^{\gamma_2} \parn{e^{\gamma_1} + t}}.
    \label{eq:ef_F}
\end{equation}
Differentiating in $t$,
\begin{equation}
    F_1'\parn{t} = \frac{e^{\gamma_{s-1}} \parn{e^{\gamma_{s-1}} - e^{\gamma_2}}}
                        {e^{\gamma_2} \parn{e^{\gamma_{s-1}} + t}^{2}} \geq 0,
    \qquad
    F_2'\parn{t} = \frac{t^{2} + 2 t e^{\gamma_1} + e^{\gamma_1 + \gamma_2}}
                        {e^{\gamma_2} \parn{e^{\gamma_1} + t}^{2}} > 0,
    \label{eq:ef_Fderivs}
\end{equation}
the first inequality holding since $\gamma_{s-1} \geq \gamma_2$, with equality throughout exactly when $s = 3$, and the second since $t > 0$. Hence $F_1'\parn{t} + F_2'\parn{t} > 0$, and since $t$ is strictly increasing in $z$,
\begin{align}
    & \frac{d}{d t}\!\parn{\frac{d \ell_{s1}}{d \ell_{21}}} > 0,
    \quad \forall\, t > 0 \nonumber \\
    \implies\ & \frac{d}{d z}\!\parn{\frac{d \ell_{s1}}{d \ell_{21}}} > 0,
    \quad \forall\, z \in \mathbb{R} \nonumber \\
    \implies\ & \frac{d^{2} \ell_{s1}}{d \ell_{21}^{2}}
    = \frac{1}{\ell_{21}'\parn{z}}
      \frac{d}{d z}\!\parn{\frac{d \ell_{s1}}{d \ell_{21}}} > 0,
    \quad \forall\, z \in \mathbb{R},
    \label{eq:ef_convex}
\end{align}
the last implication applying the chain rule to the reparametrization together with $\ell_{21}'\parn{z} > 0$ from \eqref{eq:ef_derivs}. Hence $\ell_{s1}$ is a strictly convex function of $\ell_{21}$, and no three distinct points of the curve are collinear.
 
It remains to contradict \eqref{eq:ef_affine}. Taking $x = 2$ and $x' = 1$ in \eqref{eq:ef_pair} and comparing with \eqref{eq:ef_ell_def},
\begin{equation}
    \ell_{21}\parn{z} = \log \frac{h_2}{h_1} + \theta\parn{y} \parn{T_2 - T_1}.
    \label{eq:ef_ell21_ef}
\end{equation}
If $T_2 = T_1$, then $\ell_{21}$ would be constant on $\mathcal{Y}$ by \eqref{eq:ef_ell21_ef}, contradicting $\ell_{21}'\parn{z} > 0$ in \eqref{eq:ef_derivs}. Hence $T_2 \neq T_1$, so \eqref{eq:ef_affine} applies to $\parn{u, v, r} = \parn{s, 2, 1}$ and yields constants $a = \parn{T_s - T_1} / \parn{T_2 - T_1}$ and $b$, both independent of $y$, with
\begin{equation}
    \ell_{s1}\parn{z} = a\, \ell_{21}\parn{z} + b.
    \label{eq:ef_collinear}
\end{equation}
Since $\bars{\mathcal{Y}} \geq 3$ and $\ell_{21}$ is strictly increasing in $z$, three distinct points of $\mathcal{Y}$ give three distinct points of the curve, which \eqref{eq:ef_collinear} places on a common line, contradicting the strict convexity established in \eqref{eq:ef_convex}.
Therefore, we conclude by contradiction that no base measure, sufficient statistic, natural parameter, and log-partition function satisfy \eqref{eq:ef_hyp}, and the conditional is therefore not a one-parameter exponential family in $X$.
\end{proof}

\subsection{Proof of Lemma \ref{lem:fwd_affine}}
\label{app:proof_fwd_affine}
\begin{proof}
We proceed by direct computation. From \eqref{eq:fwd_jnt}, the joint distribution under the forward model $\mathcal{M}_{X \to Y}$ is
\begin{equation}
    p_{X, Y}\parn{x, y; \mathcal{M}_{X \to Y}}
    = \pi_x H\parn{y} \exp\!\parn{\eta\!\parn{wx} T\parn{y} - a\!\parn{\eta\!\parn{wx}}},
    \quad \forall\, y \in \mathcal{Y},\ x \in \mathcal{X}.
\end{equation}
Since $\pi_x > 0$ and, by definition of an exponential family, $H\parn{y} \exp\!\parn{\eta\!\parn{wx} T\parn{y} - a\!\parn{\eta\!\parn{wx}}} > 0$ for all $y \in \mathcal{Y}$,
\begin{align}
    & p_{X, Y}\parn{x, y; \mathcal{M}_{X \to Y}} > 0,
    \quad \forall\, x \in \mathcal{X},\ y \in \mathcal{Y} \nonumber \\
    \implies\ & \sum_{x' \in \mathcal{X}}
    p_{X, Y}\parn{x', y; \mathcal{M}_{X \to Y}} > 0,
    \quad \forall\, y \in \mathcal{Y} \nonumber \\
    \implies\ & p_{Y}\parn{y; \mathcal{M}_{X \to Y}} > 0,
    \quad \forall\, y \in \mathcal{Y}.
    \label{eq:mar_Y_fwd}
\end{align}
Fix distinct $u, x \in \mathcal{X}$. Since $p_{Y}\parn{y; \mathcal{M}_{X \to Y}} > 0$ by \eqref{eq:mar_Y_fwd}, dividing numerator and denominator by the marginal equates the posterior ratio with the joint ratio,
\begin{equation}
    \frac{p_{X \mid Y}\parn{u \mid y; \mathcal{M}_{X \to Y}}}
         {p_{X \mid Y}\parn{x \mid y; \mathcal{M}_{X \to Y}}}
    = \frac{p_{X \mid Y}\parn{u \mid y; \mathcal{M}_{X \to Y}}\,
            p_{Y}\parn{y; \mathcal{M}_{X \to Y}}}
           {p_{X \mid Y}\parn{x \mid y; \mathcal{M}_{X \to Y}}\,
            p_{Y}\parn{y; \mathcal{M}_{X \to Y}}}
    = \frac{p_{X,Y}\parn{u, y; \mathcal{M}_{X \to Y}}}
           {p_{X,Y}\parn{x, y; \mathcal{M}_{X \to Y}}},
\end{equation}
the last equality applying the product rule $p_{X,Y}\parn{x,y} = p_{X \mid Y}\parn{x \mid y}\, p_{Y}\parn{y}$. Taking logarithms and substituting \eqref{eq:fwd_jnt},
\begin{align}
    R\parn{y, u, x; \mathcal{M}_{X \to Y}}
    &= \log \frac{p_{X,Y}\parn{u, y; \mathcal{M}_{X \to Y}}}
                 {p_{X,Y}\parn{x, y; \mathcal{M}_{X \to Y}}} \nonumber \\
    &= \parn{\eta\parn{wu} - \eta\parn{wx}} T\parn{y}
       + a\parn{\eta\parn{wx}} - a\parn{\eta\parn{wu}}
       + \log\!\parn{\frac{\pi_u}{\pi_x}}.
\end{align}
Define $\alpha$ and $\beta$, independent of $y$, as
\begin{equation}
    \alpha = \eta\parn{wu} - \eta\parn{wx}, \qquad
    \beta = a\parn{\eta\parn{wx}} - a\parn{\eta\parn{wu}}
    + \log\!\parn{\frac{\pi_u}{\pi_x}},
\end{equation}
where $\alpha \neq 0$ since $\eta\parn{\cdot}$ is monotone injective, $w \neq 0$, and $u \neq x$, yielding
\begin{equation}
    R\parn{y, u, x; \mathcal{M}_{X \to Y}} = \alpha\, T\parn{y} + \beta,
\end{equation}
as claimed.
\end{proof}

\subsection{Proof of Lemma \ref{lem:rev_bounds}}
\label{app:proof_rev_bounds}
\begin{proof}
We compute the two extreme adjacent log-odds-ratios from the conditional \eqref{eq:rev_cond_X} and the definition \eqref{eq:log_odds_ratio}. Using $\sigma\parn{\gamma_k - vy} = e^{\gamma_k} / \parn{e^{\gamma_k} + e^{vy}}$ together with the boundary conventions $e^{\gamma_0} = 0$ and $e^{\gamma_s} = \infty$, we obtain 
\begin{align}
    B\parn{y} &= \log\parn{\frac{p_{X \mid Y}\parn{2 \mid y; \mathcal{M}_{Y \to X}}}{p_{X \mid Y}\parn{1 \mid y; \mathcal{M}_{Y \to X}}}} \nonumber \\
    &= \log\parn{\frac{\sigma\parn{\gamma_2 - vy} - \sigma\parn{\gamma_1 - vy}}{\sigma\parn{\gamma_1 - vy} - \sigma\parn{\gamma_0 - vy}}} \nonumber \\
    &= \log\parn{\frac{e^{vy}\parn{e^{\gamma_2} - e^{\gamma_1}}}{e^{\gamma_1}\parn{e^{\gamma_2} + e^{vy}}}} \nonumber \\
    &= \log\parn{\frac{e^{\gamma_2} - e^{\gamma_1}}{e^{\gamma_1}}} - \log\parn{1 + e^{\gamma_2 - vy}}.
\end{align}
Similarly, for $U\parn{y}$ we have
\begin{align}
    U\parn{y} &= \log\parn{\frac{p_{X \mid Y}\parn{s \mid y; \mathcal{M}_{Y \to X}}}{p_{X \mid Y}\parn{s - 1 \mid y; \mathcal{M}_{Y \to X}}}} \nonumber \\
    &= \log\parn{\frac{\sigma\parn{\gamma_s - vy} - \sigma\parn{\gamma_{s - 1} - vy}}{\sigma\parn{\gamma_{s - 1} - vy} - \sigma\parn{\gamma_{s - 2} - vy}}} \nonumber \\
    &= \log\parn{\frac{e^{\gamma_{s - 2}} + e^{vy}}{e^{\gamma_{s - 1}} - e^{\gamma_{s - 2}}}} \nonumber \\
    &= \log\parn{\frac{e^{\gamma_{s - 2}}}{e^{\gamma_{s - 1}} - e^{\gamma_{s - 2}}}} + \log\parn{1 + e^{vy - \gamma_{s - 2}}}.
\end{align}
Defining constants $c_B$ and $c_U$, independent of $y$, as
\begin{equation}
    c_B = \log\parn{\frac{e^{\gamma_2} - e^{\gamma_1}}{e^{\gamma_1}}}, \quad
    c_U = \log\parn{\frac{e^{\gamma_{s - 2}}}{e^{\gamma_{s - 1}} - e^{\gamma_{s - 2}}}},
\end{equation}
where the strict ordering of the cutpoints together with $s \geq 3$, so that $s - 2 \geq 1$, ensures that $c_B$ and $c_U$ are finite, yields the closed-form expressions
\begin{equation}
    B\parn{y} = c_B - \log\parn{1 + e^{\gamma_2 - vy}}, \quad
    U\parn{y} = c_U + \log\parn{1 + e^{vy - \gamma_{s - 2}}}.
\end{equation}
The closed forms involve $y$ only through $e^{\pm vy}$ and are therefore twice differentiable at every $y \in \mathbb{R}$, whereas $B$ and $U$ are defined only on $\mathcal{Y}$, which need not be an interval. Let $\bar{B}, \bar{U} : \mathbb{R} \to \mathbb{R}$ denote the functions the closed forms define on all of $\mathbb{R}$, so $B = \bar{B}\big\vert_{\mathcal{Y}}$ and $U = \bar{U}\big\vert_{\mathcal{Y}}$. Differentiating twice with respect to $y$,
\begin{align}
    \frac{d^2 \bar{B}\parn{y}}{dy^2} &= -\frac{v^2 e^{\gamma_2 - vy}}{\parn{1 + e^{\gamma_2 - vy}}^2} < 0, \quad \forall\, y \in \mathbb{R}, \\
    \frac{d^2 \bar{U}\parn{y}}{dy^2} &= \frac{v^2 e^{vy - \gamma_{s - 2}}}{\parn{1 + e^{vy - \gamma_{s - 2}}}^2} > 0, \quad \forall\, y \in \mathbb{R},
\end{align}
both inequalities being strict because $v \neq 0$. Hence $\bar{B}$ is strictly concave and $\bar{U}$ strictly convex on $\mathbb{R}$, with $B$ and $U$ having their restrictions to $\mathcal{Y}$.
\end{proof}

\subsection{Proofs of Theorem \ref{thm:general_id}}
\label{app:proof_general_id}

\begin{namedproof}{First proof, by the curvature of Lemma~\ref{lem:rev_bounds}}
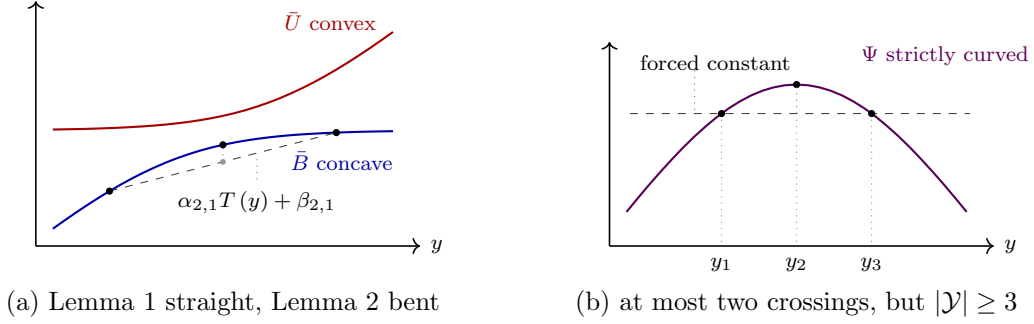
\begin{figure}[!htbp]
\centering
%==================== (a) the two lemmas ====================
\begin{minipage}[t]{0.49\textwidth}
\centering
\begin{tikzpicture}[baseline=0pt, scale=0.90]
  \draw[->,semithick] (-0.25,0) -- (5.4,0) node[right,font=\scriptsize] {$y$};
  \draw[->,semithick] (-0.25,0) -- (-0.25,3.6);
  \draw[thick,blue!65!black] plot coordinates {(0.000,0.254) (0.167,0.372) (0.333,0.487) (0.500,0.598) (0.667,0.705) (0.833,0.807) (1.000,0.904) (1.167,0.995) (1.333,1.079) (1.500,1.157) (1.667,1.229) (1.833,1.293) (2.000,1.351) (2.167,1.403) (2.333,1.448) (2.500,1.487) (2.667,1.521) (2.833,1.550) (3.000,1.575) (3.167,1.596) (3.333,1.614) (3.500,1.629) (3.667,1.641) (3.833,1.651) (4.000,1.660) (4.167,1.667) (4.333,1.673) (4.500,1.678) (4.667,1.682) (4.833,1.685) (5.000,1.688)};
  \draw[thick,red!65!black] plot coordinates {(0.000,1.712) (0.167,1.715) (0.333,1.718) (0.500,1.722) (0.667,1.727) (0.833,1.733) (1.000,1.740) (1.167,1.749) (1.333,1.759) (1.500,1.771) (1.667,1.786) (1.833,1.804) (2.000,1.825) (2.167,1.850) (2.333,1.879) (2.500,1.913) (2.667,1.952) (2.833,1.997) (3.000,2.049) (3.167,2.107) (3.333,2.171) (3.500,2.243) (3.667,2.321) (3.833,2.405) (4.000,2.496) (4.167,2.593) (4.333,2.695) (4.500,2.802) (4.667,2.913) (4.833,3.028) (5.000,3.146)};
  \draw[dashed,black!75] (0.833,0.807) -- (4.167,1.667);
  \foreach \x/\y in {0.833/0.807, 4.167/1.667}
    {\fill[black] (\x,\y) circle (1.5pt);}
  \fill[black] (2.500,1.487) circle (1.5pt);
  \draw[dotted,black!60] (2.500,1.237) -- (2.500,1.487);
  \fill[black!40] (2.500,1.237) circle (1.2pt);
  \node[font=\scriptsize,blue!65!black,anchor=west] at (3.35,1.25) {$\bar{B}$ concave};
  \node[font=\scriptsize,red!65!black,anchor=east]  at (4.9,3.30) {$\bar{U}$ convex};
  \draw[dotted,black!55] (3.00,0.98) -- (3.00,1.36);
  \node[font=\scriptsize,anchor=north] at (3.00,0.96) {$\alpha_{2,1}T\parn{y}+\beta_{2,1}$};
  \node[font=\small] at (2.5,-0.85) {(a) Lemma~\ref{lem:fwd_affine} straight, Lemma~\ref{lem:rev_bounds} bent};
\end{tikzpicture}
\end{minipage}
\hfill
%==================== (b) the contradiction ====================
\begin{minipage}[t]{0.49\textwidth}
\centering
\begin{tikzpicture}[baseline=0pt, scale=0.90]
  \draw[->,semithick] (-0.25,0) -- (5.4,0) node[right,font=\scriptsize] {$y$};
  \draw[->,semithick] (-0.25,0) -- (-0.25,3.0);
  \draw[thick,violet!70!black] plot coordinates {(0.000,0.501) (0.167,0.704) (0.333,0.901) (0.500,1.091) (0.667,1.274) (0.833,1.448) (1.000,1.611) (1.167,1.761) (1.333,1.899) (1.500,2.021) (1.667,2.126) (1.833,2.214) (2.000,2.283) (2.167,2.333) (2.333,2.363) (2.500,2.373) (2.667,2.363) (2.833,2.333) (3.000,2.283) (3.167,2.214) (3.333,2.126) (3.500,2.021) (3.667,1.899) (3.833,1.761) (4.000,1.611) (4.167,1.448) (4.333,1.274) (4.500,1.091) (4.667,0.901) (4.833,0.704) (5.000,0.501)};
  \draw[dashed,black!75] (0.05,1.947) -- (5.05,1.947);
  \fill[black] (1.397,1.947) circle (1.5pt);
  \fill[black] (3.603,1.947) circle (1.5pt);
  \fill[black] (2.500,2.373) circle (1.5pt);
  \foreach \x/\lab in {1.397/{y_1}, 2.500/{y_2}, 3.603/{y_3}}
    {\draw[dotted,black!45] (\x,0.06) -- (\x,1.93);
     \node[font=\scriptsize, anchor=north] at (\x,-0.02) {$\lab$};}
  \draw[dotted,black!60] (2.500,1.947) -- (2.500,2.373);
  \node[font=\scriptsize,violet!70!black,anchor=west] at (3.30,2.80) {$\Psi$ strictly curved};
  \draw[dotted,black!55] (1.00,2.58) -- (1.00,1.96);
  \node[font=\scriptsize,anchor=west] at (0.05,2.70) {forced constant};
  \node[font=\small] at (2.5,-0.85) {(b) at most two crossings, but $\bars{\mathcal{Y}} \geq 3$};
\end{tikzpicture}
\end{minipage}
\caption{The asymmetry behind Theorem~\ref{thm:general_id}.}
\label{fig:curvature}
\end{figure}
We proceed by contradiction. Assume that the models are not identifiable. By Definition~\ref{def:d_id}, there exist joint distributions $p_{X,Y} \in \mathcal{M}_{X \to Y}$ and $q_{X, Y} \in \mathcal{M}_{Y \to X}$ such that
\begin{equation}
    p_{X, Y}\parn{x, y} = q_{X, Y}\parn{x, y}, \quad \forall\, x \in \mathcal{X},\ y \in \mathcal{Y}.
\end{equation}
Summing over $x \in \mathcal{X}$ gives $p_Y\parn{y} = q_Y\parn{y}$, strictly positive, and dividing the joint by the common marginal yields
\begin{equation}
    p_{X \mid Y}\parn{x \mid y} = q_{X \mid Y}\parn{x \mid y}, \quad \forall\, x \in \mathcal{X},\ y \in \mathcal{Y}.
\end{equation}
Hence, for distinct $u, x \in \mathcal{X}$ and any $y \in \mathcal{Y}$, the log-odds-ratios of the two models coincide,
\begin{equation}
    R\parn{y, u, x; \mathcal{M}_{X \to Y}} = R\parn{y, u, x; \mathcal{M}_{Y \to X}}, \quad \forall\, y \in \mathcal{Y},\ u, x \in \mathcal{X},\ u \neq x.
    \label{eq:p_eq_lgo_eq_t2}
\end{equation}
From Lemma~\ref{lem:fwd_affine}, there exist real $\alpha_{u,x} \neq 0$ and
$\beta_{u,x}$, independent of $y$, such that
\begin{equation}
    R\parn{y, u, x; \mathcal{M}_{X \to Y}} = \alpha_{u, x} T\parn{y} + \beta_{u, x}, \quad \forall\, y \in \mathcal{Y},\ u, x \in \mathcal{X},\ u \neq x.
\end{equation}
Combining this with \eqref{eq:p_eq_lgo_eq_t2},
\begin{equation}
    R\parn{y, u, x; \mathcal{M}_{Y \to X}} = \alpha_{u, x} T\parn{y} + \beta_{u, x}, \quad \forall\, y \in \mathcal{Y},\ u, x \in \mathcal{X},\ u \neq x.
    \label{eq:p_eq_rev_affine_t2}
\end{equation}
Choosing $u = 2$ and $x = 1$, valid since $s \geq 3$, and substituting into
\eqref{eq:p_eq_rev_affine_t2},
\begin{align}
    & R\parn{y, 2, 1; \mathcal{M}_{Y \to X}} = \alpha_{2,1} T\parn{y} + \beta_{2,1}, \quad \forall\, y \in \mathcal{Y}, \nonumber \\
    \implies\ & B\parn{y} = \alpha_{2,1} T\parn{y} + \beta_{2,1}, \quad \forall\, y \in \mathcal{Y},
    \label{eq:p_eq_b_affine_rev_t2}
\end{align}
where $B\parn{y}$ is the lower extreme log-odds-ratio as defined in \eqref{eq:b_u_def}. Choosing $u = s$ and $x = s - 1$, also valid since $s \geq 3$, and substituting into \eqref{eq:p_eq_rev_affine_t2},
\begin{align}
    & R\parn{y, s, s - 1; \mathcal{M}_{Y \to X}} = \alpha_{s,s - 1} T\parn{y} + \beta_{s,s - 1}, \quad \forall\, y \in \mathcal{Y}, \nonumber \\
    \implies\ & U\parn{y} = \alpha_{s,s - 1} T\parn{y} + \beta_{s,s - 1}, \quad \forall\, y \in \mathcal{Y},
    \label{eq:p_eq_u_affine_rev_t2}
\end{align}
where $U\parn{y}$ is the upper extreme log-odds-ratio also defined in \eqref{eq:b_u_def}. Define
\begin{equation}
    \Psi\parn{y} \triangleq \alpha_{s,s - 1} B\parn{y} - \alpha_{2,1} U\parn{y}.
    \label{eq:p_eq_psi_def_t2}
\end{equation}
Multiplying \eqref{eq:p_eq_b_affine_rev_t2} by $\alpha_{s,s-1}$ and
\eqref{eq:p_eq_u_affine_rev_t2} by $\alpha_{2,1}$ and subtracting, the term in
$T\parn{y}$ cancels, leaving only the intercepts, so that
\begin{equation}
    \Psi\parn{y} = \alpha_{s,s - 1}\beta_{2,1} - \alpha_{2,1}\beta_{s,s - 1}, \quad \forall\, y \in \mathcal{Y},
    \label{eq:p_eq_psi_zero}
\end{equation}
the right-hand side being independent of $y$. Set $t \triangleq e^{vy}$, a strictly monotone injective transformation of $y$, and let $\mathcal{T} \triangleq \curl{e^{vy} : y \in \mathcal{Y}} \subseteq \parn{0, \infty}$. Substituting the closed forms of $\bar{B}$ and $\bar{U}$ from Lemma~\ref{lem:rev_bounds} into \eqref{eq:p_eq_psi_def_t2}, which extends $\Psi$ from $\mathcal{T}$ to all of $\parn{0, \infty}$, and using $e^{\gamma_2 - vy} = e^{\gamma_2}/t$ and $e^{vy - \gamma_{s-2}} = t/e^{\gamma_{s-2}}$ we get,
\begin{equation}
    \Psi\parn{t} = \alpha_{s, s - 1}\parn{\log t - \log\parn{e^{\gamma_2} + t}} - \alpha_{2, 1}\log\parn{e^{\gamma_{s - 2}} + t} + k, \quad \forall\, t \in \parn{0, \infty},
    \label{eq:p_eq_psi_t_t2}
\end{equation}
where $k$ collects the terms independent of $t$. Setting the derivative in $t$ to zero,
\begin{align}
    & \frac{d\Psi\parn{t}}{dt} = 0 \nonumber \\
    \implies\ & \alpha_{s, s - 1}\parn{\frac{1}{t} - \frac{1}{e^{\gamma_2} + t}} - \frac{\alpha_{2, 1}}{e^{\gamma_{s - 2}} + t} = 0.
\end{align}
Since $t > 0$, both $e^{\gamma_2} + t > 0$ and $e^{\gamma_{s - 2}} + t > 0$; clearing these denominators reduces the equation to
\begin{equation}
    \alpha_{2, 1} t^2 + \parn{\alpha_{2,1} - \alpha_{s, s - 1}} e^{\gamma_2} t - \alpha_{s, s - 1} e^{\gamma_2 + \gamma_{s - 2}} = 0.
    \label{eq:p_eq_quad_t2}
\end{equation}
By Lemma~\ref{lem:fwd_affine} and \eqref{eq:fwd_cond_Y}, $\eta\parn{\cdot}$ is monotone injective and $\alpha_{u,x} = \eta\parn{wu} - \eta\parn{wx} \neq 0$ for distinct $u, x$. Hence
\begin{equation}
    \frac{\alpha_{s, s - 1}}{\alpha_{2, 1}} = \frac{\eta\parn{sw} - \eta\parn{sw - w}}{\eta\parn{2w} - \eta\parn{w}} > 0, \quad \forall\, w \neq 0,
    \label{eq:p_eq_rat_alpha_t2}
\end{equation}
since a monotone $\eta$ makes both differences share the sign of $w$. Writing $t_+$ and $t_-$ for the two roots of \eqref{eq:p_eq_quad_t2}, their product is
\begin{equation}
    t_+ t_- = -\frac{\alpha_{s, s - 1}}{\alpha_{2, 1}} e^{\gamma_2 + \gamma_{s - 2}} < 0,
\end{equation}
by \eqref{eq:p_eq_rat_alpha_t2}, so exactly one root is strictly positive. Without loss of generality, fix $t_+ > 0$ and $t_- < 0$. Since $\mathcal{T} \subseteq \parn{0, \infty}$, the only admissible critical point of $\Psi$ is $t_+$. Consequently $\Psi$ has exactly one critical point on $\parn{0, \infty}$, is strictly monotone on each of $\parn{0, t_+}$ and $\parn{t_+, \infty}$, and hence attains any given value at most twice.

Because $t = e^{vy}$ is a bijection from $\mathbb{R}$ onto $\parn{0, \infty}$, the points of $\mathcal{Y}$ map to distinct points of $\mathcal{T}$. But \eqref{eq:p_eq_psi_zero} transported to $t$ makes $\Psi$ constant on $\mathcal{T}$, and since $\bars{\mathcal{Y}} \geq 3$ there exist three pairwise distinct $\tau_1, \tau_2, \tau_3 \in \mathcal{T}$ at which $\Psi$ attains the common value $\alpha_{s,s - 1}\beta_{2,1} - \alpha_{2,1}\beta_{s,s - 1}$, contradicting the preceding paragraph.

The contradiction establishes that the affine representations \eqref{eq:p_eq_b_affine_rev_t2} and \eqref{eq:p_eq_u_affine_rev_t2} cannot both hold. By Lemma~\ref{lem:fwd_affine} the two are equivalent to \eqref{eq:p_eq_lgo_eq_t2}, a necessary consequence of the equality of conditionals and in turn of the equality of joints. Hence no $p_{X,Y} \in \mathcal{M}_{X \to Y}$ and $q_{X,Y} \in \mathcal{M}_{Y \to X}$ coincide on $\mathcal{X} \times \mathcal{Y}$; the two classes induce disjoint sets of joint distributions, and by Definition~\ref{def:d_id} the edge direction is distributionally identifiable.

Figure~\ref{fig:curvature} gives a visual summary of the proof. Panel (a) presents the obstruction the affine representations \eqref{eq:p_eq_b_affine_rev_t2} face, a strictly concave $\bar{B}$ meeting any straight line at no more than two points, and panel (b) presents the same obstruction after the elimination, $\Psi$ strictly curved yet held constant on $\mathcal{Y}$ by \eqref{eq:p_eq_psi_zero}, which the three points supplied by $\bars{\mathcal{Y}} \geq 3$ contradict.
\end{namedproof}

\begin{namedproof}{Second proof, by Proposition~\ref{prop:ord_ne}}
We again argue by contradiction, showing this time that the forward model enforces the conditional distribution of $X$ given $Y$ to be a one-parameter exponential family in $X$, whereas the reverse model, by Proposition~\ref{prop:ord_ne}, does not. Assume that the models are not identifiable. Then, from Definition~\ref{def:d_id}, there exist joint distributions $p_{X,Y} \in \mathcal{M}_{X \to Y}$ and $q_{X,Y} \in \mathcal{M}_{Y \to X}$ such that
\begin{equation}
    p_{X,Y}\parn{x, y} = q_{X,Y}\parn{x, y}, \quad \forall\, x \in \mathcal{X},\ y \in \mathcal{Y}.
    \label{eq:p2_jnt}
\end{equation}
Summing over $x \in \mathcal{X}$, the marginals of $Y$ under the two models satisfy
\begin{align}
    & p_{X,Y}\parn{x, y} = q_{X,Y}\parn{x, y}, \quad \forall\, x \in \mathcal{X},\ y \in \mathcal{Y} \nonumber \\
    \implies\ & \sum_{x' \in \mathcal{X}} p_{X,Y}\parn{x', y} = \sum_{x' \in \mathcal{X}} q_{X,Y}\parn{x', y}, \quad \forall\, y \in \mathcal{Y} \nonumber \\
    \implies\ & p_{Y}\parn{y} = q_{Y}\parn{y}, \quad \forall\, y \in \mathcal{Y}.
\end{align}
The marginals of $Y$ thus coincide, and since $p_{Y}\parn{y} > 0$ for all $y \in \mathcal{Y}$ by \eqref{eq:mar_Y_fwd}, so too is $q_{Y}\parn{y} > 0$. Combining with \eqref{eq:p2_jnt}, the product rule gives
\begin{equation}
    p_{X \mid Y}\parn{x \mid y; \mathcal{M}_{X \to Y}} = p_{X \mid Y}\parn{x \mid y; \mathcal{M}_{Y \to X}}, \quad \forall\, x \in \mathcal{X},\ y \in \mathcal{Y}.
    \label{eq:p2_cond}
\end{equation}

We now compute the left-hand side of \eqref{eq:p2_cond}. Applying Bayes' rule to the marginal of \eqref{eq:fwd_mar_X} and the $Y \mid X$ conditional given in \eqref{eq:fwd_cond_Y}, further canceling the base measure $H\parn{y}$, which is strictly positive and common to numerator and denominator, yields
\begin{equation}
    p_{X \mid Y}\parn{x \mid y; \mathcal{M}_{X \to Y}} = \frac{\pi_x \exp\parn{\eta\parn{wx} T\parn{y} - a\parn{\eta\parn{wx}}}}{\sum_{x' \in \mathcal{X}} \pi_{x'} \exp\parn{\eta\parn{wx'} T\parn{y} - a\parn{\eta\parn{wx'}}}}, \quad \forall\, x \in \mathcal{X},\ y \in \mathcal{Y}.
    \label{eq:p2_bayes}
\end{equation}
Define the quantities
\begin{equation}
    h_x \triangleq \pi_x e^{-a\parn{\eta\parn{wx}}}, \quad
    \tilde{T}_x \triangleq \eta\parn{wx}, \quad
    \theta\parn{y} \triangleq T\parn{y}, \quad
    A\parn{\theta} \triangleq \log \sum_{x' \in \mathcal{X}} h_{x'} e^{\theta \tilde{T}_{x'}},
    \label{eq:p2_ef_data}
\end{equation}
where $h_x > 0$ for every $x \in \mathcal{X}$, since $\pi_x > 0$ by \eqref{eq:fwd_mar_X} and the log-partition function $a$ is finite on the natural parameter space. Substituting \eqref{eq:p2_ef_data} into \eqref{eq:p2_bayes} yields
\begin{align}
    & p_{X \mid Y}\parn{x \mid y; \mathcal{M}_{X \to Y}} = \frac{h_x e^{\theta\parn{y} \tilde{T}_x}}{\sum_{x' \in \mathcal{X}} h_{x'} e^{\theta\parn{y} \tilde{T}_{x'}}}, \quad \forall\, x \in \mathcal{X},\ y \in \mathcal{Y} \nonumber \\
    \implies\ & p_{X \mid Y}\parn{x \mid y; \mathcal{M}_{X \to Y}} = h_x \exp\parn{\theta\parn{y} \tilde{T}_x - A\parn{\theta\parn{y}}}, \quad \forall\, x \in \mathcal{X},\ y \in \mathcal{Y},
    \label{eq:p2_is_ef}
\end{align}
which is a one-parameter exponential family on $\mathcal{X}$ indexed by $y$, with base measure $\curl{h_x}$, sufficient statistic $\curl{\tilde{T}_x}$, scalar natural parameter $\theta\parn{y}$, and log-partition function $A$. Equation \eqref{eq:p2_is_ef} restates Lemma~\ref{lem:fwd_affine} in a different form, since taking log-ratios returns the affine representation \eqref{eq:forward_lpr} with $\alpha_{u,x} = \tilde{T}_u - \tilde{T}_x$.

We next consider the right-hand side of \eqref{eq:p2_cond}, which by \eqref{eq:rev_cond_X} is the ordered-logit conditional with edge weight $v \neq 0$ and strictly ordered cutpoints $-\infty = \gamma_0 < \gamma_1 < \cdots < \gamma_{s-1} < \gamma_s = +\infty$, on the support $\mathcal{X} = \curl{1, \dots, s}$ with $s \geq 3$, indexed by $y$ ranging over $\mathcal{Y}$ with $\bars{\mathcal{Y}} \geq 3$. These are the conditions of Proposition~\ref{prop:ord_ne}, under which the conditional admits no representation of the form \eqref{eq:p2_is_ef}, for any base measure, sufficient statistic, natural parameter, and log-partition function whatsoever.

The two sides of \eqref{eq:p2_cond} therefore cannot be equal, the left-hand side being a one-parameter exponential family in $X$ indexed by $y$ and the right-hand side admitting no such representation. Hence \eqref{eq:p2_cond} cannot hold, and being a necessary consequence of \eqref{eq:p2_jnt}, renders the assumption \eqref{eq:p2_jnt} untenable. No $p_{X,Y} \in \mathcal{M}_{X \to Y}$ and $q_{X,Y} \in \mathcal{M}_{Y \to X}$ coincide on $\mathcal{X} \times \mathcal{Y}$; the two classes induce disjoint sets of joint distributions, and by Definition~\ref{def:d_id} the edge direction is distributionally identifiable.
\end{namedproof}

\subsection{Proof of Proposition~\ref{prop:general_link}}
\label{app:proof_general_link}
\begin{proof}
Assume, to show contradiction, that the two model classes are not disjoint. Then there exist joint distributions $p_{X,Y} \in \mathcal{M}_{X \to Y}$ and $q_{X,Y} \in \mathcal{M}_{Y \to X}$ such that
\begin{equation}
    p_{X,Y}\parn{x, y} = q_{X,Y}\parn{x, y}, \quad \forall\, x \in \mathcal{X},\ y \in \mathcal{Y}.
    \label{eq:gl_jnt}
\end{equation}
Since $F$ is strictly increasing and the cutpoints are strictly ordered, every difference $F\parn{\gamma_x - vy} - F\parn{\gamma_{x-1} - vy}$ in \eqref{eq:gen_link_cond} is strictly positive, so the reverse conditional is strictly positive on $\mathcal{X}$ for every $y \in \mathcal{Y}$ and the quantities \eqref{eq:gen_link_bu} are well defined.

Summing \eqref{eq:gl_jnt} over $x \in \mathcal{X}$, the marginals of $Y$ under the two models satisfy
\begin{align}
    & p_{X,Y}\parn{x, y} = q_{X,Y}\parn{x, y}, \quad \forall\, x \in \mathcal{X},\ y \in \mathcal{Y} \nonumber \\
    \implies\ & \sum_{x' \in \mathcal{X}} p_{X,Y}\parn{x', y} = \sum_{x' \in \mathcal{X}} q_{X,Y}\parn{x', y}, \quad \forall\, y \in \mathcal{Y} \nonumber \\
    \implies\ & p_{Y}\parn{y} = q_{Y}\parn{y}, \quad \forall\, y \in \mathcal{Y}.
\end{align}
This common marginal is strictly positive, so dividing \eqref{eq:gl_jnt} by it and applying the product rule gives
\begin{equation}
    p_{X \mid Y}\parn{x \mid y} = q_{X \mid Y}\parn{x \mid y}, \quad \forall\, x \in \mathcal{X},\ y \in \mathcal{Y},
    \label{eq:gl_cond}
\end{equation}
and hence, by the definition \eqref{eq:log_odds_ratio} of the log-odds-ratio, for distinct $u, x \in \mathcal{X}$,
\begin{equation}
    R\parn{y, u, x; \mathcal{M}_{X \to Y}} = R\parn{y, u, x; \mathcal{M}_{Y \to X}}, \quad \forall\, y \in \mathcal{Y},\ u \neq x.
    \label{eq:gl_lor}
\end{equation}
From Lemma~\ref{lem:fwd_affine} there exist real $\alpha_{u,x} = \eta\parn{wu} - \eta\parn{wx} \neq 0$ and $\beta_{u,x}$, independent of $y$, with $R\parn{y, u, x; \mathcal{M}_{X \to Y}} = \alpha_{u,x} T\parn{y} + \beta_{u,x}$, so that \eqref{eq:gl_lor} forces
\begin{equation}
    R\parn{y, u, x; \mathcal{M}_{Y \to X}} = \alpha_{u,x} T\parn{y} + \beta_{u,x}, \quad \forall\, y \in \mathcal{Y},\ u \neq x.
    \label{eq:gl_affine}
\end{equation}
Choosing $u = 2$, $x = 1$ and then $u = s$, $x = s - 1$, both admissible since $s \geq 3$, and substituting into \eqref{eq:gl_affine},
\begin{equation}
    \bar{B}\parn{y} = \alpha_{2,1} T\parn{y} + \beta_{2,1}, \qquad
    \bar{U}\parn{y} = \alpha_{s,s-1} T\parn{y} + \beta_{s,s-1}, \qquad \forall\, y \in \mathcal{Y}.
    \label{eq:gl_bu}
\end{equation}
Define
\begin{equation}
    \Psi\parn{y} \triangleq \alpha_{s,s-1}\bar{B}\parn{y} - \alpha_{2,1}\bar{U}\parn{y}, \qquad y \in \mathbb{R}.
    \label{eq:gl_psi}
\end{equation}
Multiplying the first identity in \eqref{eq:gl_bu} by $\alpha_{s,s-1}$ and the second by $\alpha_{2,1}$ and subtracting, the term in $T\parn{y}$ cancels and only the intercepts remain, so that
\begin{equation}
    \Psi\parn{y} = \alpha_{s,s-1}\beta_{2,1} - \alpha_{2,1}\beta_{s,s-1}, \quad \forall\, y \in \mathcal{Y},
    \label{eq:gl_const}
\end{equation}
the right-hand side being independent of $y$.

We next fix the signs of the two slopes. Both $\alpha_{2,1} = \eta\parn{2w} - \eta\parn{w}$ and $\alpha_{s,s-1} = \eta\parn{sw} - \eta\parn{sw - w}$ compare the values of $\eta$ at two points separated by $w$, the larger argument being taken first in each. Since $\eta$ is strictly monotone and $w \neq 0$, both differences are nonzero and carry the same sign, that of $w$ when $\eta$ is increasing and its opposite when $\eta$ is decreasing. In either case
\begin{equation}
    \operatorname{sign}\parn{\alpha_{2,1}} = \operatorname{sign}\parn{\alpha_{s,s-1}}.
    \label{eq:gl_sign}
\end{equation}

Suppose first that both slopes are positive. By hypothesis $\bar{B}$ is strictly concave on $\mathbb{R}$, so $\alpha_{s,s-1}\bar{B}$ is strictly concave, being a positive multiple of it; and $\bar{U}$ is strictly convex on $\mathbb{R}$, so $-\alpha_{2,1}\bar{U}$ is strictly concave, being a negative multiple of it. Their sum $\Psi$ is therefore strictly concave on $\mathbb{R}$. Suppose instead that both slopes are negative. Then $\alpha_{s,s-1}\bar{B}$ is strictly convex and $-\alpha_{2,1}\bar{U}$ is strictly convex, so $\Psi$ is strictly convex on $\mathbb{R}$. By \eqref{eq:gl_sign} no other case arises, and $\Psi$ is strictly concave or strictly convex on $\mathbb{R}$.

A function of either kind attains any given value at most twice. Suppose $\Psi$ attained the value $c$ at three pairwise distinct points, and assume without loss of generality that $y_3 < y_2 < y_1$. Setting
\begin{equation}
    \lambda = \frac{y_2 - y_3}{y_1 - y_3} \in \parn{0, 1},
    \label{eq:gl_lambda}
\end{equation}
which lies in $\parn{0, 1}$ because $0 < y_2 - y_3 < y_1 - y_3$, the middle point admits the convex representation $y_2 = \lambda y_1 + \parn{1 - \lambda} y_3$. Strict concavity would then give
\begin{equation}
    \Psi\parn{y_2} > \lambda \Psi\parn{y_1} + \parn{1 - \lambda} \Psi\parn{y_3} = c,
\end{equation}
and strict convexity would yield the reverse strict inequality, each contradicting $\Psi\parn{y_2} = c$.

Since $\bars{\mathcal{Y}} \geq 3$, there exist three pairwise distinct points of $\mathcal{Y}$ at which, by \eqref{eq:gl_const}, the function $\Psi$ attains the common value $\alpha_{s,s-1}\beta_{2,1} - \alpha_{2,1}\beta_{s,s-1}$. This contradicts the preceding paragraph.

The contradiction establishes that \eqref{eq:gl_const} cannot hold, and therefore that the affine representations \eqref{eq:gl_bu} cannot both hold. By Lemma~\ref{lem:fwd_affine} the two are equivalent to \eqref{eq:gl_lor}, a necessary consequence of \eqref{eq:gl_cond} and in turn of \eqref{eq:gl_jnt}, so the assumption \eqref{eq:gl_jnt} is untenable. No $p_{X,Y} \in \mathcal{M}_{X \to Y}$ and $q_{X,Y} \in \mathcal{M}_{Y \to X}$ coincide on $\mathcal{X} \times \mathcal{Y}$. Thus, the two classes induce disjoint sets of joint distributions, and by Definition~\ref{def:d_id}, the edge direction is distributionally identifiable.
\end{proof}

\subsection{Proof of Theorem~\ref{thm:converse_two_points}}
\label{app:proof_converse_two_points}
\begin{proof}
We construct parameters inducing identical joint distributions under both models. Set $\Delta y = y_2 - y_1 > 0$, $t_i \triangleq e^{vy_i}$ for $i \in \curl{1, 2}$, and $L \triangleq v\Delta y$. The sign of $v$ is not free: the construction below yields a strictly decreasing sequence whose two-step differences all equal $w \Delta T$, which must therefore be positive, and since $\Delta y > 0$ and $\Delta T = c_1 \Delta y$, the positivity forces $\operatorname{sign}\parn{v} = \operatorname{sign}\parn{w c_1}$. We take $w c_1 > 0$ and hence $v > 0$ below, the opposite sign being obtained throughout by reversing the inequalities. For each $x \in \mathcal{X}$ write $l_{x}\parn{y; \mathcal{M}} = R\parn{y, x, 1; \mathcal{M}}$ for the log-odds ratio of category $x$ against category $1$ under model $\mathcal{M}$.

Now, for $j = 0, 1, \dots, s$ define,
\begin{eqnarray}
    D_j = \log\parn{\frac{e^{\gamma_j} + t_2}{e^{\gamma_j} + t_1}}, \quad \forall j \in \curl{0, 1, \dots, s},
    \label{eq:p_eq_dj_def_t3}
\end{eqnarray}
where by convention $\gamma_0 = -\infty$ and $\gamma_s = +\infty$ we would get the boundary values as,
\begin{eqnarray}
    D_s = 0, \quad D_0 = \log\parn{\frac{t_2}{t_1}} = v\Delta y = L.
    \label{eq:p_eq_d0ds_conv_t3}
\end{eqnarray}
Now, differentiating $D_j$ w.r.t. $\gamma_j$ for $j \in \curl{1, \dots, s - 1}$ gives us,
\begin{eqnarray}
    \frac{dD_j}{d\gamma_j} = \frac{\parn{t_1 - t_2}e^{\gamma_j}}{\parn{e^{\gamma_j} + t_1}\parn{e^{\gamma_j} + t_2}}, \quad \forall j \in \curl{1, \dots, s - 1}.
\end{eqnarray}
As $y_2 > y_1$ is given, $t_i$ is an injective transformation of $y_i$, giving us $t_2 > t_1$, and by definition $t_i > 0$ for $i \in \curl{1, 2}$. Thus, we get,
\begin{eqnarray}
    \frac{dD_j}{d\gamma_j} < 0, \quad \forall j \in \curl{1, \dots, s - 1}.
\end{eqnarray}
Therefore, $D_j$ is a monotone decreasing function in $\gamma_j$. As we have $-\infty = \gamma_0 < \gamma_1 < \dots < \gamma_{s - 1} < \gamma_s = +\infty$, therefore we should also have,
\begin{eqnarray}
    L = D_0 > D_1 > \dots > D_{s - 1} > D_s = 0.
    \label{eq:p_eq_d_ord_t3}
\end{eqnarray}
From Lemma~\ref{lem:fwd_affine}, we know that there exist $\alpha_{x, 1} = \eta\parn{wx} - \eta\parn{w} \neq 0$ for all $x \in \curl{2, \dots, s}$ and $\beta_{x, 1}$ independent of $y$ such that,
\begin{eqnarray}
    l_{x}\parn{y; \mathcal{M}_{X \to Y}} = \alpha_{x, 1}T\parn{y} + \beta_{x, 1}, \quad \forall y \in \mathcal{Y} \text{ and } x \in \curl{2, \dots, s}.
\end{eqnarray}
As $\eta\parn{z} = z$ is given, therefore, we would also have $\alpha_{x, 1} = w\parn{x - 1}$. Defining $\Delta T= T\parn{y_2} - T\parn{y_1}$, we will have,
\begin{eqnarray}
    l_{x}\parn{y_2; \mathcal{M}_{X \to Y}} - l_{x}\parn{y_1; \mathcal{M}_{X \to Y}} = w\parn{x - 1}\Delta T, \quad \forall x \in \curl{2, \dots, s}.
    \label{eq:p_eq_diff_l_fwd_t3}
\end{eqnarray}
Now, from the reverse model, we get,
\begin{eqnarray}
    &&l_x\parn{y_2; \mathcal{M}_{Y \to X}} - l_x\parn{y_1; \mathcal{M}_{Y \to X}} \nonumber \\
    &=& R\parn{y_2, x, 1; \mathcal{M}_{Y \to X}} - R\parn{y_1, x, 1; \mathcal{M}_{Y \to X}} \nonumber \\
    &=& \log\parn{\frac{\sigma\parn{\gamma_x - vy_2} - \sigma\parn{\gamma_{x - 1} - vy_2}}{\sigma\parn{\gamma_1 - vy_2} - \sigma\parn{\gamma_{0} - vy_2}}} - \log\parn{\frac{\sigma\parn{\gamma_x - vy_1} - \sigma\parn{\gamma_{x - 1} - vy_1}}{\sigma\parn{\gamma_1 - vy_1} - \sigma\parn{\gamma_{0} - vy_1}}} \nonumber \\
    &=& \log\parn{\frac{e^{vy_2}\parn{e^{\gamma_x} - e^{\gamma_{x-1}}}\parn{e^{\gamma_1} + e^{vy_2}}}{\parn{e^{\gamma_x} + e^{vy_2}}\parn{e^{\gamma_{x-1}} + e^{vy_2}}e^{\gamma_1}}} - \log\parn{\frac{e^{vy_1}\parn{e^{\gamma_x} - e^{\gamma_{x-1}}}\parn{e^{\gamma_1} + e^{vy_1}}}{\parn{e^{\gamma_x} + e^{vy_1}}\parn{e^{\gamma_{x-1}} + e^{vy_1}}e^{\gamma_1}}} \nonumber \\
    &=& \log\parn{\frac{e^{\gamma_1} + e^{vy_2}}{e^{\gamma_1} + e^{vy_1}}} + \log\parn{\frac{e^{vy_2}}{e^{vy_1}}} - \log\parn{\frac{e^{\gamma_x} + e^{vy_2}}{e^{\gamma_x} + e^{vy_1}}} - \log\parn{\frac{e^{\gamma_{x-1}} + e^{vy_2}}{e^{\gamma_{x-1}} + e^{vy_1}}} \nonumber \\
    &=& \parn{D_1 + D_0} - \parn{D_x + D_{x-1}}, \quad \forall x \in \curl{2, \dots, s}.
    \label{eq:p_eq_diff_l_rev_t3}
\end{eqnarray}
Equating \eqref{eq:p_eq_diff_l_fwd_t3} and \eqref{eq:p_eq_diff_l_rev_t3} gives us,
\begin{eqnarray}
    \parn{D_1 + D_0} - \parn{D_x + D_{x-1}} = \parn{x - 1}w\Delta T, \quad \forall x \in \curl{2, \dots, s}.
    \label{eq:p_eq_d01x_t3}
\end{eqnarray}
Now choose some arbitrary $i \in \curl{0, \dots, s - 3}$, the upper limit being forced by the requirement that $x = i + 3$ lie in $\curl{2, \dots, s}$. Substituting $x = i + 2$ and $x = i + 3$ in \eqref{eq:p_eq_d01x_t3} gives us,
\begin{eqnarray}
    \label{eq:p_eq_jp2_t3}
    \parn{D_1 + D_0} - \parn{D_{i + 2} + D_{i + 1}} = \parn{i + 1}w\Delta T, \\
    \parn{D_1 + D_0} - \parn{D_{i + 3} + D_{i + 2}} = \parn{i + 2}w\Delta T.
    \label{eq:p_eq_jp3_t3}
\end{eqnarray}
Subtracting \eqref{eq:p_eq_jp2_t3} from \eqref{eq:p_eq_jp3_t3} gives us,
\begin{eqnarray}
    D_{i + 1} - D_{i + 3} = w\Delta T.
    \label{eq:p_eq_j_init_t3}
\end{eqnarray}
Similarly, substituting $x = 2$ in \eqref{eq:p_eq_d01x_t3} gives us,
\begin{eqnarray}
    D_0 - D_2 = w\Delta T.
    \label{eq:p_eq_02_t3}
\end{eqnarray}
Since $i$ was arbitrary in $\curl{0, \dots, s - 3}$, relation \eqref{eq:p_eq_j_init_t3} holds with $i + 1$ ranging over $\curl{1, \dots, s - 2}$, and \eqref{eq:p_eq_02_t3} supplies the one remaining index. Together they give
\begin{eqnarray}
    D_i - D_{i + 2} = w\Delta T, \quad \forall i \in \curl{0, \dots, s - 2}.
    \label{eq:p_eq_d_rec_t3}
\end{eqnarray}
We also know that $D_0 = L$ and $D_s = 0$. Therefore, the recursion in \eqref{eq:p_eq_d_rec_t3} decouples into even and odd chains as,
\begin{eqnarray}
    D_{2k} = L - kw\Delta T, \quad D_{2k + 1} = D_1 - kw\Delta T, \quad \forall 2k, 2k+1 \in \curl{0, \dots, s}.
    \label{eq:p_eq_chains_t3}
\end{eqnarray}
If $s$ is even, i.e. $s = 2m$ for some $m \in \curl{2, 3, \dots}$, then the condition $D_s = 0$ leads us to,
\begin{eqnarray}
    D_{2m} = 0 \implies L = mw\Delta T,
\end{eqnarray}
leaving $D_1$ as a free variable. On the contrary, if $s$ is odd, i.e. $s = 2m + 1$ for some $m \in \curl{1, 2, 3, \dots}$, then the condition $D_s = 0$ leads us to,
\begin{eqnarray}
    D_{2m + 1} = 0 \implies D_1 = mw\Delta T,
\end{eqnarray}
leaving $L$ as a free variable. In order to constitute a valid solution, we need to choose the free variable such that the whole sequence of $D_j$ for all $j \in \curl{0, \dots, s}$ follows the strict ordering given in \eqref{eq:p_eq_d_ord_t3}.

For even $s = 2m$, using \eqref{eq:p_eq_chains_t3} with $L = mw\Delta T$, the strict ordering in \eqref{eq:p_eq_d_ord_t3} reduces to the two conditions $L - w\Delta T < D_1$ and $D_1 < L$, which lead to the open interval
\begin{eqnarray}
    D_1 \in \parn{\frac{m - 1}{m}L,\ L},
    \label{eq:p_eq_d1_int_t3}
\end{eqnarray}
which is non-empty as $m \geq 2$. By \eqref{eq:p_eq_dj_def_t3}, the first cutpoint $\gamma_1$ satisfies
\begin{eqnarray}
    e^{\gamma_1} = \frac{t_2 - t_1 e^{D_1}}{e^{D_1} - 1},
    \label{eq:p_eq_g1_from_d1_t3}
\end{eqnarray}
and since $D_1$ is strictly decreasing in $\gamma_1$, the interval \eqref{eq:p_eq_d1_int_t3} corresponds to an open interval for $\gamma_1$. At the endpoint $D_1 \to L$, \eqref{eq:p_eq_g1_from_d1_t3} gives $e^{\gamma_1} \to 0$, that is $\gamma_1 \to -\infty$, and at the endpoint $D_1 \to \frac{m - 1}{m}L$, it gives the finite value
\begin{eqnarray}
    e^{\gamma_1} \to \frac{t_2 e^{L/m} - t_1 e^{L}}{e^{L} - e^{L/m}},
\end{eqnarray}
which is strictly positive, since $t_2 = t_1 e^{L}$ makes the numerator $t_1 e^{L}\parn{e^{L/m} - 1} > 0$ and the denominator $e^{L} - e^{L/m} > 0$ as $L > L/m$. Therefore, for even $s$, the first cutpoint $\gamma_1$ ranges over the non-empty open interval
\begin{eqnarray}
    \gamma_1 \in \parn{-\infty,\ \log\parn{\frac{t_2 e^{L/m} - t_1 e^{L}}{e^{L} - e^{L/m}}}}.
    \label{eq:p_eq_g1_int_t3}
\end{eqnarray}

For odd $s = 2m + 1$, using \eqref{eq:p_eq_chains_t3} with $D_1 = mw\Delta T$, the strict ordering in \eqref{eq:p_eq_d_ord_t3} reduces to the two conditions $mw\Delta T < L$ and $L < \parn{m + 1}w\Delta T$, which lead to the open interval
\begin{eqnarray}
    L \in \parn{mw\Delta T,\ \parn{m + 1}w\Delta T},
    \label{eq:p_eq_L_int_t3}
\end{eqnarray}
which is non-empty as $w\Delta T > 0$. Since $L = v\Delta y$ and $\Delta T = c_1\Delta y$, the interval \eqref{eq:p_eq_L_int_t3} is equivalently an open interval for the reverse edge weight $v$,
\begin{eqnarray}
    v \in \parn{mwc_1,\ \parn{m + 1}wc_1},
    \label{eq:p_eq_v_int_t3}
\end{eqnarray}
which is non-empty, its length being $w c_1$, positive under the sign convention fixed at the start of the proof.

In either parity, a model parameter ranges over a non-empty open interval, the first cutpoint $\gamma_1$ in \eqref{eq:p_eq_g1_int_t3} when $s$ is even and the reverse edge weight $v$ in \eqref{eq:p_eq_v_int_t3} when $s$ is odd. For every admissible value of this parameter, \eqref{eq:p_eq_chains_t3} determines all $D_j$, and \eqref{eq:p_eq_dj_def_t3} returns strictly ordered cutpoints $\gamma_j$ through $e^{\gamma_j} = \parn{t_2 - t_1 e^{D_j}}/\parn{e^{D_j} - 1}$. Setting $\beta_{x, 1} = l_{x}\parn{y_1; \mathcal{M}_{Y \to X}} - \alpha_{x, 1}T\parn{y_1}$ and applying Lemma~\ref{lem:fwd_affine}, the forward marginal is recovered as
\begin{eqnarray}
    \frac{\pi_x}{\pi_1} = \exp\parn{\beta_{x, 1} - a\parn{\eta\parn{w}} + a\parn{\eta\parn{wx}}} > 0, \quad \forall x \in \curl{2, \dots, s},
\end{eqnarray}
which, after normalization, yields a valid categorical marginal. By construction, these forward parameters $\parn{\pi, w}$ with $w \neq 0$ and reverse parameters $\parn{\gamma, v}$ with $v \neq 0$ produce equal increments \eqref{eq:p_eq_diff_l_fwd_t3} and \eqref{eq:p_eq_diff_l_rev_t3}, and since the log-odds ratios at $y_1$ are matched through the choice of $\beta_{x, 1}$, the two models induce identical conditional laws $p_{X \mid Y}$ on $\mathcal{Y}$. In the reverse model the marginal of $Y$ is an unconstrained positive distribution on $\mathcal{Y}$, specified independently of the cutpoints $\gamma$ and the edge weight $v$, so we set it equal to the marginal induced by the forward model,
\begin{eqnarray}
    p_Y\parn{y; \mathcal{M}_{Y \to X}} = \sum_{x' \in \mathcal{X}} p_{X, Y}\parn{x', y; \mathcal{M}_{X \to Y}}, \quad \forall y \in \mathcal{Y}.
\end{eqnarray}
With the conditionals and the marginals of $Y$ now equal, the joint distributions coincide on $\mathcal{X} \times \mathcal{Y}$. As the admissible parameter ranges over a non-empty open interval, the constructed pairs form a family of forward and reverse parameters inducing the same joint law, and hence $\mathcal{M}_{X \to Y}$ and $\mathcal{M}_{Y \to X}$ are not distributionally identifiable.
\end{proof}

\subsection{Proof of Theorem~\ref{thm:converse_two_categories}}
\label{app:proof_converse_two_categories}
\begin{proof}
We prove both directions of the equivalence. As it is given that $\bars{\mathcal{X}} = 2$, the reverse model has a single finite cutpoint $\gamma_1$. For each $y \in \mathcal{Y}$ the only distinct pair $u, x \in \mathcal{X}$ is $u = 2, x = 1$, so the only log-odds-ratio is $R\parn{y, 2, 1; \mathcal{M}}$ for $\mathcal{M} \in \curl{\mathcal{M}_{X \to Y}, \mathcal{M}_{Y \to X}}$. For brevity, we define $l\parn{y; \mathcal{M}} \triangleq R\parn{y, 2, 1; \mathcal{M}}$.

Unlike the case $s \geq 3$, here the reverse log-odds-ratio $l\parn{y; \mathcal{M}_{Y \to X}}$ is affine in $y$ on $\mathcal{Y}$. Indeed,
\begin{eqnarray}
    l\parn{y; \mathcal{M}_{Y \to X}} &=& \log\parn{\frac{1 - \sigma\parn{\gamma_1 - vy}}{\sigma\parn{\gamma_1 - vy}}}, \quad \forall y \in \mathcal{Y}, \nonumber \\
    &=& \log\parn{e^{vy - \gamma_1}}, \quad \forall y \in \mathcal{Y}, \nonumber \\
    &=& vy - \gamma_1, \quad \forall y \in \mathcal{Y}.
    \label{eq:p_eq_rev_l_t4}
\end{eqnarray}
From Lemma~\ref{lem:fwd_affine}, for the forward model there exist $\alpha_{2,1} \neq 0$ and $\beta_{2,1}$ independent of $y$ such that
\begin{eqnarray}
    l\parn{y; \mathcal{M}_{X \to Y}} = \alpha_{2, 1}T\parn{y} + \beta_{2, 1}, \quad \forall y \in \mathcal{Y}.
    \label{eq:p_eq_fwd_l_t4}
\end{eqnarray}

We first prove that if $T\parn{y}$ is affine on the support, then the models are not distributionally identifiable. Suppose $T\parn{y} = c_1 y + c_0$ for all $y \in \mathcal{Y}$ with $c_1 \neq 0$. Given any forward model with parameters $\parn{\pi, w}$, substituting $T\parn{y} = c_1 y + c_0$ into \eqref{eq:p_eq_fwd_l_t4} gives
\begin{eqnarray}
    l\parn{y; \mathcal{M}_{X \to Y}} = \alpha_{2,1}c_1\, y + \parn{\alpha_{2,1}c_0 + \beta_{2, 1}}, \quad \forall y \in \mathcal{Y},
\end{eqnarray}
which is affine in $y$. Choosing the reverse parameters as
\begin{eqnarray}
    v = \alpha_{2,1}c_1 \neq 0, \qquad \gamma_1 = -\parn{\alpha_{2,1}c_0 + \beta_{2,1}},
\end{eqnarray}
where $v \neq 0$ since $\alpha_{2,1} \neq 0$ and $c_1 \neq 0$, the reverse log-odds-ratio in \eqref{eq:p_eq_rev_l_t4} becomes
\begin{eqnarray}
    l\parn{y; \mathcal{M}_{Y \to X}} = \alpha_{2,1}c_1\, y + \parn{\alpha_{2,1}c_0 + \beta_{2, 1}} = l\parn{y; \mathcal{M}_{X \to Y}}, \quad \forall y \in \mathcal{Y},
\end{eqnarray}
so the two models induce identical conditional laws $p_{X \mid Y}$ on $\mathcal{Y}$. It remains to match the marginals of $Y$. In the reverse model, the marginal of $Y$ is an unconstrained positive distribution on $\mathcal{Y}$,
specified independently of $\gamma_1$ and $v$, so we may set it equal to the marginal of $Y$ induced by the forward model,
\begin{eqnarray}
    p_Y\parn{y; \mathcal{M}_{Y \to X}} = \sum_{x' \in \mathcal{X}} p_{X, Y}\parn{x', y; \mathcal{M}_{X \to Y}}, \quad \forall y \in \mathcal{Y}.
\end{eqnarray}
With the conditionals and the marginals of $Y$ now equal, the joint distributions coincide on $\mathcal{X} \times \mathcal{Y}$, and hence the models are not distributionally identifiable.

We next prove that if $T\parn{y}$ is not affine on $\mathcal{Y}$, then the models are distributionally identifiable. By non-affineness on the support, there exist pairwise distinct $y_1, y_2, y_3 \in \mathcal{Y}$, which is possible since $\bars{\mathcal{Y}} \geq 3$, such that
\begin{eqnarray}
    \frac{T\parn{y_2} - T\parn{y_1}}{y_2 - y_1} \neq \frac{T\parn{y_3} - T\parn{y_1}}{y_3 - y_1}.
    \label{eq:p_eq_nonaffine_t4}
\end{eqnarray}
Suppose, for contradiction, that the models are not distributionally identifiable. Then there exist parameters inducing the same joint distribution for both models. Summing over $\mathcal{X}$ gives equal marginals of $Y$, and dividing gives equal conditionals, hence equal log-odds-ratios,
\begin{eqnarray}
    l\parn{y; \mathcal{M}_{X \to Y}} = l\parn{y; \mathcal{M}_{Y \to X}}, \quad \forall y \in \mathcal{Y}.
\end{eqnarray}
Substituting \eqref{eq:p_eq_rev_l_t4} and \eqref{eq:p_eq_fwd_l_t4},
\begin{eqnarray}
    \alpha_{2, 1}T\parn{y} + \beta_{2, 1} = vy - \gamma_1, \quad \forall y \in \mathcal{Y}.
    \label{eq:p_eq_fr_t4}
\end{eqnarray}
Evaluating \eqref{eq:p_eq_fr_t4} at $y_1, y_2, y_3$ and taking differences,
\begin{eqnarray}
    \alpha_{2, 1}\parn{T\parn{y_2} - T\parn{y_1}} = \parn{y_2 - y_1}v, \quad \alpha_{2, 1}\parn{T\parn{y_3} - T\parn{y_1}} = \parn{y_3 - y_1}v,
\end{eqnarray}
and since $y_2 \neq y_1$, $y_3 \neq y_1$, and $\alpha_{2,1} \neq 0$, dividing gives
\begin{eqnarray}
    \frac{T\parn{y_2} - T\parn{y_1}}{y_2 - y_1} = \frac{v}{\alpha_{2, 1}} = \frac{T\parn{y_3} - T\parn{y_1}}{y_3 - y_1}.
\end{eqnarray}
This contradicts \eqref{eq:p_eq_nonaffine_t4}. The contradiction arises solely from the assumption that a common joint distribution exists, so no forward and reverse parameters can induce the same joint distribution on $\mathcal{X} \times \mathcal{Y}$. Equivalently, the model classes $\mathcal{M}_{X \to Y}$ and $\mathcal{M}_{Y \to X}$ induce disjoint sets of joint distributions, and by Definition~\ref{def:d_id} the edge direction is distributionally identifiable.
\end{proof}

\subsection{Proof of Theorem~\ref{thm:multivariate_id}}
\label{app:proof_multivariate_id}
\begin{proof}
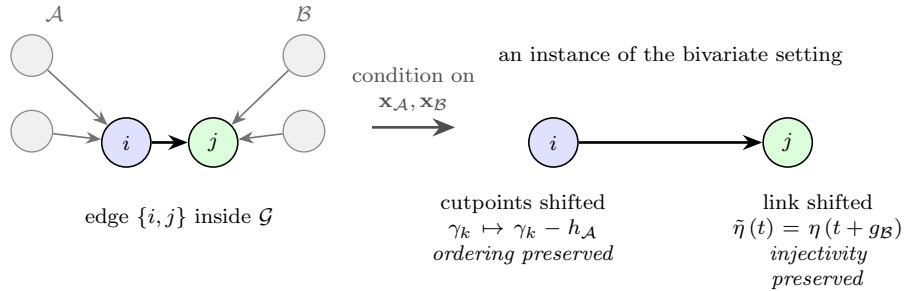
\begin{figure}[!htbp]
\centering
\begin{tikzpicture}[
  onode/.style = {circle, draw=black, line width=0.6pt, fill=blue!12,
                  minimum size=6.5mm, inner sep=0pt, font=\scriptsize},
  enode/.style = {circle, draw=black, line width=0.6pt, fill=green!14,
                  minimum size=6.5mm, inner sep=0pt, font=\scriptsize},
  onode2/.style = {circle, draw=black!45, line width=0.5pt, fill=black!6,
                  minimum size=5.5mm, inner sep=0pt, font=\scriptsize},
  ed/.style    = {-{Stealth[length=2mm,width=1.6mm]}, semithick, black!55},
  hed/.style   = {-{Stealth[length=2.6mm,width=2mm]}, line width=1pt, black},
]
  % ---- left: edge embedded in a DAG ----
  \node[onode2] (a1) at (-1.4,1.15) {};
  \node[onode2] (a2) at (-1.4,0.15) {};
  \node[onode2] (b1) at (2.2,1.15) {};
  \node[onode2] (b2) at (2.2,0.15) {};
  \node[onode]  (i)  at (-0.15,0.0) {$i$};
  \node[enode]  (j)  at (1.0,0.0)  {$j$};
  \draw[ed] (a1) -- (i);  \draw[ed] (a2) -- (i);
  \draw[ed] (b1) -- (j);  \draw[ed] (b2) -- (j);
  \draw[hed] (i) -- (j);
  \node[font=\scriptsize, text=black!60] at (-1.1,1.7) {$\mathcal{A}$};
  \node[font=\scriptsize, text=black!60] at (2.2,1.7)  {$\mathcal{B}$};
  \node[font=\scriptsize] at (0.55,-1.0) {edge $\curl{i,j}$ inside $\mathcal{G}$};

  % ---- arrow ----
  \draw[-{Stealth[length=3mm,width=2.4mm]}, line width=1pt, black!70]
       (3.1,0.2) -- (4.2,0.2)
       node[midway, above=2pt, font=\scriptsize, align=center]
       {condition on\\ $\vect{x}_{\mathcal{A}}, \vect{x}_{\mathcal{B}}$};

  % ---- right: reduced bivariate instance ----
  \node[onode] (i2) at (5.5,0.0) {$i$};
  \node[enode] (j2) at (8.6,0.0) {$j$};
  \draw[hed] (i2) -- (j2);
  \node[font=\scriptsize, anchor=north, text width=26mm, align=center] at (5.1,-0.55)
       {cutpoints shifted\\ $\gamma_k \mapsto \gamma_k - h_{\mathcal{A}}$\\ \emph{ordering preserved}};
  \node[font=\scriptsize, anchor=north, text width=26mm, align=center] at (9.0,-0.55)
       {link shifted\\ $\tilde{\eta}\parn{t} = \eta\parn{t + g_{\mathcal{B}}}$\\ \emph{injectivity preserved}};
  \node[font=\scriptsize] at (7.05,1.15) {an instance of the bivariate setting};
\end{tikzpicture}
\caption{The reduction behind Theorem~\ref{thm:multivariate_id}.}
\label{fig:mvreduction}
\end{figure}

Suppose, for contradiction, that the direction of the edge $\curl{i, j}$ is not distributionally identifiable. Then, from Definition~\ref{def:g_id}, there exist joint distributions $p_{\mathcal{V}} \in \mathcal{M}\parn{\mathcal{G}^{i \to j}}$ and $q_{\mathcal{V}} \in \mathcal{M}\parn{\mathcal{G}^{j \to i}}$ such that
\begin{eqnarray}
    p_{\mathcal{V}}\parn{\vect{x}} = q_{\mathcal{V}}\parn{\vect{x}}, \quad \forall\, \vect{x} \in \prod_{k \in \mathcal{V}} \mathcal{X}_k.
    \label{eq:p_eq_jnt_mv}
\end{eqnarray}
Write $\mathcal{Z} \triangleq \mathcal{V} \setminus \curl{i, j}$, and define $\mathcal{A} \triangleq \mathcal{P}\parn{i} \setminus \curl{j}$ and $\mathcal{B} \triangleq \mathcal{P}\parn{j} \setminus \curl{i}$. Since $\mathcal{G}^{i \to j}$ and $\mathcal{G}^{j \to i}$ differ only in the orientation of $\curl{i, j}$, the sets $\mathcal{A}$ and $\mathcal{B}$ are the same in both graphs, and every $k \in \mathcal{Z}$ has the same parent set in both graphs.

Every conditional distribution is determined by the joint distribution, so \eqref{eq:p_eq_jnt_mv} gives, for each $k \in \mathcal{Z}$,
\begin{eqnarray}
    p\parn{x_k \mid \vect{x}_{\mathcal{P}\parn{k}}} = q\parn{x_k \mid \vect{x}_{\mathcal{P}\parn{k}}}.
    \label{eq:p_eq_zcond_mv}
\end{eqnarray}
Fix any $\vect{x}_{\mathcal{Z}}$ in the support of $\mathcal{Z}$, which satisfies $p\parn{\mathcal{Z} = \vect{x}_{\mathcal{Z}}} > 0$ by the strict positivity of $p_{\mathcal{V}}$, and write $\vect{x}_{\mathcal{A}}$ and $\vect{x}_{\mathcal{B}}$ for its components indexed by $\mathcal{A}$ and $\mathcal{B}$. For distinct $u, x \in \curl{1, \dots, s}$ and $y \in \mathcal{X}_j$, define
\begin{eqnarray}
    R\parn{y, u, x} \triangleq \log\parn{\frac{p\parn{X_i = u \mid X_j = y, \mathcal{Z} = \vect{x}_{\mathcal{Z}}}}{p\parn{X_i = x \mid X_j = y, \mathcal{Z} = \vect{x}_{\mathcal{Z}}}}},
    \label{eq:p_eq_rdef_mv}
\end{eqnarray}
which by \eqref{eq:p_eq_jnt_mv} takes the same value whether computed from $p_{\mathcal{V}}$ or from $q_{\mathcal{V}}$.

Factorizing $p_{\mathcal{V}}$ according to $\mathcal{G}^{i \to j}$, in which $\mathcal{P}\parn{i} = \mathcal{A}$ and $\mathcal{P}\parn{j} = \mathcal{B} \cup \curl{i}$, and canceling the factors that do not depend on $X_i$, we get
\begin{eqnarray}
    R\parn{y, u, x} = \log\parn{\frac{p\parn{u \mid \vect{x}_{\mathcal{A}}}}{p\parn{x \mid \vect{x}_{\mathcal{A}}}}} + \log\parn{\frac{p\parn{y \mid u, \vect{x}_{\mathcal{B}}}}{p\parn{y \mid x, \vect{x}_{\mathcal{B}}}}} + \Lambda\parn{y, u, x},
    \label{eq:p_eq_rfwd_mv}
\end{eqnarray}
where
\begin{eqnarray}
    \Lambda\parn{y, u, x} \triangleq \sum_{k \in \mathcal{Z}} \log\parn{\frac{p\parn{x_k \mid \vect{x}_{\mathcal{P}\parn{k}}}\big\vert_{X_i = u}}{p\parn{x_k \mid \vect{x}_{\mathcal{P}\parn{k}}}\big\vert_{X_i = x}}}
\end{eqnarray}
collects the contributions of the nodes in $\mathcal{Z}$, the terms of which vanish for every $k$ with $i \notin \mathcal{P}\parn{k}$. Factorizing $q_{\mathcal{V}}$ according to $\mathcal{G}^{j \to i}$, in which $\mathcal{P}\parn{i} = \mathcal{A} \cup \curl{j}$ and $\mathcal{P}\parn{j} = \mathcal{B}$, and canceling in the same way,
\begin{eqnarray}
    R\parn{y, u, x} = \log\parn{\frac{q\parn{u \mid y, \vect{x}_{\mathcal{A}}}}{q\parn{x \mid y, \vect{x}_{\mathcal{A}}}}} + \Lambda'\parn{y, u, x},
    \label{eq:p_eq_rrev_mv}
\end{eqnarray}
where $\Lambda'$ is defined as $\Lambda$ with $p$ replaced by $q$. By \eqref{eq:p_eq_zcond_mv} the two sums coincide, $\Lambda = \Lambda'$, so equating \eqref{eq:p_eq_rfwd_mv} and \eqref{eq:p_eq_rrev_mv} gives
\begin{eqnarray}
    \log\parn{\frac{p\parn{u \mid \vect{x}_{\mathcal{A}}}}{p\parn{x \mid \vect{x}_{\mathcal{A}}}}} + \log\parn{\frac{p\parn{y \mid u, \vect{x}_{\mathcal{B}}}}{p\parn{y \mid x, \vect{x}_{\mathcal{B}}}}} = \log\parn{\frac{q\parn{u \mid y, \vect{x}_{\mathcal{A}}}}{q\parn{x \mid y, \vect{x}_{\mathcal{A}}}}}, \quad \forall\, y \in \mathcal{X}_j, \; u \neq x.
    \label{eq:p_eq_cancel_mv}
\end{eqnarray}

We now exhibit two bivariate models to which \eqref{eq:p_eq_cancel_mv} applies, the reduction being the one drawn in Figure~\ref{fig:mvreduction}. Let $g_{\mathcal{B}}$ denote the contribution of $\vect{x}_{\mathcal{B}}$ to the linear predictor of $X_j$, so that the conditional $p\parn{y \mid x, \vect{x}_{\mathcal{B}}}$ is the regular one-parameter exponential family conditional \eqref{eq:fwd_cond_Y} with sufficient statistic $T$, support $\mathcal{X}_j$, and natural parameter $\eta\parn{w x + g_{\mathcal{B}}}$. Writing $\tilde{\eta}\parn{t} \triangleq \eta\parn{t + g_{\mathcal{B}}}$, which is monotone injective since $\eta$ is, and taking the categorical marginal $p\parn{\cdot \mid \vect{x}_{\mathcal{A}}}$, which is strictly positive, the pair
\begin{eqnarray}
    \tilde{p}\parn{x, y} \triangleq p\parn{x \mid \vect{x}_{\mathcal{A}}} \, p\parn{y \mid x, \vect{x}_{\mathcal{B}}}
    \label{eq:p_eq_aux_fwd_mv}
\end{eqnarray}
is a forward model of the form $\mathcal{M}_{X \to Y}$ with edge weight $w \neq 0$, link $\tilde{\eta}$, and $s \geq 3$ categories, the support $\mathcal{X}_j$ of node $j$ playing the role of $\mathcal{Y}$ in the bivariate models of Section~\ref{ss:id_bivar}. Similarly, let $h_{\mathcal{A}}$ denote the contribution of $\vect{x}_{\mathcal{A}}$ to the linear predictor of $X_i$, so that $q\parn{x \mid y, \vect{x}_{\mathcal{A}}}$ is the ordered-logit conditional \eqref{eq:rev_cond_X} with cutpoints $\gamma_1 - h_{\mathcal{A}} < \cdots < \gamma_{s-1} - h_{\mathcal{A}}$ and edge weight $v \neq 0$, the shift preserving the strict ordering. Taking the marginal of $X_j$ to be the marginal induced by \eqref{eq:p_eq_aux_fwd_mv}, which is admissible since the marginal of $Y$ in $\mathcal{M}_{Y \to X}$ is an unconstrained positive distribution on $\mathcal{X}_j$, the pair
\begin{eqnarray}
    \tilde{q}\parn{x, y} \triangleq \parn{\sum_{x' \in \curl{1, \dots, s}} \tilde{p}\parn{x', y}} q\parn{x \mid y, \vect{x}_{\mathcal{A}}}
    \label{eq:p_eq_aux_rev_mv}
\end{eqnarray}
is a reverse model of the form $\mathcal{M}_{Y \to X}$.

By \eqref{eq:p_eq_cancel_mv}, the log-odds ratios of $\tilde{p}$ and $\tilde{q}$ against category $1$ agree for every $y \in \mathcal{X}_j$, and a categorical law is determined by its log-odds against a fixed category, so the conditional laws of $X_i$ given $X_j$ coincide under $\tilde{p}$ and $\tilde{q}$. The marginals of $X_j$ coincide by the construction in \eqref{eq:p_eq_aux_rev_mv}, and therefore
\begin{eqnarray}
    \tilde{p}\parn{x, y} = \tilde{q}\parn{x, y}, \quad \forall\, x \in \curl{1, \dots, s}, \; y \in \mathcal{X}_j.
    \label{eq:p_eq_aux_eq_mv}
\end{eqnarray}
Since $s \geq 3$ and $\bars{\mathcal{X}_j} \geq 3$, Theorem~\ref{thm:general_id} states that no forward model of the form $\mathcal{M}_{X \to Y}$ and reverse model of the form $\mathcal{M}_{Y \to X}$ can induce the same joint distribution, which contradicts \eqref{eq:p_eq_aux_eq_mv}.

The contradiction establishes that the bivariate models \eqref{eq:p_eq_aux_fwd_mv} and \eqref{eq:p_eq_aux_rev_mv} cannot coincide, and therefore that \eqref{eq:p_eq_cancel_mv} cannot hold. Since \eqref{eq:p_eq_cancel_mv} was obtained from the two factorizations \eqref{eq:p_eq_rfwd_mv} and \eqref{eq:p_eq_rrev_mv} of the log-odds ratio \eqref{eq:p_eq_rdef_mv}, which are necessary consequences of the equality of joints \eqref{eq:p_eq_jnt_mv}, the assumption \eqref{eq:p_eq_jnt_mv} is untenable. We conclude that no $p_{\mathcal{V}} \in \mathcal{M}\parn{\mathcal{G}^{i \to j}}$ and $q_{\mathcal{V}} \in \mathcal{M}\parn{\mathcal{G}^{j \to i}}$ can coincide on $\prod_{k \in \mathcal{V}} \mathcal{X}_k$. Equivalently, the two orientations induce disjoint sets of joint distributions, and by Definition~\ref{def:g_id} the direction of the edge $\curl{i, j}$ is distributionally identifiable.
\end{proof}

\subsection{Proof of Theorem~\ref{thm:orient_id}}
\label{app:proof_orient_id}
\begin{proof}
We proceed by contradiction, reducing any competing orientation to a bivariate comparison governed by Theorem~\ref{thm:general_id}. Densities are taken throughout with respect to the product of the counting measure on the ordinal coordinates and the dominating measure of the relevant exponential family on the remaining coordinates, and every conditional in the model is strictly positive on its support, so $p_{\mathcal{V}}$ is strictly positive.

We first record that $p_{\mathcal{V}}$ is causally minimal with respect to $\mathcal{G}$. By Proposition 4 of \citet{peters2014causal}, causal minimality with respect to a graph $\mathcal{H}$ is equivalent to
\begin{eqnarray}
    X_k \not\perp X_l \mid \vect{x}_{\mathcal{P}_{\mathcal{H}}\parn{k} \setminus \curl{l}}, \quad \forall\, k \in \mathcal{V}, \; l \in \mathcal{P}_{\mathcal{H}}\parn{k}.
    \label{eq:p_eq_min_or}
\end{eqnarray}
Fix such a $k$ and $l$ and hold the remaining parents at any value. Since the weight of the edge from $l$ into $k$ is nonzero, two distinct values of $X_l$ give two distinct values of the linear predictor of $X_k$. These index distinct conditional laws, by injectivity of $\eta$ together with minimality of $T_k$ when $k$ is an exponential family node, and by strict monotonicity of $\sigma$ when $k$ is ordinal. Both parent configurations carry positive probability, so \eqref{eq:p_eq_min_or} holds.

Assume, for contradiction, that some $\mathcal{G}' \in \mathcal{O}\parn{\mathcal{G}_u}$ with $\mathcal{G}' \neq \mathcal{G}$ carries an ordinal-exponential LPM generating $p_{\mathcal{V}}$. Then $p_{\mathcal{V}}$ is strictly positive and is Markov and causally minimal with respect to both $\mathcal{G}$ and $\mathcal{G}'$, which are the hypotheses of part (i) of Proposition 29 of \citet{peters2014causal}. Two distinct orientations must disagree on some edge, but an arbitrary disagreeing edge need not admit a conditioning set lying outside the descendants of its endpoints in both graphs at once, which is what the reduction below requires. The proposition produces one that does, supplying nodes $a$ and $b$, with $b \to a$ in $\mathcal{G}$ and $a \to b$ in $\mathcal{G}'$, together with the sets
\begin{eqnarray}
    \mathcal{A} \triangleq \mathcal{P}_{\mathcal{G}}\parn{a} \setminus \curl{b}, \qquad \mathcal{B} \triangleq \mathcal{P}_{\mathcal{G}'}\parn{b} \setminus \curl{a}, \qquad \mathcal{C} \triangleq \mathcal{A} \cup \mathcal{B},
    \label{eq:p_eq_sets_or}
\end{eqnarray}
every member of $\mathcal{C}$ being a non-descendant of $a$ in $\mathcal{G}$ other than $b$ and a non-descendant of $b$ in $\mathcal{G}'$ other than $a$. Since $\mathcal{G}$ and $\mathcal{G}'$ orient the common skeleton, $\parn{a, b} \in \mathcal{E}_u$, so exactly one of $a$ and $b$ is ordinal and the other an exponential family node.

Fix any $\vect{x}_{\mathcal{C}}$ in the support of $\mathcal{C}$, which carries positive probability. By the non-descendant property just recorded, every member of $\mathcal{C} \setminus \mathcal{A}$ is a non-descendant of $a$ in $\mathcal{G}$ that is not a parent of $a$, so the local Markov property renders $X_a$ independent of them given its parents, and $p\parn{x_a \mid x_b, \vect{x}_{\mathcal{C}}} = p\parn{x_a \mid x_b, \vect{x}_{\mathcal{A}}}$ is the parametric conditional of $X_a$ given its parents in $\mathcal{G}$. Hence
\begin{eqnarray}
    p\parn{x_a, x_b \mid \vect{x}_{\mathcal{C}}} = p\parn{x_b \mid \vect{x}_{\mathcal{C}}} \, p\parn{x_a \mid x_b, \vect{x}_{\mathcal{A}}},
    \label{eq:p_eq_pair_g_or}
\end{eqnarray}
and the same argument applied to $\mathcal{G}'$, using $\mathcal{P}_{\mathcal{G}'}\parn{b} = \mathcal{B} \cup \curl{a}$, gives
\begin{eqnarray}
    q\parn{x_a, x_b \mid \vect{x}_{\mathcal{C}}} = q\parn{x_a \mid \vect{x}_{\mathcal{C}}} \, q\parn{x_b \mid x_a, \vect{x}_{\mathcal{B}}}.
    \label{eq:p_eq_pair_gp_or}
\end{eqnarray}

Suppose $b$ is ordinal and $a$ an exponential family node. In \eqref{eq:p_eq_pair_g_or} the factor $p\parn{x_b \mid \vect{x}_{\mathcal{C}}}$ is a strictly positive categorical law, and $p\parn{x_a \mid x_b, \vect{x}_{\mathcal{A}}}$ is the conditional \eqref{eq:fwd_cond_Y} with natural parameter $\eta\parn{w x_b + g_{\mathcal{A}}}$, where $g_{\mathcal{A}}$ collects the contribution of $\vect{x}_{\mathcal{A}}$. Writing $\tilde{\eta}\parn{t} \triangleq \eta\parn{t + g_{\mathcal{A}}}$, monotone injective since $\eta$ is, \eqref{eq:p_eq_pair_g_or} is a forward model of the form $\mathcal{M}_{X \to Y}$ with edge weight $w \neq 0$. In \eqref{eq:p_eq_pair_gp_or} the factor $q\parn{x_a \mid \vect{x}_{\mathcal{C}}}$ is a strictly positive law on $\mathcal{X}_a$, and $q\parn{x_b \mid x_a, \vect{x}_{\mathcal{B}}}$ is the ordered-logit conditional \eqref{eq:rev_cond_X} with cutpoints displaced by the contribution of $\vect{x}_{\mathcal{B}}$, which preserves their strict ordering, so \eqref{eq:p_eq_pair_gp_or} is a reverse model of the form $\mathcal{M}_{Y \to X}$ with edge weight $v \neq 0$. Exchanging the two displays covers the case in which $a$ is ordinal.

The ordinal member of the pair carries at least three categories and the exponential family member at least three points of support, so by Theorem~\ref{thm:general_id} the two model classes induce disjoint sets of joint distributions on $\parn{X_a, X_b}$. But \eqref{eq:p_eq_pair_g_or} and \eqref{eq:p_eq_pair_gp_or} are both computed from $p_{\mathcal{V}}$ and therefore coincide. This contradiction establishes that no such $\mathcal{G}'$ exists.
\end{proof}

\subsection{Proof of Theorem~\ref{thm:multivariate_converse}}
\label{app:proof_multivariate_converse}
\begin{proof}
We construct joint distributions on $\mathcal{V}$ coinciding under the two orientations. Write $\mathcal{Z} \triangleq \mathcal{V} \setminus \curl{i, j}$. Since neither $i$ nor $j$ is adjacent in $\mathcal{G}_u$ to any node other than the other endpoint, both have no parents and no children in $\mathcal{Z}$ in either orientation, and every $k \in \mathcal{Z}$ has the same parent set, contained in $\mathcal{Z}$, in $\mathcal{G}^{i \to j}$ and $\mathcal{G}^{j \to i}$. Factorizing according to each graph therefore gives
\begin{eqnarray}
    p_{\mathcal{V}}\parn{\vect{x}} = p_{i, j}\parn{x, y} \prod_{k \in \mathcal{Z}} p\parn{x_k \mid \vect{x}_{\mathcal{P}\parn{k}}}, \nonumber \\
    q_{\mathcal{V}}\parn{\vect{x}} = q_{i, j}\parn{x, y} \prod_{k \in \mathcal{Z}} q\parn{x_k \mid \vect{x}_{\mathcal{P}\parn{k}}},
    \label{eq:p_eq_split_mvc}
\end{eqnarray}
so that the joint distribution over $\mathcal{V}$ factorizes into the law of the pair $\parn{X_i, X_j}$ and a factor depending only on $\mathcal{Z}$.

In $\mathcal{G}^{i \to j}$ the node $i$ is a root and therefore follows an arbitrary strictly positive categorical marginal on $\curl{1, \dots, s}$, and $X_j$ follows the regular one-parameter exponential family conditional \eqref{eq:fwd_cond_Y} given $X_i$ with sufficient statistic $T$, link $\eta$, and support $\mathcal{X}_j$. Hence $p_{i, j}$ ranges exactly over the forward models $\mathcal{M}_{X \to Y}$ with edge weight $w \neq 0$. In $\mathcal{G}^{j \to i}$, the node $j$ is a root and therefore follows an arbitrary strictly positive marginal on $\mathcal{X}_j$, and $X_i$ follows the ordered-logit conditional \eqref{eq:rev_cond_X} given $X_j$ with cutpoints $\gamma_1 < \cdots < \gamma_{s-1}$. Hence $q_{i, j}$ ranges exactly over the reverse models $\mathcal{M}_{Y \to X}$ with edge weight $v \neq 0$.

Consider first the case $\bars{\mathcal{X}_j} = 2$ with an affine sufficient statistic and a canonical link. By Theorem~\ref{thm:converse_two_points} there exist forward parameters $\parn{\pi, w}$ with $w \neq 0$ and reverse parameters $\parn{\gamma, v}$ with $v \neq 0$ such that
\begin{eqnarray}
    p_{i, j}\parn{x, y} = q_{i, j}\parn{x, y}, \quad \forall\, x \in \curl{1, \dots, s}, \; y \in \mathcal{X}_j.
    \label{eq:p_eq_pair_eq_mvc}
\end{eqnarray}
Consider next the case $s = 2$ with a sufficient statistic affine on $\mathcal{X}_j$. By Theorem~\ref{thm:converse_two_categories} there exist forward parameters $\parn{\pi, w}$ with $w \neq 0$ and reverse parameters $\parn{\gamma_1, v}$ with $v \neq 0$ satisfying \eqref{eq:p_eq_pair_eq_mvc}. In either case the law of the pair coincides under the two orientations.

Choosing the conditionals of the nodes in $\mathcal{Z}$ to be identical under the two graphs, which is admissible since their parent sets coincide, the second factors in \eqref{eq:p_eq_split_mvc} agree, and combining this with \eqref{eq:p_eq_pair_eq_mvc} gives
\begin{eqnarray}
    p_{\mathcal{V}}\parn{\vect{x}} = q_{\mathcal{V}}\parn{\vect{x}}, \quad \forall\, \vect{x} \in \prod_{k \in \mathcal{V}} \mathcal{X}_k.
\end{eqnarray}
The two orientations, therefore, induce joint distributions that coincide on $\prod_{k \in \mathcal{V}} \mathcal{X}_k$, so their sets of induced joint distributions are not disjoint, and by Definition~\ref{def:g_id} the direction of the edge $\parn{i, j}$ is not distributionally identifiable.
\end{proof}

\section{Algorithms for Mixed DAG Discovery}
\label{app:alg}

\subsection{Score and Optimization Problem}
\label{app:alg_score}

Let $\vect{X} \in \mathbb{R}^{n \times d}$ be an $n$-sample dataset generated by a $d$-node DAG $\mathcal{G} = \parn{\mathcal{V}, \mathcal{E}}$ with $\mathcal{V} = \mathcal{V}_{\operatorname{ord}} \cup \mathcal{V}_{\operatorname{exp}}$. Edges are permitted only between nodes of unlike type, enforced by the mask
\begin{equation}
\sqr{\vect{M}}_{i, j} = \begin{cases} 0, & \text{if } i, j \in \mathcal{V}_{\operatorname{ord}} \text{ or } i, j \in \mathcal{V}_{\operatorname{exp}}, \\ 1, & \text{otherwise}, \end{cases}
\label{eq:bipartite_mask}
\end{equation}
so $\mathcal{G}$ is bipartite by construction. Root nodes of a candidate structure are fitted rather than fixed: a root ordinal node retains its $s_i - 1$ cutpoints, giving an unconstrained categorical marginal on $\curl{1, \dots, s_i}$, and a root exponential-family node has an empty linear predictor with its natural parameter fitted as a single free scalar. Both conventions are narrower than the arbitrary positive marginal of Section~\ref{sec:bg}, so restrict rather than enlarge the model class.

Let $\mathcal{P}_{\mathcal{G}_{\operatorname{est}}}\parn{i}$ be the parent set of node $i$ under a candidate $\mathcal{G}_{\operatorname{est}}$ with weighted adjacency matrix $\vect{W}_{\operatorname{est}}$, and $k_i$ the number of fitted parameters in its conditional,
\begin{equation}
k_i = \begin{cases} \max\curl{\bars{\mathcal{P}_{\mathcal{G}_{\operatorname{est}}}\parn{i}}, 1}, & i \in \mathcal{V}_{\operatorname{exp}}, \\ \bars{\mathcal{P}_{\mathcal{G}_{\operatorname{est}}}\parn{i}} + s_i - 1, & i \in \mathcal{V}_{\operatorname{ord}}. \end{cases}
\label{eq:param_count}
\end{equation}
Candidates are scored by the Bayesian Information Criterion \citep{chickering2002optimal},
\begin{eqnarray}
\operatorname{BIC}\parn{\mathcal{G}_{\operatorname{est}}; \vect{X}}
 \triangleq -\sum_{j = 1}^{n} \sum_{i = 1}^{d} \ln p\parn{\vect{X}_{j, i} \mid \vect{X}_{j, \mathcal{P}_{\mathcal{G}_{\operatorname{est}}}\parn{i}}; \vect{W}_{\operatorname{est}}, \boldsymbol{\gamma}_i} + \, \frac{\log n}{2} \sum_{i = 1}^{d} k_i,
\label{eq:bic_app}
\end{eqnarray}
repeating \eqref{eq:bic}. Both terms are sums over nodes, so the score decomposes into node contributions. Writing $\mathcal{G}_u = \parn{\mathcal{V}, \mathcal{E}_u}$ for the skeleton, taken as given, and $\mathcal{O}\parn{\mathcal{G}_u}$ for its acyclic orientations, each respecting \eqref{eq:bipartite_mask}, the optimization problem is
\begin{equation}
\mathcal{G}_{\operatorname{est}}^{*} = \arg\min_{\mathcal{G}_{\operatorname{est}} \in \mathcal{O}\parn{\mathcal{G}_u}} \operatorname{BIC}\parn{\mathcal{G}_{\operatorname{est}}; \vect{X}},
\label{eq:opt_problem}
\end{equation}
with $\parn{\vect{W}_{\operatorname{est}}, \curl{\boldsymbol{\gamma}_i}_{i \in \mathcal{V}_{\operatorname{ord}}}}$ obtained by per-node conditional maximum likelihood estimation (Appendix~\ref{app:per_node_mle}). The node conditionals available to $\mathcal{V}_{\operatorname{exp}}$ are listed in Table~\ref{tab:exp_fam_dists}.

\begin{table}[!htbp]
\caption{Regular one-parameter exponential families used as node conditionals.}
\label{tab:exp_fam_dists}
\setlength{\tabcolsep}{4pt}
\small
\centering
\begin{tabular}{@{} l l l l l @{}}
\hline
\textbf{Distribution} & \textbf{Natural Parameter} & \textbf{Link Function} & \textbf{Suff.\ Stat.} & \textbf{Support} \\
 & $\eta$ & $\eta = g(z_i)$ & $T(x)$ & $\mathcal{X}$ \\
\hline
Exponential & $-\lambda,\; (\lambda > 0)$ & $\eta = -e^{z_i}$ & $x$ & $[0, \infty)$ \\
Poisson & $\log \lambda,\; (\lambda > 0)$ & $\eta = z_i$ & $x$ & $\{0, 1, \dots\}$ \\
Gaussian (fixed variance $\sigma^2$) & $\frac{\mu}{\sigma^2},\; (\mu \in \mathbb{R})$ & $\eta = z_i$ & $x$ & $\mathbb{R}$ \\
Gamma (fixed shape $\alpha$) & $-\beta,\; (\beta > 0)$ & $\eta = -e^{z_i}$ & $x$ & $(0, \infty)$ \\
Binomial (fixed trials $n \geq 3$) & $\log\frac{p}{1 - p},\; (p \in (0,1))$ & $\eta = z_i$ & $x$ & $\{0, \dots, n\}$ \\
Pascal (fixed successes $r$) & $\log(1 - p),\; (p \in (0,1))$ & $\eta = -\log(1 + e^{z_i})$ & $x$ & $\{0, 1, \dots\}$ \\
Gamma (fixed rate $\beta$) & $\alpha - 1,\; (\alpha > 0)$ & $\eta = e^{z_i} - 1$ & $\log(x)$ & $(0, \infty)$ \\
\hline
\end{tabular}
\end{table}

\subsection{Search Procedures}
\label{app:algorithms}

A skeleton with $\bars{\mathcal{E}_u}$ edges admits at most $2^{\bars{\mathcal{E}_u}}$ orientations, the acyclic ones constituting $\mathcal{O}\parn{\mathcal{G}_u}$. We solve \eqref{eq:opt_problem} by two procedures.

\textbf{Exhaustive search.} Enumerate $\mathcal{O}\parn{\mathcal{G}_u}$, fit each member, and return the minimizer. The global minimizer is recovered exactly, so any deviation from the truth is finite-sample noise rather than optimization error. Feasible up to a few dozen edges.

\textbf{Greedy reversal search.} Start from an arbitrary acyclic orientation and repeatedly reverse the single edge whose reversal most reduces the score, subject to acyclicity, halting when no reversal improves the score. Insertion and deletion do not arise, the skeleton being fixed. By decomposability of \eqref{eq:bic_app} a reversal refits only the two endpoints, giving per-iteration cost $O\parn{\bars{\mathcal{E}_u}}$. The procedure is a heuristic and may halt where several joint reversals would improve the score. Algorithm~\ref{alg:greedy_bic} states the procedure.

\begin{algorithm}[ht]
\caption{Greedy Orientation Search with BIC Score}
\label{alg:greedy_bic}
\begin{algorithmic}[1]
\Require Data matrix $\vect{X}$, skeleton $\mathcal{G}_u = (\mathcal{V}, \mathcal{E}_u)$, initial acyclic orientation $\mathcal{G}_{\operatorname{est}}$
\Ensure Estimated orientation $\mathcal{G}_{\operatorname{est}}^{*}$

\State Compute $\operatorname{BIC}\parn{\mathcal{G}_{\operatorname{est}}; \vect{X}}$ via \eqref{eq:bic_app}

\Repeat
    \State $\Delta^{*} \gets 0$, \texttt{best\_edge} $\gets$ \texttt{None}
    \For{each edge $e \in \mathcal{E}_u$ whose reversal in $\mathcal{G}_{\operatorname{est}}$ preserves acyclicity}
        \State Let $\mathcal{G}_e$ be $\mathcal{G}_{\operatorname{est}}$ with $e$ reversed
        \State Compute $\Delta_e \gets \operatorname{BIC}\parn{\mathcal{G}_e; \vect{X}} - \operatorname{BIC}\parn{\mathcal{G}_{\operatorname{est}}; \vect{X}}$ by refitting the two endpoints of $e$
        \If{$\Delta_e < \Delta^{*}$}
            \State $\Delta^{*} \gets \Delta_e$, \texttt{best\_edge} $\gets e$
        \EndIf
    \EndFor
    \If{\texttt{best\_edge} is not \texttt{None}}
        \State Reverse \texttt{best\_edge} in $\mathcal{G}_{\operatorname{est}}$
    \EndIf
\Until{\texttt{best\_edge} is \texttt{None}}

\State \Return $\mathcal{G}_{\operatorname{est}}^{*} \gets \mathcal{G}_{\operatorname{est}}$
\end{algorithmic}
\end{algorithm}

\section{Per-Node Conditional MLE}
\label{app:per_node_mle}

Both algorithms of Appendix~\ref{app:alg} score a candidate DAG by fitting the per-node conditional parameters by maximum likelihood on a fixed parent set, then computing \eqref{eq:bic_app}. Since the joint negative log-likelihood and the penalty decompose over nodes, the fitting can be performed one node at a time. We describe the two node types separately.

\subsection{Exponential-Family Nodes}
\label{app:per_node_expfam}

For node $i \in \mathcal{V}_{\operatorname{exp}}$ with conditional family $p_i$ and known sufficient statistic $T_i$, the conditional MLE problem is
\begin{equation}
\hat{\vect{w}}_i = \arg\min_{\vect{w} \in \mathbb{R}^{\bars{\mathcal{P}\parn{i}}}} -\frac{1}{n} \sum_{j = 1}^{n} \log p_i\!\parn{\vect{X}_{j, i} \mid \vect{X}_{j, \mathcal{P}\parn{i}}; \vect{w}},
\label{eq:expfam_mle}
\end{equation}
where the conditional log-likelihood takes the exponential-family form with natural parameter $\eta_i = g_i\parn{\vect{w}^{\mathsf{T}} \vect{X}_{j, \mathcal{P}\parn{i}}}$. The objective in \eqref{eq:expfam_mle} is convex in $\vect{w}$ for the families of Table~\ref{tab:exp_fam_dists} under the canonical link, and we solve \eqref{eq:expfam_mle} via L-BFGS-B \citep{byrd1995limited} initialized at $\vect{w} = \vect{0}$. For families with an exponential link (Poisson, Exponential, Gamma with fixed shape, Gamma with fixed rate), the linear predictor is constrained within an interval keeping $\bars{\vect{w}^{\mathsf{T}} \vect{X}_{j, \mathcal{P}\parn{i}}}$ bounded for the bulk of the sample, preventing numerical overflow in $\exp$ during the line search. The procedure terminates when the relative change in the objective falls below $\epsilon_{\operatorname{f}} = 10^{-8}$ or after $200$ iterations.

\subsection{Ordinal Nodes}
\label{app:per_node_ord}

For node $i \in \mathcal{V}_{\operatorname{ord}}$ with $s_i$ categories, the conditional MLE problem jointly fits the weight vector $\vect{w}_i$ and the cutpoints $\boldsymbol{\gamma}_i = \parn{\gamma_{i, 1}, \dots, \gamma_{i, s_i - 1}}$. The conditional negative log-likelihood is
\begin{eqnarray}
l_i\parn{\vect{w}_i, \boldsymbol{\gamma}_i; \vect{X}} &=& -\frac{1}{n} \sum_{j = 1}^{n} \log\Bigl[\sigma\parn{\gamma_{i, \vect{X}_{j, i}} - \vect{w}_i^{\mathsf{T}} \vect{X}_{j, \mathcal{P}\parn{i}}} - \, \sigma\parn{\gamma_{i, \vect{X}_{j, i} - 1} - \vect{w}_i^{\mathsf{T}} \vect{X}_{j, \mathcal{P}\parn{i}}}\Bigr],
\nonumber \\
\label{eq:ordinal_nll}
\end{eqnarray}
where $\sigma$ is the logistic sigmoid and the sentinel values $\gamma_{i, 0} = -\infty$, $\gamma_{i, s_i} = +\infty$ are used as needed at the boundary categories.

To enforce the strict ordering constraint $\gamma_{i, 1} < \cdots < \gamma_{i, s_i - 1}$ without a constrained optimizer, we reparametrize the cutpoints in terms of unconstrained parameters $\boldsymbol{\alpha}_i \in \mathbb{R}^{s_i - 1}$ via
\begin{eqnarray}
\gamma_{i, 1} &=& \alpha_{i, 1}, \nonumber \\
\gamma_{i, k} &=& \gamma_{i, k - 1} + \log\parn{1 + \exp\parn{\alpha_{i, k}}}, \quad k = 2, \dots, s_i - 1. \nonumber \\
\label{eq:cutpoint_reparam}
\end{eqnarray}
The softplus increment $\log\parn{1 + \exp\parn{\alpha_{i, k}}} > 0$ guarantees strict ordering for any $\boldsymbol{\alpha}_i \in \mathbb{R}^{s_i - 1}$, and the joint parameter vector $\parn{\vect{w}_i, \boldsymbol{\alpha}_i} \in \mathbb{R}^{\bars{\mathcal{P}\parn{i}} + s_i - 1}$ is then optimized over an unconstrained domain. We set $\vect{w}_i = \vect{0}$ and $\boldsymbol{\alpha}_i$ to the inverse of the empirical cumulative distribution function of $X_i$ via the logit link, then convert the resulting cutpoint differences to the softplus parametrization. The joint optimization is carried out via L-BFGS-B with analytical gradients, using the same convergence tolerance as in the exponential-family case.

\section{Multi-Node Experiments}
\label{app:d_node}

Setup as in Appendix~\ref{app:exp_details_dnode}: $d = 20$ bipartite skeletons, greedy reversal search (Algorithm~\ref{alg:greedy_bic}) initialized at a uniformly drawn acyclic orientation, sparse $\parn{p = 0.1}$ and dense $\parn{p = 0.9}$ regimes with $10$ and $90$ edges, eight distribution modes, $B = 1000$ trials per configuration. Figure~\ref{fig:d_node_nshd} reports $\rho$ against $N$.

\begin{figure}[!htbp]

\subfigure[Sparse bipartite DAG, $d = 20$, $10$ edges]{
\includegraphics[width=0.48\textwidth]{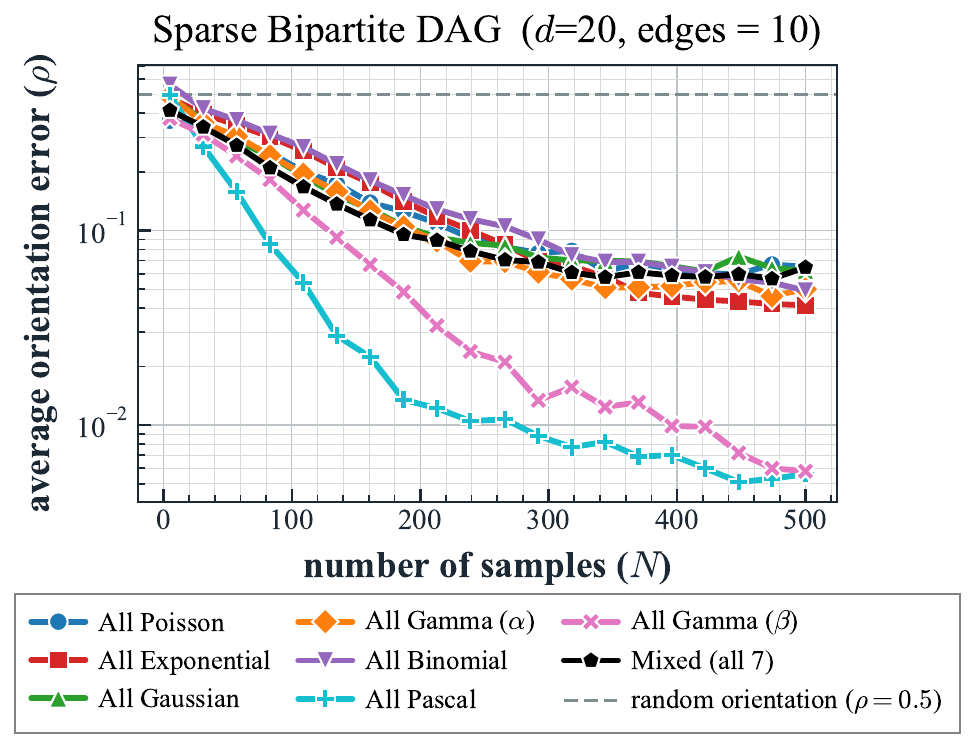}
\label{fig:dn_sparse_nshd}
}
\subfigure[Dense bipartite DAG, $d = 20$, $90$ edges]{
\includegraphics[width=0.48\textwidth]{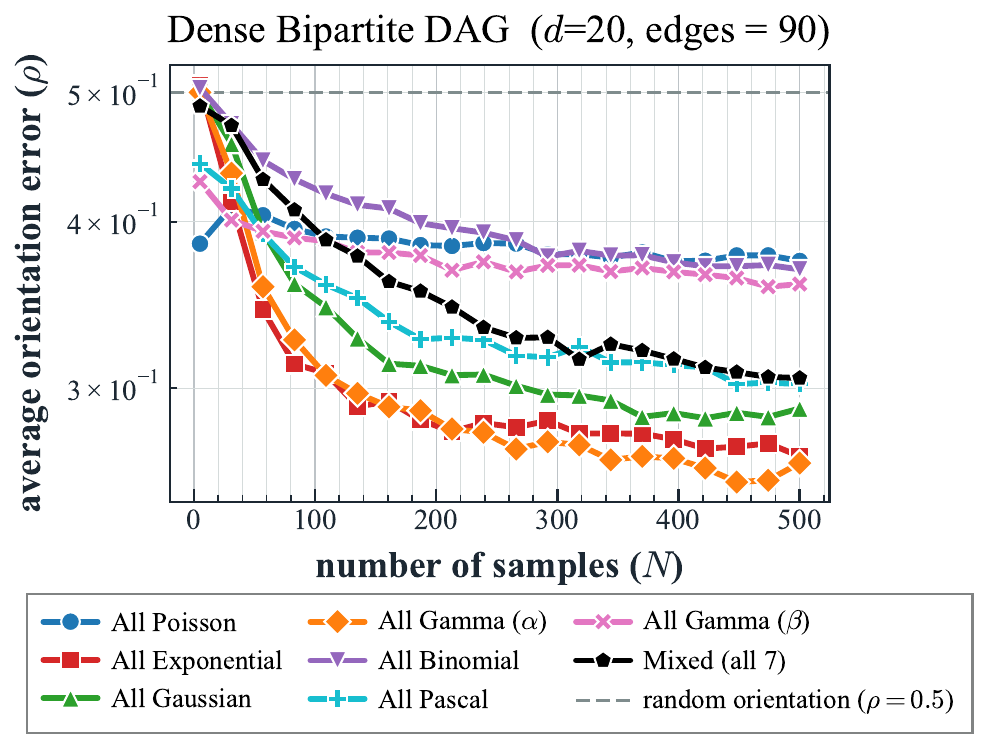}
\label{fig:dn_dense_nshd}
}
\caption{Orientation error rate $\rho$ versus sample size $N$ for $d = 20$ bipartite DAGs,
sparse and dense, over $B = 1000$ trials per configuration.}
\label{fig:d_node_nshd}
\end{figure}

\textbf{Sparse regime.} All eight modes decay across the range, with most trials at $N = 500$ orienting every edge correctly. The modes separate rather than cluster, Pascal and Gamma~$\parn{\beta}$ lowest and Poisson and the mixed mode highest, spanning about one order of magnitude at $N = 500$.

\textbf{Dense regime.} Gamma~$\parn{\alpha}$, Exponential, and Gaussian fall furthest; the mixed mode and Pascal occupy the middle; Poisson, Binomial, and Gamma~$\parn{\beta}$ plateau highest. No mode reaches the floor, and the curves are flat to within trial-to-trial variation over the final portion of the range. The dense model carries $119$ free parameters against $45$ for the sparse one, giving roughly four observations per parameter at $N = 500$; the search is moreover NP-hard in the large-sample limit for any consistent scoring criterion \citep{chickering2004large} and tractable only once the graph is sparse \citep{claassen2013learning}. Neither consideration bears on the population-level identifiability of Section~\ref{sec:id}.

The mixed mode lies within the envelope of the homogeneous modes in both regimes.

\section{Experimentation Details}
\label{app:exp_details}

Implementation in Python \texttt{3.13.14} with NumPy \texttt{2.3.4} and SciPy \texttt{1.16.3}, executed on a multi-core CPU cluster. Codebase at \url{https://github.com/SamMathelete/Ordinal_EF}.

\subsection{Generative Parameters}
\label{app:exp_details_common}

Ordinal cutpoints were generated as $s - 1$ draws from $\operatorname{Uniform}\parn{-1, 1}$, sorted, with $\gamma_{i,0} = -\infty$ and $\gamma_{i,s} = +\infty$ appended; these were drawn once per configuration and held fixed across sample sizes and trials. The link function was ordered logit throughout, and exponential-family hyperparameters were fixed across all experiments: Poisson and Exponential had no free hyperparameters, Gaussian used $\sigma^2 = 1.0$, Gamma with fixed shape used $\alpha = 2.0$, Binomial used $n_{\operatorname{trials}} = 5$, Pascal used $r = 3$, and Gamma with fixed rate used $\beta = 1.0$. For root nodes, the linear predictor is identically zero, so a root exponential-family node has natural parameter $\eta_i\parn{0}$ and a root ordinal node has $\operatorname{Pr}\parn{X_i \leq x} = \sigma\parn{\gamma_{i, x}}$ for $x = 1, \dots, s-1$; this gives unit rate (Exponential, Gamma with fixed shape), unit mean (Poisson), zero mean (Gaussian), success probability $1/2$ (Binomial, Pascal), and unit shape (Gamma with fixed rate), and no root parameter is held at its generating value during estimation. Seeding used a single integer seed of $7$, with per-configuration seeds adding a stable hash of the configuration identifiers, and the $b$-th data replicate with the configuration seed offset by $b$.

\subsection{Three-Node Experiments}
\label{app:exp_details_3node}

The skeleton $X_1 - X_2 - X_3$ was supplied and never estimated, with $X_1, X_3 \in \mathcal{V}_{\operatorname{ord}}$ for $s = 4$ and $X_2 \in \mathcal{V}_{\operatorname{exp}}$. The candidate set $\mathcal{O}\parn{\mathcal{G}_u} = \curl{\mathcal{G}_1, \mathcal{G}_2, \mathcal{G}_3, X_1 \leftarrow X_2 \leftarrow X_3}$ was enumerated exhaustively. Edge weights were set to $w = 0.6$, with $N$ ranging over $20$ equally spaced values in $\sqr{5, 500}$ and $B = 10000$ trials per configuration. One configuration was run per family of Table~\ref{tab:exp_fam_dists}.

\subsection{Multi-Node Experiments}
\label{app:exp_details_dnode}

The dimension was $d = 20$, with $\bars{\mathcal{V}_{\operatorname{ord}}}$ drawn uniformly from $\curl{\lfloor d/2 \rfloor - 1, \lfloor d/2 \rfloor, \lfloor d/2 \rfloor + 1}$ and node identities assigned by random permutation. Each admissible bipartite edge was included independently with probability $p \in \curl{0.1, 0.9}$, yielding $10$ and $90$ edges respectively, with one ground-truth graph per regime. Edge weight magnitudes were drawn from $\operatorname{Uniform}\parn{0.5, 1.0}$ with uniform signs, and ordinal nodes had $s \in \curl{3, 4}$ drawn uniformly. The type partition, weights, and cutpoints were held fixed per regime, so trials are data replicates from one model. Eight distribution modes were considered: seven homogeneous modes, one per family of Table~\ref{tab:exp_fam_dists}, and one mixed mode in which each $i \in \mathcal{V}_{\operatorname{exp}}$ draws uniformly from the same table. $N$ ranged over $20$ equally spaced values in $\sqr{5, 500}$, with $B = 1000$ trials per configuration.

\subsection{Optimizer Settings}
\label{app:exp_details_alg}
Per-node conditional maximum likelihood estimation (Appendix~\ref{app:per_node_mle}) used L-BFGS-B, with at most $200$ iterations and objective tolerance $\epsilon_{\operatorname{f}} = 10^{-8}$. The greedy search (Algorithm~\ref{alg:greedy_bic}) terminated when no reversal improved the score by more than $\epsilon_{\operatorname{s}} = 10^{-6}$, or after $4\bars{\mathcal{E}_u} + 1$ reversals, whichever occurred first.

\bibliography{references}

\end{document}